\documentclass{article}
\usepackage[journal=PRE,mode=blank]{ems-journal}

\numberwithin{equation}{section}

\usepackage{url}

\usepackage{subcaption}
\usepackage{tikz}
\usepackage{mathrsfs,bbm,url,comment}

\newtheorem{theorem}{Theorem}[section]
\newtheorem{proposition}[theorem]{Proposition}

\newtheorem{lem}[theorem]{Lemma}
\newtheorem{cor}[theorem]{Corollary}

\newtheorem{rmk}[theorem]{Remark}
\newtheorem{assumption}[theorem]{Assumption}
\newtheorem{definition}[theorem]{Definition}

\usepackage{multirow}
\usepackage{colortbl} % For colored rows
\definecolor{darkgreen}{rgb}{0, .5, 0}
\definecolor{darkcerulean}{rgb}{0.03, 0.27, 0.49}
\definecolor{smokyblack}{rgb}{0.06, 0.05, 0.03}
\definecolor{warmblack}{rgb}{0.0, 0.26, 0.26}
\definecolor{cobalt}{rgb}{0.0, 0.28, 0.67}
\definecolor{aoEnglish}{rgb}{0.0, 0.5, 0.0}
\definecolor{carribeangreen}{rgb}{0.0, 0.8, 0.6}
\definecolor{persiangreen}{rgb}{0.0, 0.65, 0.58}
\definecolor{lightgray}{gray}{0.93}
\definecolor{midgray}{gray}{0.6}

\definecolor{darkgoldenrod}{rgb}{0.72, 0.53, 0.04}
\definecolor{electricviolet}{rgb}{0.56, 0.0, 1.0}
\definecolor{internationalorange}{rgb}{1.0, 0.31, 0.0}
\definecolor{portlandorange}{rgb}{0.80, 0.28, 0.168}

\usepackage{pifont}   % for \ding
\definecolor{deepjunglegreen}{rgb}{0.0, 0.29, 0.29}
\definecolor{dogwoodrose}{rgb}{0.84, 0.09, 0.41}
\definecolor{sanddune}{rgb}{0.59, 0.44, 0.09}
\definecolor{slategray}{rgb}{0.44, 0.5, 0.56}
\definecolor{burntsienna}{rgb}{0.91, 0.45, 0.32}

\usepackage{pifont}

\usepackage[linesnumbered,ruled,vlined]{algorithm2e}

\SetCommentSty{mycommfont}
\SetKwInput{KwIn}{Input}
\SetKwInput{KwOut}{Output}

\newcommand{\eqdef}{\ensuremath{\stackrel{\mbox{\upshape\tiny def.}}{=}}}

\begin{document}

%---
% Insert the title of your paper and (if necessary)
% a short title for the running head.
%---
\title{A Closed-Form Formula for Consistent Lipschitz Regression on Metric Spaces with Sparse Neural Network Realizations}
\titlemark{Consistent Lipschitz Regression on Metric Spaces}

%---

% Authors listed alphabetically by surname.

\emsauthor{1}{
	\givenname{Ruiyang}
	\surname{Hong}
	\mrid{}
	\zblid{}
	\orcid{}}{R.~Hong}

\emsauthor{2}{
	\givenname{Hrad}
	\surname{Ghoukasian}
	\mrid{}
	\zblid{}
	\orcid{}}{H.~Ghoukasian}

\emsauthor{3}{
	\givenname{Anastasis}
	\surname{Kratsios}
	\mrid{}
	\zblid{}
	\orcid{0000-0001-6791-3371}}{A.~Kratsios}

\Emsaffil{1}{
	\department{1}{Department of Mathematics and Statistics}
	\organisation{1}{McMaster University}
	\rorid{1}{02fa3aq29}
	\address{1}{1280 Main Street West}
	\zip{1}{L8S 4K1}
	\city{1}{Hamilton, Ontario}
	\country{1}{Canada}
	\affemail{1}{hongr5@mcmaster.ca}
	\department{2}{}
	\organisation{2}{Vector Institute for Artificial Intelligence}
	\rorid{2}{03kqdja62}
	\address{2}{W1140-108 College Street}
	\zip{2}{M5G 0C6}
	\city{2}{Toronto, Ontario}
	\country{2}{Canada}
	\affemail{2}{}}

\Emsaffil{2}{
	\department{1}{The Edward S. Rogers Sr. Department of Electrical \& Computer Engineering}
	\organisation{1}{University of Toronto}
	\rorid{1}{02fa3aq29}
	\address{1}{40 St George St}
	\zip{1}{M5S 2E4}
	\city{1}{Toronto, Ontario}
	\country{1}{Canada}
	\affemail{1}{ hrad.ghoukasian@mail.utoronto.ca}
%
	% \department{2}{}
	% \organisation{2}{Vector Institute for Artificial Intelligence}
	% \rorid{2}{03kqdja62}
	% \address{2}{W1140-108 College Street}
	% \zip{2}{M5G 0C6}
	% \city{2}{Toronto, Ontario}
	% \country{2}{Canada}
	% \affemail{2}{}
    }

\Emsaffil{3}{
	\department{1}{Department of Mathematics and Statistics}
	\organisation{1}{McMaster University}
	\rorid{1}{02fa3aq29}
	\address{1}{1280 Main Street West}
	\zip{1}{L8S 4K1}
	\city{1}{Hamilton, Ontario}
	\country{1}{Canada}
	\affemail{1}{kratsioa@mcmaster.ca}
	\department{2}{}
	\organisation{2}{Vector Institute for Artificial Intelligence}
	\rorid{2}{03kqdja62}
	\address{2}{W1140-108 College Street}
	\zip{2}{M5G 0C6}
	\city{2}{Toronto, Ontario}
	\country{2}{Canada}
	\affemail{2}{}}

%---

%---
% Add MSC 2020 codes according to https://zbmath.org/classification/.
% A unique primary MSC code (in curly brackets) is mandatory,
% while secondary MSC codes (in square brackets) are optional.
%---
\classification[62G08, 41A25, 68Q32]{68T07}

%---
% Add a list of keywords.
%---


%---
% Insert your abstract.
%---
\begin{abstract}
% Kernel ridge regression is both computationally and analytically tractable because its estimator admits a simple closed-form expression and optimizing a convex objective; neither feature is generally available for deep neural networks.
Several classical machine-learning methods, such as KRRs and SVRs, are both computationally and analytically tractable since their estimators either admit closed-form expressions or are obtained by minimizing convex training objectives; neither feature is generally available for deep neural networks.
We address this by introducing a simple closed-form ``two-stage'' compositional formula $\hat{f}$ for reconstructing an unknown Lipschitz function $f:\mathcal{X}\to \mathbb{R}$ on a metric space $(\mathcal X,\rho)$ from $N$ i.i.d.\ noisy observations.

Our main result is a high-probability uniform ($L^{\infty}$) recovery guarantee that jointly controls approximation and statistical errors while enjoying an optimization error of zero; in particular, we do not assume oracle access to an approximate ERM.
Our secondary main results establish the optimality of our formula in three complementary senses. 1) \textit{Function space:} On Ahlfors-regular metric spaces, the hypothesis class parameterized by our formula attains the optimal fat-shattering dimension. 2) \textit{Parameter space:} Its dependence on the parameters is maximally numerically stable, in the sense that a smaller approximation error cannot be achieved with a smaller Lipschitz dependence on the model parameters. 3) \textit{Forward pass:} Its dependence on the input is maximally regular, matching the Lipschitz constant of the target function $f$.
When $\mathcal X=[0,1]^d$ is equipped with the $\ell^\infty$ norm, $\hat{f}$ admits algorithmic \texttt{ReLU}-MLP and exact \texttt{ReLU}-multi-head transformer realizations of depth $\mathcal{O}(\log(N))$ with $\mathcal{O}(N)$ nonzero parameters.
\end{abstract}
\maketitle

\section{Introduction}
\label{s:Intro}

The nonparametric regression problem is at the core of supervised machine learning: the objective is to learn an unknown function $f:\mathcal{X}\to \mathbb{R}$ defined on a metric space $(\mathcal{X},\rho)$ from noisy training data $(X_1,Y_1),\dots,(X_N,Y_N)$, defined on a probability space $(\Omega,\mathcal{F},\mathbb{P})$ and taking values in $\mathcal{X}\times \mathbb{R}$, where
\begin{equation}
\label{eq:noisy}
Y_n = f(X_n) + \varepsilon_n 
\qquad\quad
\mbox{for } n=1,\dots,N,
\end{equation}
and the $\varepsilon_n$ are i.i.d.\ centred random variables satisfying suitable integrability conditions. A central goal of non-parametric regression is to best infer $f$, by identifying parameters for a finitely-parameterized model which is as close to $f$ as possible; using only the information present in the training data. 
For instance, when $(\mathcal{X},\rho)=(\mathbb{R}^d,\|\cdot\|_1)$, we are in the standard deep-learning setting in which \texttt{ReLU} multi-layer perceptrons (MLPs) are analysed. When $\mathcal{X}$ is a set of finitely supported probability measures on $[0,1]^d$, we enter settings arising in the analysis of point clouds on $[0,1]^d$, such as the DeepSets architecture~\cite{zaheer2017deep}, or in probabilistic formalizations of context~\cite{furuya2026transformers} arising in in-context learning. When $\mathcal{X}=L^2([0,1]^d)$, or a Sobolev subspace thereof, we are in a special case of operator learning, namely learning real-valued functionals.

Unlike classical machine-learning tools---e.g.\ regularized kernel regression, whose ridge-penalized estimator admits an explicit closed-form expression~\cite[Theorem~5.2]{kimeldorf1971some}; Group LASSO~\cite{yuan2006model}, ENET~\cite{zou2005regularization} regression variants, and SVRs~\cite{smola2004tutorial}, which are defined through convex programs; or boosting procedures such as XGBoost~\cite{chen2016xgboost} and LightGBM~\cite{ke2017lightgbm}, which admit efficient greedy training procedures---the empirical training objectives of \textit{deep} neural networks (NNs), i.e.\ models leveraging \textit{compositionality}, are generally non-convex, and tractable closed-form solutions are not known in general. Although one can kernelize a deep-learning model, e.g.\ by using frozen hidden feature maps~\cite{lukovsevivcius2009reservoir,rahimi2007random,gonon2023approximation,neufeld2023universal} or passing to an NTK limit~\cite{jacot2018ntk}, doing so can strictly worsen approximation rates on compositional function classes~\cite{mhaskar2016deep} and yield poorer statistical performance than end-to-end deep learning~\cite{allen2019can}. Moreover, even one step of end-to-end gradient-based training can already exhibit a statistical separation from the corresponding kernel method~\cite{ba2022high}.

This makes it substantially more difficult to simultaneously control the \textit{approximation}, \textit{statistical} (generalization), and \textit{optimization} errors of neural networks. The first of these errors quantifies the mismatch between the target function and the model class induced by restricting the target function to the chosen model class; the second quantifies the discrepancy between the inferred model's predictive performance on the training data and its unobserved population (test) performance; and the optimization error quantifies the gap in the empirical objective between an optimal model within the chosen class, e.g.\ an empirical-risk minimizer (ERM), and the model selected by the data-dependent training algorithm. While approximation-theoretic guarantees
(e.g.~\cite{yarotsky2017error,petersen2018optimal,shen2022optimal,kratsios2025kolmogorov}) address the first of these errors and statistical-learning guarantees address the second, such guarantees typically do not control the optimization error incurred by the training algorithm.

%We identify a \textit{closed-form}, cf.~\eqref{eq:formula__realizable} below, solution to this problem within the class of $\omega$-uniformly-continuous (e.g.\ Lipschitz or H\"{o}lder\footnote{Via snowflaking.}) functions when the random queries $X_1,\dots,X_N$ have compactly supported law.  Our solution achieves near-minimal computation cost as 1) it has a closed-form formula using 2) an optimal number of parameters coinciding with the lower-bounds in stable non-linear approximation theory; i.e.\ the best rate of approximation rates possible if the model depends in a Lipschitz fashion on its trainable parameters, cf.~\cite{petrova2023lipschitz}, over the class $\mathcal{F}$ up to log-factors. 

\subsection{Main Contributions}
\label{s:Intro__ss:MainContributins}

We resolve this gap by identifying a simple closed-form ``compositional'' formula, in~\eqref{eq:formula__realizable}, for which we prove a high-probability uniform recovery guarantee directly, without assuming access to an exact or approximate ERM (Theorem~\ref{thrm:nonparametric_noisy_case}).  Our formula is optimal in a wide-range of sense discussed below, in some cases up to logarithmic factors.

Importantly, when $(\mathcal{X},\rho)=([0,1]^d,\|\cdot\|_{1})$, our formula admits an \emph{explicit} \texttt{ReLU}-MLP realization (Proposition~\ref{prop:representation}); and when $(\mathcal{X},\rho)$ is a suitable space of finitely supported probability measures on $\mathbb{R}^d$, our formula admits a DeepSets realization for point clouds; see~\eqref{eq:DeepSets_factorization} and~\cite{zaheer2017deep}.  This allows us to transfer the recovery guarantee to these explicitly constructed deep-learning realizations without introducing a separate optimization-error term.  The concreteness of the formula also provides an interpretable description of what the resulting neural-network realization computes.

Our main result (Theorem~\ref{thrm:nonparametric_noisy_case}) gives a simple closed-form estimator $
    \hat{f}:\mathcal{X}\to\mathbb{R}
$ which uniformly recovers $f$, with high probability, under the well-specified regression model $Y_n=f(X_n)+\varepsilon_n$, where the noises $\{\varepsilon_n\}_{n=1}^N$ are i.i.d.\ and independent of the sample locations $\{X_n\}_{n=1}^N$.  The observed training dataset is
$\mathbb{D}_N
    \eqdef
    ((X_n,Y_n))_{n=1}^N$.

\begin{figure}[th!b]%[H]
    \centering
    \includegraphics[width=\linewidth]{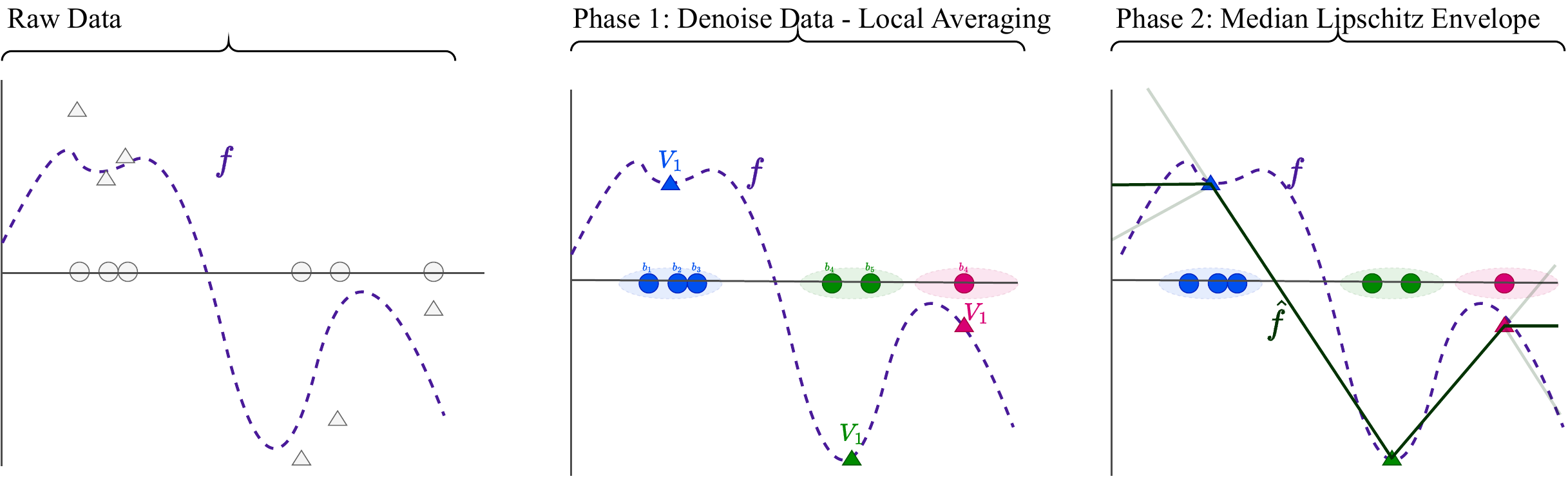}
    \caption{\textbf{Intuition behind the Formula.} The theoretical estimator in~\eqref{eq:formula__realizable} performs two operations.  First, Algorithm~\ref{alg:LocalAveragingRobustification} uses the second sample block to compute a locally averaged value $V_k^{(\beta,q)}$ at each reference point $\xi_k$ from the first block. Second, the estimator takes the midpoint of the upper and lower $L$-Lipschitz envelopes generated by the effective pairs $(\xi_k,V_k^{(\beta,q)})_{k=1}^{N_0}$.  The resulting map is $L$-Lipschitz.  In the noisy setting, no exact interpolation of the locally averaged values is asserted.}
    \label{fig:epsilon_neighbourhood_graph}
\end{figure}

\paragraph{The Formula}
Our formula, illustrated in Figure~\ref{fig:epsilon_neighbourhood_graph} and detailed in~\eqref{eq:formula__realizable}, operate as follows. We first split the training sample into a reference block of size $N_0$ and an averaging block of size $N_1$, where $N=N_0+N_1$. Let $\beta\ge0$, $q\in[N_1]_+$, and $L>0$, and define the neighbourhood-averaging matrix $G^{(\beta,q)}\in\mathbb{R}^{N_1\times N_0}$, as illustrated in Figure~\ref{fig:epsilon_neighbourhood_graph}. We then define the updated formula as
\begin{equation}
\label{eq:formula__realizable}
\begin{aligned}
    \hat{f}_{\mathbb{D}_N}^{(\beta,q,L)}(x)
    &=\tfrac{1}{2}\Big(\min_{k\in[N_0]_+}\big(V_k^{(\beta,q)}+\chi_k\big)+\max_{k\in[N_0]_+}\big(V_k^{(\beta,q)}-\chi_k\big)\Big),\\
    V^{(\beta,q)}&=a^{\operatorname{av}}G^{(\beta,q)},\\
    \chi&=\big(L\,\rho(x,X_k)\big)_{k=1}^{N_0},\\
    a^{\operatorname{av}}&=(Y_{N_0+1},\ldots,Y_{N_0+N_1}).
\end{aligned}
\end{equation}
%%%
In analogy with attention terminology, each reference point $X_k$ acts as a key, $V_k^{(\beta,q)}$ is its locally averaged value, and $L\rho(x,X_k)$ is the metric penalty associated with the query $x$. This terminology is only mnemonic: the aggregation operation in~\eqref{eq:formula__realizable} is the midpoint of a minimum and a maximum Lipschitz envelope, rather than standard softmax attention.
Lastly, the ``neighbourhood-averaging operator'', cf.~Figure~\ref{fig:epsilon_neighbourhood_graph}, is computed by the $N_1\times N_0$ matrix $G^{(\beta,q)}$ whose $(i,k)$-th coordinate is
\begin{equation}
\label{eq:local_averaging__operator}
\begin{aligned}
G^{(\beta,q)}_{ik}
&\eqdef \frac{1}{q}\,
\mathbbm{1}_{\{\mathbf{1}_{N_1}^{\top}H_{:,k}\ge q\}}\,
H_{ik}^{(\beta)}\,
\mathbbm{1}_{\{\mathbf{1}_{i}^{\top}H_{[1:i],k}\le q\}},\\
H_{ik}^{(\beta)}
&\eqdef \mathbbm{1}_{\{\rho(X_{N_0+i},X_k)\le\beta\}},
\end{aligned}
\end{equation}
for $i\in[N_1]_+$ and $k\in[N_0]_+$, where $H_{:,k}$ denotes the $k$-th column of $H$ and $H_{[1:i],k}$ denotes the first $i$ entries of that column.

%%%

Our main result (Theorem~\ref{thrm:nonparametric_noisy_case}) shows that the split-sample estimator~\eqref{eq:formula__realizable} is uniformly consistent. In particular, if the ground truth $f$ is $1$-Lipschitz, then for any accuracy parameter $\varepsilon>0$ and any user-specified failure probability $0<\delta<1$, setting our formula's parameters to
$
    L=1
$, $\beta=\frac{\varepsilon}{2}$, and
$
    q=q_\varepsilon\eqdef\left\lceil\varepsilon^{-2}\right\rceil
$
guarantees that
\[
    \sup_{x\in X}
    \big|
        \hat{f}_{\mathbb D_N}^{(\beta,q_\varepsilon,1)}(x)-f(x)
    \big|
    \le
    \varepsilon
    \Big(
        1+\sqrt{
            C\log\big(\tfrac{4N_0}{\delta}\big)
        }
    \Big)
\]
holds with probability at least $1-\delta$\footnote{Over the draw of $\mathbb{D}_N$.} and
$
    N=
    \widetilde{\mathcal O}_{a,s}
    \big(
        \varepsilon^{-(s+2)}
        +
        \varepsilon^{-s}\log(1/\delta)
    \big)
$ samples suffice. The split $N=N_0+N_1$ chosen as in Theorem~\ref{thrm:nonparametric_noisy_case}.

\subsection{Secondary Contributions: Optimality of our Formula}
\label{s:Intro__ss:SecondaryContributins}

Beyond our main reconstruction guarantees, we establish that our closed-form formula is optimal, or near-optimal, from three complementary perspectives. These results distinguish between the parametric complexity of a neural realization, the combinatorial complexity of the realized function class, and the numerical stability of the parameter-to-realization map. 

\paragraph{Sparse realization}
For $p\in\{1,\infty\}$, we prove that our reconstruction formula on $([0,1]^d,|\cdot|_p)$ admits an exact ReLU-MLP realization of depth $\mathcal{O}(\log N)$, width $\mathcal{O}(N)$, and only $\mathcal{O}(N)$ non-zero parameters (Proposition~\ref{prop:representation}). In the oracle setting, this realization achieves the uniform approximation bound $|f-\hat{f}|_{\infty}\lesssim N^{-1/d}$, matching the optimal exponent for continuous weight assignment established in~\cite{yarotsky2018optimal}; namely, the exponent $1/d$ cannot be improved when the network weights are required to depend continuously on the target function. Moreover, unlike the sparse \texttt{ReLU} estimators in~\cite{schmidthieber2020nonparametric}, whose statistical guarantees require approximate empirical-risk minimization, our network is constructed explicitly from the possibly noisy observations and therefore requires no iterative neural-network training; its recovery bound contains no separate optimization-error term due to the closed-form expression.  In particular, we do not assume access to an approximate ERM as in~\cite[Theorem~1]{schmidthieber2020nonparametric}.

\paragraph{Optimal complexity of the realized function class.}
We prove that, on any compact $d$-Ahlfors regular metric measure space, the hypothesis class realized by our formula has fat-shattering dimension of the optimal order $(L/\varepsilon)^d$ at approximation scale $\varepsilon$ (Theorem~\ref{thrm:fatshatteringotimality}). This is an intrinsic function-space notion of optimality: every $L$-Lipschitz hypothesis class capable of uniformly approximating the bounded $L$-Lipschitz ball to accuracy $\varepsilon$ must have fat-shattering dimension at least of this order. Thus, our realized class has exactly the scale-sensitive complexity required for approximation--large enough to approximate the entire Lipschitz ball, but no larger than necessary. This provides a scale-sensitive analogue of the VC-dimension-based obstruction in ~\cite{shen2022optimal}: using the fat-shattering dimension of~\cite{kearns1994efficient} and related to uniform learnability in~\cite{alon1997scale,bartlett1996fat}.

\paragraph{Stable near-minimax parameterization.}
In the oracle-access to noiseless training data regime, Theorem~\ref{thm:stable_near_minimax_optimality_oracle_regime} shows that our parameter-to-realization map attains the Lipschitz-width benchmark of~\cite{petrova2023lipschitz} up to a logarithmic factor.  Thus, among approximation families whose parameter-to-function realization maps satisfy the prescribed Lipschitz-stability constraint, our approximation rate is near-minimax up to this factor.  This contrasts with the discontinuous weight-selection regimes considered in~\cite{yarotsky2018optimal,ShenYangZhang_JMPA_OptApprx_ReLU}, where the network weights may depend discontinuously on the target function; cf.~\cite{petrova2023limitations,cohen2022optimal}.

\subsection{Related Research}
\label{s:Introduction__ss:Related}

When $\mathcal{X}=\mathbb{R}^d$ and $\rho=\|\cdot\|_p$ for some $1\le p\le\infty$, a substantial approximation-theoretic literature provides quantitative universal \emph{approximation} guarantees for neural networks; see, e.g.~\cite{yarotsky2017error,petersen2018optimal,shen2022optimal,kratsios2025kolmogorov}. These results guarantee variants of the following: for every continuous $f$ and $\varepsilon>0$, there \textit{exists} a neural network $f_{\operatorname{apx}}:\mathbb{R}^d\to\mathbb{R}$ with
\begin{equation}
\label{eq:background__univ_approx}
    \sup_{x\in[0,1]^d}|f_{\operatorname{apx}}(x)-f(x)|
    \le \varepsilon.
\end{equation}
Thus, they guarantee the existence of such an $f_{\operatorname{apx}}$ but are essentially ``oracle-quality'' from an algorithmic perspective: one may select suitable network parameters without addressing how those parameters are to be recovered from noisy training data.

Statistical results instead study a data-dependent estimator $\hat{f}\in\mathcal{H}$ and control its population risk under suitable assumptions on the data distribution, loss, and complexity of $\mathcal{H}$. Typically one assumes access to an approximate (or exact) ERM $\hat{f}$
\begin{equation}
\label{eq:approx_ERM}
    \frac{1}{N}\sum_{n=1}^N
        \ell\big(\hat{f}(X_n),Y_n\big)
\le
    \inf_{\tilde f\in\mathcal{H}}
        \frac{1}{N}\sum_{n=1}^N
        \ell\big(\tilde f(X_n),Y_n\big)
    +\varepsilon_N
% ,
\end{equation}
where $\varepsilon_N\ge0$ is an optimization error, typically required to vanish as the number of samples $N\to\infty$. For example, in the squared-loss setting, \cite[Theorem~1]{schmidthieber2020nonparametric} controls the statistical risk in terms of an empirical optimization-gap quantity, while \cite[Corollary~1]{schmidthieber2020nonparametric} specializes to an exact empirical-risk minimizer.

Once an estimator with sufficiently small empirical risk has been identified, statistical-learning theory can yield high-probability bounds of the schematic form
\begin{equation}
\label{eq:background__statistical}
    \underbrace{
        \mathbb{E}\!\left[
            \ell\big(\hat{f}(X),Y\big)
        \right]
    }_{\text{Population Risk } \mathcal{R}(\hat{f})}
    \le
    \underbrace{
        \frac{1}{N}\sum_{n=1}^N
            \ell\big(\hat{f}(X_n),Y_n\big)
    }_{\text{Empirical Risk } \widehat{\mathcal{R}}_{\mathcal{D}^N}(\hat{f})}
    +
    \underbrace{
        \operatorname{Gen}_N(\mathcal{H},\delta)
    }_{\text{Generalization Term}},
\end{equation}
with probability at least $1-\delta$, where $\operatorname{Gen}_N(\mathcal{H},\delta)\to0$ under appropriate complexity and integrability assumptions. Such terms may be controlled through local Rademacher complexities~\cite{bartlett2005local}, fat-shattering or related scale-sensitive dimensions~\cite{alon1997scale}, and VC- or pseudo-dimension estimates for neural networks~\cite{bartlett2019nearly}. Complementary neural-network bounds exploit parameter norms and margins~\cite{bartlett2017spectrally,neyshabur2018a}, while instance-dependent guarantees may depend directly on the learned predictor and observed data~\cite{hou2023instance}, or more generally on the geometry of an associated limiting Gaussian process~\cite{bartl2025we}.

\paragraph{Where's the Gap: Controlling the Optimization Error}
Universal approximation results, e.g.~\eqref{eq:background__univ_approx}, can guarantee that the approximation error of a sufficiently rich neural-network class is arbitrarily small. Statistical results can then control the corresponding estimation and generalization errors, provided that a suitable approximate ERM as in~\eqref{eq:approx_ERM} can be algorithmically identified. The remaining problem is that the \textbf{optimization error} in~\eqref{eq:approx_ERM} \emph{need not} converge to zero in practice.

Indeed, there is mounting evidence that standard training algorithms, e.g.\ (stochastic) gradient-descent-type methods, can fail to identify empirical-risk minimizers~\cite{jentzen2025nonconvergence,do2026nonconvergence,cheridito2021nonconvergence}, and can even yield statistically inconsistent predictors or fail to attain the optimal population risk~\cite{holzmuller2022training,do2025optimalrisk}; see also~\cite{safran2018spurious} for spurious local minima in two-layer \texttt{ReLU} networks. These difficulties reflect, in particular, the non-convexity of neural-network loss landscapes and the absence of a tractable closed-form characterization of optimal neural networks; the latter is the focus of this paper.

\paragraph{Related estimators and approximation results in the non-Euclidean setting}
In~\cite{balazs2026nearoptimal} the authors gives a tractable, near-minimax delta-convex estimator in Euclidean spaces and remarks that a max-out realization, but does not provide an exact sparse \texttt{ReLU}-MLP complexity theorem or extend the theory to general metric spaces;~\cite{gottlieb2017efficientregression} treat general metric spaces but compute an approximate minimizer of a Lipschitz-regularized empirical objective using specialized optimization. In contrast, our guarantee applies directly to the displayed metric-space estimator without an ERM assumption, while its $\ell^1$ and $\ell^\infty$ specializations admit exact sparse \texttt{ReLU} realizations by explicit weight assignment, with stated depth, width, and parameter counts.

It is also worth noting the universal approximation results on metric spaces of~\cite{kratsios2022universalgeo,kratsios2023metric,acciaio2024designing}.  While these results treat varying degrees of general metric, or topological~\cite{kratsios2020noneuclidean,galimberti2025neural,godeke2025encoderdecoder,ismailov2026topological}, spaces as their input spaces and some allowing for general topological output spaces; they all provide existence theorems, and none provide statistical guarantees, nor any implementable closed-form expressions.

\paragraph{Related Regression Formulas and Certifiable Training Algorithms}
While~\cite{kratsios2025beyond} provides a training algorithm for a single-head transformer with distance-based features that learns an unknown Lipschitz function $f:[0,1]^d\to\mathbb{R}$ from noisy observations, it does not establish the three optimality properties proved here: matching the target Lipschitz constant, optimal $\gamma$-fat-shattering dimension, or near-optimal Lipschitz-width efficiency in the sense of~\cite{petrova2023lipschitz}. Moreover, its recovery guarantee is not uniform, and no extension of the algorithm beyond its Euclidean setting is provided.

% \paragraph{Relation to Nadaraya-Watson Estimators} 
%%
The complementary analysis studies various overfitting regimes for an explicit distance-based Nadaraya-Watson~\cite{nadaraya1964estimating,watson1964smooth}-type formula studied in~\cite{barzilai2026beyond}; which can be linked~\cite{goel2024can} to transformers.  
Although these formulas do interpolate in the binary classification setting, can generalize, and their analysis likely extends to more general metric spaces; the Lipschitz constant of their estimators is easily bounded below by $
\tfrac{2}{\min_{i\neq j\in [N]_+:\,Y_i\neq Y_j}\|X_i-X_j\|}
$ which diverges as the sample size $N$ does.  This is \textit{not} the case for our formula~\eqref{eq:formula__realizable} which is independent of $N$.  
This diverging behaviour is also shared by the consistent Lipschitz estimators of~\cite{balazs2026nearoptimal}; who showed that their estimators have a diverging Lipschitz constant on the order of $\mathcal{O}(\sqrt{\log(N))}$ in~\cite[Lemma 16]{balazs2026nearoptimal}.

This diverging-Lipschitz pathology does not arise for the McShane--Whitney formula-type regressors $\widetilde{f}_N$ in~\cite[Lemma~7]{luxburg2004distance}\footnote{With their mixing constant set to $\tfrac{1}{2}$.} with a prescribed Lipschitz constant $L\ge 0$.  Their formula coincides with the central McShane-Whitney stage of ours when the neighbourhood-averaging step is omitted and the raw labels are used as the effective values; this is not the regime considered in Theorem~\ref{thrm:random_apprx_noisy_robustified}.  
However, their formula need not be uniformly consistent with $f$ when computed on noisy training data; indeed, any exact interpolant $\widetilde f_N$ satisfies
\begin{equation}
\label{eq:explosion_inconsistent}
    \|\widetilde{f}_N-f\|_\infty
    \ge
    \max_{n\in[N]_+}
    |\widetilde{f}_N(X_n)-f(X_n)|
    =
    \max_{n\in[N]_+}
    |\varepsilon_n|
.
\end{equation}
In the case of non-degenerate Gaussian noise
$\varepsilon_n\stackrel{\mathrm{i.i.d.}}{\sim}\mathcal N(0,\sigma^2)$
with $\sigma>0$, the right-hand side of~\eqref{eq:explosion_inconsistent}
diverges to infinity in probability (indeed, almost surely) as $N\to\infty$.
%%%
This is precisely the role of our neighbourhood-averaging operator $G^{(\beta,q)}$, which leverages the regularity of $f$ to approximately ``cancel out'' the noise by averaging labels $Y_m$ whose inputs are close to a given $X_n$ (i.e.\ $\rho(X_n,X_m)\approx0$).  In this way, it implicitly constructs an approximately noiseless effective training dataset before applying the central McShane--Whitney formula, thereby enabling consistency.

% \paragraph{Kernelization vs. Deep Learning}
% It is worth noting that the problem of algorithmically identifying an ERM becomes substantially more tractable when the deep-learning model is kernelized; e.g.\ through neural tangent kernels (NTKs)~\cite{jacot2018ntk} or kernels induced by random neural networks~\cite{gonon2020risk}, which reduce learning to convex kernel methods and permit sharp statistical guarantees~\cite{tsigler2023benign,barzilai2024generalization,cheng2024comprehensive}. However, kernelization can fundamentally change the learned model: neural networks can provably outperform their random-feature and NTK counterparts~\cite{ghorbani2019limitations,li2020beyond,malach2021quantifying}, while even a single gradient step can separate feature learning from the corresponding frozen-feature kernel model~\cite{ba2022high}. Thus, kernel limits provide a tractable surrogate for neural-network training, but do not in general capture the behaviour of trained neural networks.

\subsection*{Organization of Paper}
\label{s:Intro__ss:ToC}
%  
% \tableofcontents
The remainder of the paper is organized as follows. Section~\ref{s:Prelims} introduces the notation and background. Section~\ref{s:Main} proves the main reconstruction guarantees and gives the neural-network realizations. Section~\ref{s:Main__ss:Optim} establishes the optimality properties of the formula. Section~\ref{s:Applicatoins} develops applications, while Section~\ref{s:NumericalIllustrations} presents numerical experiments. Section~\ref{s:Conclusion} concludes, and the appendices collect the proofs and technical details.

\section{Preliminaries}
\label{s:Prelims}

\subsection{Notation}
\label{s:Prelim__ss:Notation}

We collect here the metric and approximation-theoretic notation used throughout the paper.  Unless otherwise stated, all metric spaces are assumed to be non-empty, and all measures are Borel measures on the metric topology.  We also note that definitions of standard deep learning architectures are recorded in Appendix~\ref{a:deeplearningarchs}.

\begin{definition}[$\ell^1$-product metrics]
Let $\{(X_i,\rho_i)\}_{i\in I}$ be a finite or countable family of metric spaces.  On the product $\prod_{i\in I}X_i$, we use the $\ell^1$-product distance
\[
        \rho^{(1)}
        \big(
            (x_i)_{i\in I},(y_i)_{i\in I}
        \big)
    \eqdef
        \sum_{i\in I}\rho_i(x_i,y_i)
.
\]
In particular, if $(X,\rho)$ is a metric space and $N\in\mathbb N_+$, then $X^N$ is equipped with
$
    \rho_1(\mathbb X,\widetilde{\mathbb X})
    \eqdef
    \sum_{n=1}^N \rho(x_n,\widetilde x_n)
$ for $\mathbb X=(x_n)_{n=1}^N$, and 
    $\widetilde{\mathbb X}=(\widetilde x_n)_{n=1}^N
$.
\end{definition}
This convention is deliberately additive: a perturbation of a data cloud is measured by summing the coordinate-wise perturbations, rather than by hiding them behind a maximum.
Thus, if
$
    \mathcal D_N=((x_n,y_n))_{n=1}^N
$
and
$
    \widetilde{\mathcal D}_N=((\widetilde x_n,\widetilde y_n))_{n=1}^N
$
belong to $[X\times\mathbb R]^N$, then the induced product metric is
\[
    \rho^{(1)}(\mathcal D_N,\widetilde{\mathcal D}_N)
    =
    \sum_{n=1}^N
        \left(
            \rho(x_n,\widetilde x_n)
            +
            |y_n-\widetilde y_n|
        \right).
\]
The diameter of a metric space $(X,\rho)$ is
$
    \operatorname{diam}(X)
\eqdef
    \sup_{x,y\in X}\rho(x,y)
$.  
A \textit{metric measure space} (mms) is a triple $(X,\rho,\mathfrak m)$, where $(X,\rho)$ is a metric space and $\mathfrak m$ is a Borel measure on $X$.

\begin{definition}[Doubling metric spaces]
A metric space $(X,\rho)$ is called doubling if there exists
$\lambda_X\in\mathbb N_+$ such that every ball $B(x,r)\subseteq X$ can be covered by at most $\lambda_X$ balls of radius $r/2$.  The smallest such $\lambda_X$ is called the doubling constant of $(X,\rho)$, and the doubling dimension is
$
    \operatorname{ddim}(X,\rho)
\eqdef
    \log_2(\lambda_X)
$.
\end{definition}

\begin{definition}[$\varepsilon$-nets, covering numbers, and packing numbers]
Let $(X,\rho)$ be a metric space and let $\varepsilon>0$.  A set
$\mathcal X_\varepsilon\subseteq X$ is called an $\varepsilon$-net of $X$ if
$
    X
    \subseteq
    \bigcup_{x\in\mathcal X_\varepsilon}
        B(x,\varepsilon).
$
The $\varepsilon$-covering number of $X$ is
\[
    \mathcal N(X,\rho,\varepsilon)
    \eqdef
    \inf
    \left\{
        |\mathcal X_\varepsilon|
        :
        \mathcal X_\varepsilon\subseteq X
        \text{ is an } \varepsilon\text{-net of }X
    \right\}.
\]
\end{definition}

Covering numbers measure how efficiently the space can be observed, while packing numbers measure how many mutually distinguishable points the space can contain at a prescribed resolution.

\begin{definition}[$d$-Ahlfors regular metric measure spaces]
Let $d>0$.  A metric measure space $(X,\rho,\mathfrak m)$ is called
$d$-Ahlfors regular if $\operatorname{supp}(\mathfrak m)=X$ and there exist constants
$0<c_{\rm A}\le C_{\rm A}<\infty$ such that
\[
    c_{\rm A} r^d
    \le
    \mathfrak m(B(x,r))
    \le
    C_{\rm A} r^d
\]
for every $x\in X$ and every
$
    0<r\le\operatorname{diam}(X).
$
\end{definition}
Ahlfors regularity is the scale-invariant regime in which the metric dimension seen by coverings agrees with the dimension seen by volume growth.
\begin{definition}[Lipschitz maps and Lipschitz model classes]
Let $(X,\rho_X)$ and $(Y,\rho_Y)$ be metric spaces, and let $L\ge0$.  A map
$f:X\to Y$ is called $L$-Lipschitz if
\[
    \rho_Y(f(x),f(y))
    \le
    L\,\rho_X(x,y)
\]
for every $x,y\in X$.  Its Lipschitz constant is
$
    \operatorname{Lip}(f)
\eqdef
    \sup_{x\neq y}
    \,
        \tfrac{\rho_Y(f(x),f(y))}{\rho_X(x,y)}
$.
When $h:X\to\mathbb R$ is real-valued, we also write $
    |\nabla h|(x)
    \eqdef
    \limsup_{y\to x}
        \tfrac{|h(x)-h(y)|}{\rho(x,y)}
$ for its pointwise upper slope, with the convention that the value is $+\infty$ if the corresponding $\operatorname{limsup}$ is infinite.
\hfill\\
\noindent
For a metric space $(X,\rho)$, we define the unit Lipschitz model class
\[
    \mathcal F_1
    \eqdef
    \operatorname{Lip}((X,\rho),[-1,1];1)
    =
    \left\{
        f:X\to[-1,1]
        :
        |f(x)-f(y)|
        \le
        \rho(x,y)
        \text{ for all }x,y\in X
    \right\}.
\]
If $(X,\rho)$ has doubling dimension $0<d<\infty$, we abbreviate
$
    \mathcal F_{1:d}
\eqdef
    \operatorname{Lip}((X,\rho),[-1,1];1)
$.
All such classes are equipped with the uniform norm
$
    \|f\|_\infty
\eqdef
    \sup_{x\in X}|f(x)|
$.
\end{definition}

\section{Reconstruction Guarantees}
\label{s:Main}

This section contains our main reconstruction result; followed by variants with different degrees of oracle access.  Those variants are used to help frame our main result, which assumes no type of oracle access, to other results in the literature; especially in the approximation theory for neural networks sub-area.

Our main result holds under the following assumptions.
\begin{assumption}[Distributional Assumptions]
\label{assumption:main_1Lip}
Let $(X,\rho)$ be a compact metric space, and let $\mu\in\mathcal{P}(X\times\mathbb{R})$.  Let $
    \mu_X\eqdef(\pi_X)_\#\mu
$
and $
    \nu\eqdef(\pi_{\mathbb{R}})_\#\mu
$. Let
$
    (Z_n)_{n\in\mathbb{N}_+}
    =
    ((X_n,\varepsilon_n))_{n\in\mathbb{N}_+}
$ be an i.i.d.\ sequence with common law $\mu$.  Assume that:
\begin{enumerate}
    \item[(i)] There exist constants $a>0$ and $s>0$ such that
    $\operatorname{supp}(\mu_X)=X$ and
    \[
        \mu_X(B_\rho(x,r))
        \ge
        a r^s
    \]
    for every $x\in X$ and every
    $0<r\le\operatorname{diam}(X)$.

    % \item[(ii)] The noise law $\nu$ is integrable, i.e.\ $\mathbb{E}_{\varepsilon\sim \nu}[|\varepsilon|]<\infty$, centred, i.e.\ $\mathbb{E}_{\varepsilon\sim \nu}[\varepsilon]=0$, and there is a $C>0$ such that it satisfies the log-Sobolev inequality%
    % % \footnote{The left-hand side of~\eqref{eq:noise_LSI} is precisely the entropy of $g^2$ with respect to $\nu$.}
    % \begin{equation}
    % \label{eq:noise_LSI}
    %     \underbrace{
    %         \mathbb{E}_{\varepsilon\sim \nu}[g^2(\varepsilon)\log(g(\varepsilon))]
    %     -
    %     \mathbb{E}_{\varepsilon\sim \nu}[g^2(\varepsilon)]
    %     \,
    %     \log\big(
    %         % \int_{\mathbb{R}}g^2\,d\nu
    %         \mathbb{E}_{\varepsilon\sim \nu}[g^2(\varepsilon)]
    %     \big)
    %     }_{\eqdef \operatorname{Ent}_{\nu}(g^2)}
    % \le
    %     C
    %     \int_{\mathbb{R}}
    %         |\nabla g|^2\,d\nu
    % \end{equation}
    % for every locally Lipschitz
    % $g:\mathbb{R}\to\mathbb{R}$ for which both sides are finite.
    \item[(ii)] The noise law $\nu$ is integrable, i.e.\ $\mathbb{E}_{\varepsilon\sim \nu}[|\varepsilon|]<\infty$, centred, i.e.\ $\mathbb{E}_{\varepsilon\sim \nu}[\varepsilon]=0$, and there is a $C>0$ such that it satisfies the log-Sobolev inequality
    \begin{equation}
    \label{eq:noise_LSI}
        \underbrace{
            \mathbb{E}_{\varepsilon\sim \nu}
            \big[
                g^2(\varepsilon)\log(g^2(\varepsilon))
            \big]
        -
            \mathbb{E}_{\varepsilon\sim \nu}
            \big[
                g^2(\varepsilon)
            \big]
            \,
            \log\bigg(
                \mathbb{E}_{\varepsilon\sim \nu}
                \big[
                    g^2(\varepsilon)
                \big]
            \bigg)
        }_{\eqdef \operatorname{Ent}_{\nu}(g^2)}
    \le
        C\,
        \mathbb{E}_{\varepsilon\sim \nu}
        \big[
            |\nabla g|(\varepsilon)^2
        \big]
    \end{equation}
    for every locally Lipschitz
    $g:\mathbb{R}\to\mathbb{R}$ for which both sides are finite, where
    $0\log(0)\eqdef0$.

    \item[(iii)] The design and noise are independent:
    $
        \mu=\mu_X\otimes\nu
    $.
\end{enumerate}
\end{assumption}
For each fixed target
$f\in\operatorname{Lip}((X,\rho);1)$, define the observed responses by
\[
    Y_n\eqdef f(X_n)+\varepsilon_n,
    \qquad n\in[N]_+,
\]
and distinguish the latent design-noise sample
$
    \mathbb{Z}_N
    \eqdef
    ((X_n,\varepsilon_n))_{n=1}^N
$
from the observed training dataset
$
    \mathbb{D}_N
    \eqdef
    ((X_n,Y_n))_{n=1}^N
$.

\begin{theorem}[Uniform Recovery From Noisy Samples]
\label{thrm:nonparametric_noisy_case}
Suppose that Assumption~\ref{assumption:main_1Lip} holds.  Let $N_0,N_1\in\mathbb{N}_+$, set $N\eqdef N_0+N_1$, and use the first $N_0$ observations as the reference block and the remaining $N_1$ observations as the averaging block.
Fix $\varepsilon\in(0,1)$ and $\delta\in(0,1)$ such that
$
    \varepsilon/2\le\operatorname{diam}(X)
$.
Set
$
    \eta\eqdef\frac{\varepsilon}{4}
$, $\beta\eqdef\frac{\varepsilon}{2}$, $q_\varepsilon
    \eqdef \left\lceil\varepsilon^{-2}\right\rceil$ and assume that
\begin{equation}
\label{eq:main_robustified_sample_condition}
\begin{aligned}
    N_0
    &\ge
    \frac{1}{a(\varepsilon/8)^s}
    \log\left(
        \frac{
            4\mathcal N(X,\rho,\varepsilon/8)
        }{
            \delta
        }
    \right),
    \\
    N_1
    &\ge
    \frac{1}{a(\varepsilon/2)^s}
    \max\left\{
        2q_\varepsilon,
        8\log\left(
            \frac{4N_0}{\delta}
        \right)
    \right\}.
\end{aligned}
\end{equation}
For every fixed
$f\in\operatorname{Lip}((X,\rho);1)$, let $\hat{f}\eqdef \hat{f}_{\mathbb{D}_N}^{(\beta,q_\varepsilon,1)}$ be as in~\eqref{eq:formula__realizable}.  Then:
\begin{enumerate}
    \item[(i)] For every realization of $\mathbb{D}_N$, the map
    $
        \hat{f}
        :X\to\mathbb{R}
    $
    is $1$-Lipschitz.

    \item[(ii)] With respect to the joint law
    $
        \mu^{\otimes N}
        =
        (\mu_X\otimes\nu)^{\otimes N},
    $
    \begin{equation}
    \label{eq:main_robustified_uniform_bound}
        \mathbb{P}\biggr(
            \left\|
                \hat{f}
                -
                f
            \right\|_\infty
            \le
            \varepsilon
            \Big(
                1
                +
                \sqrt{
                    C
                    \log\big(
                        \tfrac{4N_0}{\delta}
                    \big)
                }
            \Big)
        \biggl)
        \ge
        1-\delta.
    \end{equation}
\end{enumerate}
\end{theorem}
For fixed $\delta$, the lower-Ahlfors entropy estimate gives
$
    N_0
    =
    \mathcal{O}_{a,s}
    (
        \varepsilon^{-s}
        \log\left(\tfrac{1}{\varepsilon}\right)
    ),
$
whereas the second condition in
\eqref{eq:main_robustified_sample_condition} gives
$
    N_1
    =
    \mathcal{O}_{a,s}
    (
        \varepsilon^{-(s+2)}
        +
        \varepsilon^{-s}
        \log\left(\tfrac{1}{\varepsilon}\right)
    ).
$
Since Theorem~\ref{thrm:nonparametric_noisy_case} requires a sample size $N=N_0+N_1$ then 
\begin{equation}
\label{eq:DA_RATE}
    N\in 
    \widetilde{\mathcal{O}}_{a,s}
    (\varepsilon^{-(s+2)})
.
\end{equation}

\subsection{Discussion and Implications}
\label{s:Discussion_ss:OracleAccess}

We now compare and relate our results to others found in the literature, ignoring the formula aspect of it, and focusing on additional amounts of oracle access one often finds being assumed.

\subsubsection{The \textit{Noiseless} nonparametric regression setting}
\label{s:Main__ss:Reconstruction___sss:noiseless}

When the target function is not obscured by noise, we are in the usual noiseless regression setting, i.e.\ when $
    Y_n=f(X_n)
$ for every $n\in[N]_+$, while $f$ remains a general Lipschitz function and need not belong to our class of neural networks. In this case, no sample splitting or neighbourhood-averaging step is required.  Instead, we apply the central McShane- Whitney~\cite{McShane1934ExtensionRangeFunctions,Whitney1934AnalyticExtensions} formula directly to the raw pairs $(X_n,Y_n)_{n=1}^N$.  More generally, for
$
    M\in\mathbb N_+
$
and $\theta
    =
    \big(
        (a_m,b_m,c_m)
    \big)_{m=1}^M
    \in
    \big(
        \mathbb R\times X\times[0,\infty)
    \big)^M$,  define
\begin{equation}
\label{eq:formula}
\begin{aligned}
    \hat{f}_{\theta}(x)
    &\eqdef
    \tfrac{1}{2}
    \Bigg(
        \min_{m\in[M]_+}
        \big(
            a_m+u_m(x)
        \big)
        +
        \max_{m\in[M]_+}
        \big(
            a_m-u_m(x)
        \big)
    \Bigg),
    \\
    u(x)
    &\eqdef
    \big(
        c_m\rho(x,b_m)
    \big)_{m=1}^M.
\end{aligned}
\end{equation}
This is the central midpoint of the McShane~\cite{McShane1934ExtensionRangeFunctions} and Whitney~\cite{Whitney1934AnalyticExtensions} extension formulas, and coincides with the estimator in~\cite[Lemma~7]{luxburg2004distance} when its mixing constant is set to $\tfrac12$.%
\footnote{The central midpoint minimizes, pointwise, the maximum distance to the two envelope values and therefore incurs one half of their gap under the corresponding worst-case
criterion.}
For the noiseless training dataset, we take
$
    M=N
$
and
\[
    \theta
    =
    \big(
        (Y_n,X_n,L)
    \big)_{n=1}^N.
\]

\begin{theorem}[Regular Uniform Approximation$\tilde{O}(\varepsilon^{-2(s+1)})$ Samples Suffice]
\label{thrm:random_apprx_noisy}
If $f:X\to \mathbb{R}$ is $0\le L$-Lipschitz and if Assumption~\ref{assumption:main_1Lip} holds, then: 
for every $\varepsilon>0$, 
and each $f\in \operatorname{Lip}((X,\rho);1)$ and every
$\delta \in(0,1)$, if
$% \begin{equation}
% \label{eq:required_no_samples_noisy_Ahlfors}
    N
    \ge
    \frac{1}{a}
    \left(\frac{4}{\varepsilon}\right)^s
    \log\left(
        2
    \frac{N_{\varepsilon/4}(X,\rho)}{
        \delta
    }\right)
$
then, we have
\[
    \mathbb{P}\biggl(
        \sup_{x\in X}\,
        \big|
            f(x)
            -
            \hat{f}_{\theta}
            % \mathcal{E}(f|\mathbb{D}_N)
            (x)
        \big|
        <
        \varepsilon
    \biggr)
\ge
    1-\delta
\]
with $M=N$ and $\theta=\big((Y_m,X_m,L)\big)_{m=1}^M$.
\end{theorem}
We emphasize that the reconstruction in Theorem~\ref{thrm:random_apprx_noisy} is \textit{strictly stronger} than a risk-bound considered in classical statistics; i.e.\ that the empirical mean of the loss and the true mean of the loss evaluated at the ERM are close; rather we show that $f$ is uniformly close to the target.

\subsubsection{The classical approximation setting: Oracle Access to the target function}
\label{s:Main__ss:Reconstruction___sss:oracleaccess}

We place ourselves within the classical neural-network approximation theoretic-level of flexibility; where we (the user) can choose our model parameters freely and, perhaps more dramatically, we have ``oracle access'' to the target function's values at any set of points which we want to query for free; cf.~\cite{hornik1991approximation,cybenko1989approximation,pinkus1999approximation,ismailov2014approximation,yarotsky2018optimal,LuydmillaJP_UAT_2018,petersen2018optimal,kidger2020universal,papon2022universal,siegel2023optimal,gonon2023approximation,marcati2023exponential,hong2024bridging,gonon2025universal,schneider2026nonlocal,cuchiero2026global}.  
Thus, under these standard oracle-access assumptions made in constructive approximation theory we obtain the following optimal approximation for the Lipschitz class
\[
     \operatorname{Lip}((X,\rho),[-1,1];1)
    \eqdef 
    \{
        f:(X,\rho)\to ([-1,1],|\cdot|):\, 
        \mbox{ where $f$ is at-most 1-Lipschitz}
    .
    \}
\]
If $(X,\rho)$ is 
$d$-doubling then, we write $\mathcal{F}_{1:d}$.
\begin{proposition}[Approximation of Lipschitz Functions]
\label{prop:approximation}
Let $(X,\rho)$ be a compact $0<d$-doubling space of diameter at-most $1$; then, there are constants
$
    0<C_d<\infty
$\footnote{Depending only on $(X,\rho,\mathfrak{m})$.}, such that: for every $2\le N \in \mathbb{N}$ there are $\{x_n\}_{n=1}^N\subseteq X$ so that: for every $f\in\mathcal F_{1:d}$, the parameter vector
$
    \theta_f
\eqdef
    \big((f(x_n),x_n,1)\big)_{n=1}^N
$
satisfies
\begin{equation}
\label{eq:optimal_approximation}
    \|f-\hat{f}_{\theta_f}\|_{C(X)}
\le
    C_d\,N^{-1/d}
\end{equation}
\end{proposition}

\subsection{\texttt{ReLU}-Neural Network Realizations over $\mathbb{R}^d$}

When $(X,\rho)=(\mathbb{R}^d,\|\cdot\|_p)$ for $p\in \{1,\infty\}$ then~\eqref{eq:formula} is realizable by a \texttt{ReLU}-MLP of depth $\mathcal{O}(\log(N))$, width $\mathcal{O}(N)$, and with $\mathcal{O}(N)$ non-zero neurons.  Thus~\eqref{eq:formula} is much more general than a \texttt{ReLU}-MLP as it can is compatible with inputs in general metric spaces.  Indeed, given any parameter $\theta$ it is enough to run Algorithm~\ref{alg:relu_envelope_realization} to obtain a \texttt{ReLU}-MLP realization; in the case where $p=1$.

\begin{algorithm}[H]
\caption{\texttt{ReLU}-MLP realization of $\hat{f}_\theta(\cdot|G^{(\beta)})$}
\label{alg:relu_envelope_realization}
\KwIn{$\theta=((a_m,b_m,c_m))_{m=1}^M$, $G^{(\beta)}$}

$K\gets\lceil\log_2 M\rceil$\;
$V\gets aG^{(\beta)}$
\tcp*[r]{Local averaging}

$\operatorname{Min}_2(s,t)
\eqdef
\operatorname{ReLU}(t)
-\operatorname{ReLU}(-t)
-\operatorname{ReLU}(t-s)$\;
$\operatorname{Max}_2(s,t)
\eqdef-\operatorname{Min}_2(-s,-t)$\;

\For{$m\in[M]_+$ \emph{in parallel}}{
    $D_m(x)\gets
    \mathbf{1}_d^\top
    \big(
        \operatorname{ReLU}(x-b_m)
        +
        \operatorname{ReLU}(b_m-x)
    \big)$
    \tcp*[r]{$\ell^1$ distance}
    $U_m^{(0)}\gets V_m+c_mD_m(x)$\;
    $L_m^{(0)}\gets V_m-c_mD_m(x)$\;
}

\If{$M<2^K$}{
    \For{$r=M+1,\ldots,2^K$}{
        $U_r^{(0)}\gets U_M^{(0)}$,\quad
        $L_r^{(0)}\gets L_M^{(0)}$
        \tcp*[r]{Pad to $2^K$}
    }
}

\For{$\ell=1,\ldots,K$}{
    \For{$j\in[2^{K-\ell}]_+$ \emph{in parallel}}{
        $U_j^{(\ell)}
        \gets
        \operatorname{Min}_2
        \big(U_{2j-1}^{(\ell-1)},U_{2j}^{(\ell-1)}\big)$\;
        $L_j^{(\ell)}
        \gets
        \operatorname{Max}_2
        \big(L_{2j-1}^{(\ell-1)},L_{2j}^{(\ell-1)}\big)$
        \tcp*[r]{Envelope reduction}
    }
}

\KwOut{$\displaystyle
\Phi_\theta(x)\gets
\tfrac12\big(U_1^{(K)}+L_1^{(K)}\big)$}
\end{algorithm}

% \begin{proposition}[\texttt{ReLU}-MLP Version on $\mathbb{R}^d$]
% \label{prop:representation}
% Let $d\in\mathbb{N}_+$ and $(X,\rho)=([0,1]^d,\|\cdot\|_1)$. Then $\hat{f}$ can be realized exactly by a \texttt{ReLU}-MLP $\Phi:\mathbb{R}^d\to\mathbb{R}$ of depth $\mathcal{O}(\log(N))$, width $\mathcal{O}(N)$, and with $\mathcal{O}(N)$ non-zero parameters.
% \end{proposition}
\begin{proposition}[Algorithmic Sparse \texttt{ReLU}-MLP Realization]
\label{prop:representation}
Let $d,M\in\mathbb N_+$, let
$
    (X,\rho)=([0,1]^d,\|\cdot\|_1),
$
let $G^{(\beta)}\in\mathbb{R}^{M\times M}$ be fixed, and let
$
    \theta=((a_m,b_m,c_m))_{m=1}^M
    \in
    (\mathbb{R}\times X\times[0,\infty))^M.
$
% \hfill\\
% \noindent
Then Algorithm~\ref{alg:relu_envelope_realization} constructs a
\texttt{ReLU}-MLP
$
    \Phi_{\theta,G^{(\beta)}}:\mathbb{R}^d\to\mathbb{R}
$
of depth $\mathcal{O}(\log M)$, width $\mathcal{O}(M)$, and
$\mathcal{O}(M)$ nonzero parameters\footnote{The hidden constants depend only
on $d$.}, satisfying
\begin{equation}
\label{eq:realization}
    \Phi_{\theta,G^{(\beta)}}(x)
    =
    \hat{f}_\theta(x\mid G^{(\beta)})
\end{equation}
for every $x\in X$, where $\hat{f}_\theta$ is given by
\eqref{eq:formula__realizable}.
\end{proposition}

Now, applying a mild variant (Proposition~\ref{prop:transformerification__SparseVersion}) of the MLP-t-transformer conversion trick in~\cite[Proposition A.1]{kratsios2026adaptivity} we directly deduce the following.

\begin{cor}[Sparse Multi-head \texttt{ReLU}-Transformer Realization]
\label{cor:transformer_realization}
Let $d,M\in\mathbb N_+$, let
$
    (X,\rho)=([0,1]^d,\|\cdot\|_1),
$
let $G^{(\beta)}\in\mathbb{R}^{M\times M}$ be fixed, and let
$
    \theta=((a_m,b_m,c_m))_{m=1}^M
    \in
    (\mathbb{R}\times X\times[0,\infty))^M.
$
\hfill\\
\noindent
For every inverse-temperature parameter $\lambda>0$ and every prescribed
number of attention heads $H\in\mathbb N_+$, there exists an $H$-head
\texttt{ReLU}-transformer
$
    \hat{T}:\mathbb{R}^d\to\mathbb{R}
$
of depth $\mathcal{O}(\log M)$, width $\mathcal{O}(M+H)$, and
$\mathcal{O}(M)$ nonzero parameters such that
\[
    \hat{T}(x)
    =
    \hat{f}_\theta(x\mid G^{(\beta)})
\]
for every $x\in X$, where $\hat{f}_\theta$ is given by
\eqref{eq:formula__realizable}.
\end{cor}

Since the transformer statement is obtained by formally embedding the preceding MLP realization, rather than through an intrinsically attention-based computation, we continue the analysis using the more basic and better studied MLP formulation.

\subsubsection{Comparison: Complexity of Our Formula Vs.\ Exact KRR}
\label{s:Discussion_ss:OracleAccess__sss:Complexity_vsKRR}
We close this section by noting that when $(X,\rho)=([0,1]^d,\|\cdot\|_{\infty})$ for some $2\le d\in \mathbb{N}_+$, then our formula, cf.~\eqref{eq:formula__realizable}, requires $\mathcal{O}(N^2)$ storage and costs $\mathcal{O}(N^2)$ to evaluate on the forward pass, in the worst case, due to the neighbourhood matrix $G^{(\beta)}$.  Further, a $d$-dimensional range tree for the $N$ sample points can be constructed in $\mathcal{O}(N\log^{d-1}(N))$ time; see~\cite[Theorem~5.9]{deberg2008computational}.  The neighbourhood graph can then be reported output-sensitively, with an additional dependence on its number of edges, which is $\mathcal{O}(N^2)$ in the worst case; cf.~\cite[Exercise~5.10 (b)]{deberg2008computational}.

A natural comparator is the exact closed-form formula for kernel ridge regression; e.g.~\cite[Equation~(2)]{cheng2024comprehensive}, which requires solving a linear system involving an $N\times N$ kernel matrix.  The algebraic complexity of this operation is $\mathcal{O}(N^p)$ for some $2\le p<2.372$; cf.~\cite{dupont2026improving}, and it remains open whether $p=2$.  Thus, if $p>2$, our formula has a strictly smaller worst-case asymptotic exponent.  In standard numerical implementations, dense linear systems are typically solved using factorizations such as LU, whose classical complexity is $\mathcal{O}(N^3)$; cf.~\cite[Theorem~2.23]{lu2024numerical}.

\paragraph{Comparison with Inexact KRR via Mini-batch GD}
It is, nevertheless, worth noting that KRR can be approximated iteratively.  Since its regularized objective is strongly convex with Lipschitz gradient, full gradient descent converges geometrically to its unique minimizer~\cite[Theorem~2.1.15]{nesterov2018lectures}; hence, for fixed condition number, an $\varepsilon$-accurate solution requires $\mathcal{O}(N^2\log(1/\varepsilon))$ operations.  
Likewise, mini-batch variants with batch size $B$ can reduce the complexity of each gradient update from $\mathcal{O}(N^2)$ to $\mathcal{O}(BN)$, although generally at the cost of requiring more iterations.

\section{Optimality Guarantees}
\label{s:Main__ss:Optim}

In this section, we examine the optimality of our formula using different lenses.  We first consider the standard approximation theory for deep learning vantage point of its parameter space.  Next, we consider a more learning theoretic lens by analyzing the combinatorial complexity of the function space (hypothesis class) which it realizes.  Finally, we consider a modern numerical analysis perspective and study the stability of its parameter-to-realization map.  In each case, we state the precise benchmark under which our formula is optimal or near-optimal, possibly up to logarithmic factors.  We choose this order of presentation, since it organizes these concepts by what we feel is the complexity of their definitions.

% \subsection{Optimal Lipschitz Regularity of the Forward Pass}
% \label{s:Main__ss:Optima___sss:Regularity}

% Our first result guarantees that \textit{every} parameter specification of our formula is maximally regular on the forward pass. That is, its Lipschitz constant is prescribable and, when prescribed to match that of the target function, the resulting estimator inherits the same Lipschitz regularity.
% \begin{proposition}[Optimal Lipschitz Regularity]
% \label{prop:formula_lipschitz}
% In the setting of Theorem~\ref{thrm:nonparametric_noisy_case}, for any parameter
% $\theta=((a_m,b_m,L))_{m=1}^M\in(\mathbb R\times X\times\{L\})^M$,
% the map
% $
% \hat{f}_\theta(\cdot|G^{(\beta)})
% :(X,\rho)\to(\mathbb{R},|\cdot|)
% $
% is $L$-Lipschitz.
% \end{proposition}
\subsection{Prescribed Lipschitz Regularity of the Forward Pass}
\label{s:Main__ss:Optima___sss:Regularity}

Our first result gives a prescribed upper bound on the forward-pass regularity of every parameter specification of our formula.  Specifically, when the envelope slope parameter is $L$, the resulting estimator has Lipschitz constant at most $L$.
\begin{proposition}[Prescribed Lipschitz Regularity]
\label{prop:formula_lipschitz}
In the setting of Theorem~\ref{thrm:nonparametric_noisy_case}, for any parameter $
    \theta
    =
    ((a_m,b_m,L))_{m=1}^M
    \in
    (\mathbb R\times X\times\{L\})^M,
$
the map $
    \hat{f}_\theta(\cdot\mid G^{(\beta)})
    :
    (X,\rho)
    \to
    (\mathbb{R},|\cdot|)
$ satisfies
\[
    \operatorname{Lip}
    \big(
        \hat{f}_\theta(\cdot\mid G^{(\beta)})
    \big)
    \le
    L.
\]
\end{proposition}

While there exist specialized neural network constructions, for networks with inputs in $[0,1]^d$, with controlled Lipschitz constants~\cite{meunier2022dynamical,araujo2023a}, as well as universal approximation theorems for $1$-Lipschitz \texttt{ReLU}-MLPs~\cite{hong2024bridging,riegler2024generating} and ResNets~\cite{murari2026approximation}, most neural network classes do not have uniformly controlled Lipschitz constants over all parameter choices. For instance, the ``universal'' \texttt{ReLU}-MLP approximators in~\cite{petersen2018optimal}, or the memorizers in~\cite{vardi2021optimal}, have weights that diverge as a function of the approximation error, suggesting that their Lipschitz constants may likewise diverge.

\begin{rmk}
\label{rmk:EmpLipschitz}
In practice, it is convenient to set $L$ in Proposition~\ref{prop:formula_lipschitz} to the empirical Lipschitz constant of the noisy training data, defined for $N>1$ by
\begin{equation}
\label{eq:empirical_Lipschtiz}
    L_{\mathcal{D}}
\eqdef
    \max\limits_{\underset{n\neq m}{n,m\in[N]_+}}
    \frac{|Y_n-Y_m|}{\rho(X_n,X_m)}.
\end{equation}
For $N=1$, we set $L_{\mathcal{D}}\eqdef0$.
\end{rmk}

\subsection{Parametric Complexity of the \texttt{ReLU}-MLP Realization \texorpdfstring{on $\mathbb{R}^d$}{}}
\label{s:Main__ss:Optim___sss:ParametricOptimality_Rd}

We now unpack the complexity-theoretic implication of Proposition~\ref{prop:representation} \textit{parametric complexity} explanation of our formula in the case of \texttt{ReLU}-MLP realizations in the context of approximation and statistical learning literatures.  
First, we notice that the \texttt{ReLU}-MLP realization in Proposition~\ref{prop:representation} is sparse: a comparably wide fully connected architecture of depth $\mathcal{O}(\log N)$ would generally contain up to $\mathcal{O}(N^2\log N)$ parameters, whereas our realization uses only $\mathcal{O}(N)$ non-zero parameters; cf.~\cite{bolcskei2019optimal}. We thus automatically deduce the following uniform approximation guarantee with $\mathcal{O}(N)$ non-zero weights and optimal regularity in the $\ell^1$ and $\ell^\infty$ norms. Since the number of non-zero weights grows at most linearly in $N$, the resulting $\mathcal{O}(N^{-1/d})$ approximation rate matches the optimal continuous-weight-assignment exponent of~\cite[Proposition~1 and Theorem~1(b)]{yarotsky2018optimal}.
\hfill\\
\noindent
In particular, the exponent $1/d$ is precisely the boundary beyond which the faster $\mathcal{O}(N^{-q})$ rates, $1/d<q\le2/d$, require discontinuous weight assignment in the framework of~\cite[Theorem~2]{yarotsky2018optimal}; cf.~\cite{ShenYangZhang_JMPA_OptApprx_ReLU}.
\begin{cor}[Sparse $\ell^p$-Regular \texttt{ReLU}-MLP Realization: Approximation]
\label{cor:lp_relu_approximation}
Let $d\in\mathbb{N}_+$, $L>0$, $1\le p\le\infty$, and $X=[0,1]^d$, and fix an $L$-Lipschitz function $f:([0,1]^d,\|\cdot\|_p)\to(\mathbb{R},|\cdot|)$. For every $N\ge2$, there exists a \texttt{ReLU}-MLP $\Phi:\mathbb{R}^d\to\mathbb{R}$ of depth $\mathcal{O}(\log N)$, width $\mathcal{O}(N)$, and with $\mathcal{O}(N)$ non-zero parameters such that:
\begin{enumerate}
    \item[(i)] \textbf{Uniform Approximation:} 
    \[
    \sup\limits_{x\in[0,1]^d}|f(x)-\Phi(x)|
    \le 
    \frac{
        C'_{p,d}
        \,
        L
        }{
        N^{1/d}
        }
    ;
    \]
    \item[(ii)] \textbf{Regularity:} $\Phi$ is $L\,d^{\min\{1/p,\,1-1/p\}}$-Lipschitz with respect to $\|\cdot\|_p$,
\end{enumerate}
where $
C'_{p,d}\eqdef d
$
if $1\le p\le2$ and
$
C'_{p,d}\eqdef d^{1/p}
$
if $2<p\le\infty$.
\end{cor}

We emphasize that Corollary~\ref{cor:realizability} requires neither oracle access to $f$ at freely chosen locations nor optimization over the network weights: the \texttt{ReLU}-MLP is constructed explicitly from noisy i.i.d.\ samples, unlike the preceding deterministic approximation results where the parameters may depend on oracle evaluations of $f$.

\paragraph{Comparison to Minimax Statistical Guarantees for (Approximate) ERMs}
In this setting, the minimax sup-norm rate is
$
    (\log N/N)^{1/(d+2)}
$
by~\cite{Stone1982OptimalGlobalRates}, whereas our construction yields
$
    \|\hat{f}-f\|_\infty
    =
    \mathcal{O}_{\mathbb P}\!\big(
        N^{-1/(d+2)}\sqrt{\log N}
    \big).
$
Thus, we recover the optimal polynomial exponent up to logarithmic factors. This is also comparable, at the level of the polynomial exponent, to the sparse \texttt{ReLU}-MLP estimators of~\cite{schmidthieber2020nonparametric}, which satisfy
$
    \mathbb{E}\|\hat{f}-f\|_{L^2}^2
    \lesssim
    N^{-2/(d+2)}\log^3 N
$;
equivalently, a root-$L^2$ rate of order
$
    N^{-1/(d+2)}\log^{3/2}N
$
for the $1$-Lipschitz class.

\begin{cor}
\label{cor:realizability}
In the setting of Theorem~\ref{thrm:nonparametric_noisy_case}, there exists a $1$-Lipschitz \texttt{ReLU}-MLP $\Phi:\mathbb{R}^d\to \mathbb{R}$ satisfying
\begin{equation}
\label{eq:relu_robustified_uniform_bound}
    \mathbb{P}\left(
        \sup_{x\in X}
        \big|
            f(x)-\Phi(x)
        \big|
        \le
        \varepsilon
        \left(
            1+
            \sqrt{
                C
                \log\left(
                    \frac{2N_0}{\delta}
                \right)
            }
        \right)
    \right)
    \ge
    1-\delta.
\end{equation}
\end{cor}

Unlike the preceding approximation-theoretic results,
Corollary~\ref{cor:realizability} requires neither oracle evaluations of
$f$ at freely chosen locations nor optimization over the network weights.
After the explicit neighbourhood-averaging preprocessing step, the effective
\texttt{ReLU}-MLP weights are assigned directly from the observed sample.
In particular, the guarantee does not assume that an empirical-risk
minimizer, or a sufficiently accurate approximate minimizer, has been
identified; it holds directly for the sparse \texttt{ReLU}-MLP realization
of the estimator in~\eqref{eq:formula__realizable}.

\subsection{Complexity: Realized Function Class}
\label{s:Main__ss:Optim___sss:Pseudodimension}
We now evaluate the combinatorial complexity of the \textit{set} of functions \textit{realized} by our formula. We recall that the fat-shattering dimension, introduced in~\cite{kearns1994efficient}, is a scale-sensitive analogue of the Vapnik-Chervonenkis (VC) dimension for real-valued function classes. Finiteness of the fat-shattering dimension at every positive scale characterizes uniform learnability~\cite{alon1997scale,bartlett1996fat}. Intuitively, it measures the largest set of inputs on which every binary labelling can be realized by functions in the class while maintaining a prescribed margin $\gamma>0$ in function value.

\begin{definition}[$\gamma$-Fat-Shattering Dimension]
Let $X$ be a set, let $\mathcal{F}\subseteq\mathbb{R}^{X}$, and fix $\gamma>0$. A finite set $\{x_1,\ldots,x_m\}\subseteq X$ is said to be $\gamma$-\emph{fat shattered} by $\mathcal{F}$ if there are $t_1,\ldots,t_m\in\mathbb{R}$ such that: for every $\sigma=(\sigma_1,\ldots,\sigma_m)\in\{-1,1\}^m$, there is an  $f_\sigma\in\mathcal{F}$ satisfying
\[
    \sigma_j(f_\sigma(x_j)-t_j)\ge\gamma
\]
for every $j\in[m]_+$. The $\gamma$-fat-shattering dimension $ \operatorname{fat}_{\gamma}$ of $\mathcal{F}$ is the largest such $m$.
\end{definition}

Fix a regularity parameter $L>0$ and a model-size parameter
$M\in\mathbb{N}_+$.  We define the $L$-regular parameter space by
\begin{equation}
\label{eq:regular_parameter_space}
    \Theta_{M:L}
    \eqdef
    \left(
        \mathbb{R}\times X\times\{L\}
    \right)^M.
\end{equation}
For every
$
    \theta=((a_m,b_m,L))_{m=1}^M\in\Theta_{M:L},
$
define
\[
    \hat{f}_\theta(x)
    \eqdef
    \frac{1}{2}
    \left[
        \min_{m\in[M]_+}
        \left\{
            a_m+L\rho(x,b_m)
        \right\}
        +
        \max_{m\in[M]_+}
        \left\{
            a_m-L\rho(x,b_m)
        \right\}
    \right],
    \qquad x\in X.
\]
We denote the corresponding central-envelope class by
\begin{equation}
\label{eq:regular_parameter_space__realization}
    \mathcal F_{M:L}
    \eqdef
    \left\{
        \hat{f}_\theta
        :
        \theta\in\Theta_{M:L}
    \right\}.
\end{equation}
For the split-sample estimator in
Theorem~\ref{thrm:nonparametric_noisy_case}, the effective parameter is
\[
    \theta^{(\beta,q_\varepsilon,1)}
    =
    \left(
        \left(
            V_k^{(\beta,q_\varepsilon)},X_k,1
        \right)
    \right)_{k=1}^{N_0},
\]
and therefore
$
    \hat{f}_{\mathbb{D}_N}^{(\beta,q_\varepsilon,1)}
    =
    \hat{f}_{\theta^{(\beta,q_\varepsilon,1)}}
    \in
    \mathcal F_{N_0:1}$.

The following result gives a scale-sensitive analogue of the classical VC-dimension obstruction to uniform approximation used, for example, in~\cite[Theorem~2.4]{shen2022optimal}: uniform approximation of a Lipschitz ball at resolution $\varepsilon$ necessarily requires combinatorial complexity of order $\varepsilon^{-d}$ at the same resolution.

\begin{theorem}[Optimal Fat-Shattering Dimension at the Approximation Scale]
\label{thrm:fatshatteringotimality}
Let $(X,\rho,\mathfrak{m})$ be a compact $d$-Ahlfors regular metric measure space of diameter $1$, where $d>0$, and let $B,L>0$. There are constants $\varepsilon_0,c,C>0$
% \footnote{Depending only on $(X,\rho,\mathfrak m)$}
, such that: 
for any
$
    \mathcal F
    \subseteq
    \operatorname{Lip}((X,\rho),\mathbb{R};L)
$
and any
$
    0<\varepsilon<\min\{B/2,L\varepsilon_0\},
$
if
\begin{equation}
\label{eq:eps_approx__minimaxformulation}
    \sup_{f\in\operatorname{Lip}((X,\rho),[-B,B];L)}
    \inf_{\hat{f}\in\mathcal F}
    \|f-\hat{f}\|_\infty
\le
    \varepsilon,
\end{equation}
then
\begin{equation}
\label{eq:tight_psuedodimension}
    c
    \left(
        \frac{L}{\varepsilon}
    \right)^d
\le
    \operatorname{fat}_{\varepsilon}(\mathcal F)
\le
    C
    \left(
        \frac{L}{\varepsilon}
    \right)^d.
\end{equation}
In particular, $
    \mathcal{F}_{M:L}
\subseteq
    \operatorname{Lip}((X,\rho),\mathbb{R};L)$ for any $M\in \mathbb{N}_+
$ and there is a constant $C_d'>0$%
\footnote{Depending only on $(X,\rho,\mathfrak m)$.}%
, such that~\eqref{eq:eps_approx__minimaxformulation} holds for $\mathcal F_{M:L}$ whenever
$
    M
    \ge
    C_d'
    (L/\varepsilon)^d.
$
\end{theorem}
\begin{proof}
See Section~\ref{s:FatShattering__ss:LB}.
\end{proof}

\subsection{Complexity: Parameter-to-Realization}
\label{s:Main__ss:Optim___sss:Stability}
There are several notions of optimality in the constructive approximation literature, ranging from linear complexity measures, such as Kolmogorov widths, linear widths, and related $n$-widths~\cite{pinkus1985nwidths}, to sampling and information complexity~\cite{novak2008tractability}, nonlinear dictionary-based notions such as best $N$-term approximation~\cite{devore1998nonlinear}, and fully nonlinear notions such as metric entropy~\cite{kolmogorov1959entropy}. A further line of work studies nonlinear continuous encodings through the ``manifold widths'' introduced in~\cite{devore1989optimal,devore1993wavelet} and recently refined in~\cite{cohen2022optimal,petrova2023lipschitz}.

In particular, the Lipschitz widths of~\cite{petrova2023lipschitz} quantify the smallest worst-case approximation error achievable by representing a model class $\mathcal K$ as the image of an $n$-dimensional parameter ball under a $\gamma$-Lipschitz realization map. Equivalently, once the approximating $\gamma$-Lipschitz $n$-parameter family has been fixed, an adversary selects the element of $\mathcal K$ that is worst approximated by that family.

\begin{definition}[{Lipschitz width;~\cite[Equations (2.1)-(2.3]{petrova2023lipschitz}}]
\label{def:Lipschitz_width}
Let $(\mathcal{X},\|\cdot\|_{\mathcal{X}})$ be a Banach space, let
$\mathcal K\subseteq \mathcal{X}$ be bounded, let $n\in\mathbb N_+$, and let
$\gamma\ge 0$.  
For each fixed norm $\|\cdot\|_{Y_n}$ on $\mathbb{R}^n$; the $\|\cdot\|_{Y_n}$-Lipschitz width of $\mathcal{K}$ is defined\footnote{We denote the  closed unit ball
$
    B_{Y_n}
\eqdef
    \{y\in\mathbb{R}^n:\|y\|_{Y_n}\le 1\}
$ if a normed space $(\mathbb{R}^n,\|\cdot\|_{Y_n})$.}~%
by
\begin{equation}
\label{eq:def_LipWidth}
    d^\gamma(\mathcal K,Y_n)_{\mathcal{X}}
\eqdef
    \inf_{\Phi}
        \sup_{f\in\mathcal K}
            \inf_{y\in B_{Y_n}}
                \|f-\Phi(y)\|_{\mathcal{X}},
\end{equation}
where the outer infimum in~\eqref{eq:def_LipWidth} is taken over all $\gamma$-Lipchitz maps
$
    \Phi:B_{Y_n}\to\mathcal{X}
$.
The $n$-th Lipschitz width of $\mathcal K$ is obtained by further minimizing over all norms $\|\cdot\|_{Y_n}$ on $\mathbb{R}^n$; i.e.\
\begin{equation}
\label{eq:def_LipWidth_n}
    d^\gamma_n(\mathcal K)_{\mathcal{X}}
\eqdef
    \inf\big\{
        d^\gamma(\mathcal K,Y_n)_{\mathcal{X}}
    :
    \,
        \|\cdot\|_{Y_n} 
            \text{ a norm on } 
        \mathbb{R}^n
    \big\}
.
\end{equation}
\end{definition}

Before proving the optimality of our scheme, we clarify the benchmark against
which optimality is measured.  
We consider the Lipschitz-width lower bounds, which impose stability
of the parameter-to-function realization map and therefore provide a numerically
reliable notion of nonlinear optimality in constructive approximation
theory~\cite{petrova2023lipschitz,cohen2022optimal}.
\begin{theorem}[Stable near-minimax optimality in the oracle regime]
\label{thm:stable_near_minimax_optimality_oracle_regime}
Let $(X,\rho,\mathfrak{m})$ be a compact $0<d$-Ahlfors regular metric measure space of diameter at-most $1$.
Let $\gamma_d \eqdef 1 + \sqrt d$. 
Then, 
the map $([-1,1] \times X \times [0,1])^N\mapsto \hat{f}_{\theta} \in C(X)$ is $\gamma_d$-Lipschitz
and there are constants 
$
    0<c_d'\le C_d<\infty
$\footnote{Depending only on $(X,\rho,\mathfrak{m})$.}, such that: for every $2\le N \in \mathbb{N}$ there are $\{x_n\}_{n=1}^N\subseteq X$ so that:
\[
    % c_d'
    % \,
    N^{-1/d}
    \,
    \log_2(N)^{-1/d}
\lesssim_{c_d'}
    d^{\gamma_d}_{P_N}(\mathcal F_{1:d})_{C(X)}
\le
    \sup_{f\in\mathcal F_{1:d}}
        \|f-\hat{f}_{\theta_f}\|_{C(X)}
\lesssim_{C_f}
    % C_d
    % \,
    % P_
    N^{-1/d}
% .
\]
where $\hat{f}_{\theta_f}$ is as in Proposition~\ref{prop:approximation}.
\end{theorem}
Theorem~\ref{thm:stable_near_minimax_optimality_oracle_regime} shows that~\eqref{eq:formula} attains the optimal Lipschitz-width rate up to the logarithmic factor $(\log_2(N))^{1/d}$.

\subsubsection{Discussion: Other Notions of Optimal Nonlinear Approximation}
\label{s:Optimality__ss:Approx}
Before moving on, it is very much worth interpreting Theorem~\ref{thm:stable_near_minimax_optimality_oracle_regime} formulated for Lipschitz widths of~\cite{petrova2023lipschitz} with other modern notions of non-linear approximation.

\paragraph{Comparison with Manifold Widths}
% \hfill\\
This notion is closely related to the stable manifold widths of~\cite{cohen2022optimal}, which quantify how well all elements of $\mathcal K$ can be stably encoded and reconstructed through an $n$-dimensional Banach space $(\mathbb R^n,\|\cdot\|_{Y_n})$ using a $\gamma$-Lipschitz encoder-decoder pair. The key distinction is that Lipschitz widths require stability only of the realization map from parameters to functions, whereas stable manifold widths require stability of both the encoder and decoder. Moreover, modified stable manifold widths are sufficiently fine to characterize the bounded approximation property, cf.~\cite{grothendieck1955produits,figiel1973approximation,enflo1973counterexample}, in the sense of~\cite[Theorem~2.4]{cohen2022optimal}. Conversely, stable manifold widths are no finer than Lipschitz widths, up to the natural adjustment of the admissible Lipschitz constant, as shown in~\cite[Theorem~5.1]{petrova2023lipschitz}; this comparison may be strict, as demonstrated by~\cite[Theorem~5.3]{petrova2023lipschitz}.

\paragraph{Comparison with Yarotsky's ``Discontinuous Weight Selection'' Rate}
% \hfill\\
The optimality notions in~\cite{yarotsky2018optimal,ShenYangZhang_JMPA_OptApprx_ReLU} are primarily combinatorial or architectural: they show that any neural-network class approximating a prescribed smoothness class must possess sufficient expressive capacity, but do not control the continuity or conditioning of the parameter-to-function realization map. By contrast, Lipschitz widths impose this stability directly and therefore exclude ``\textit{\textbf{unstable} approximation schemes}'' whose rates rely on numerically fragile encodings or ill-conditioned parametrizations.

\section{Applications}
\label{s:Applicatoins}
In this section, we consider a few further applications of our main result; beyond the \texttt{ReLU}-MLP implications mentioned above.

\subsection{Guidance on Choice of Activation Function for Low-Regularity Problems}
\label{s:Arb_regularity}
In many applications, the natural regularity is \emph{lower} than Lipschitz continuity. For instance, Brownian motion has almost surely locally $\alpha$-H\"older paths for every $0<\alpha<\tfrac12$~\cite[Chapter~I, Theorem~2.2]{RevuzYor1999Continuous}, while bounded-vorticity solutions of the two-dimensional Euler equations induce log-Lipschitz velocity fields~\cite{yudovich1963nonstationary}; see~\cite[Lemma~2.1]{shikhkhalil2025loss} for the explicit modulus $\omega(t)\asymp Lt\log(e/t)$.

Our theory naturally handles both through classical metric transforms, originating with Wilson~\cite{Wilson1935MetricTransformations} and closely related to metric snowflaking~\cite{Schoenberg1937MetricSpaces,assouad111plongements,naor2012assouad}, recently used in Neural Snowflake models~\cite{saez2024neural,saez2025neural}. Indeed, for any non-decreasing concave modulus $\omega$ with $\omega(0)=0$, $\rho^\omega\eqdef\omega\circ\rho$ is again a metric, cf.~\cite{mendelnaor2011dichotomies}, and $f$ is $\omega$-uniformly continuous on $(X,\rho)$ if and only if it is $1$-Lipschitz on $(X,\rho^\omega)$. The usual snowflake $\rho^\alpha$ is recovered by taking $\omega(t)=t^\alpha$.  

As an example, we simple oracle-access to $f$-type setting of Proposition~\ref{prop:approximation}.  In that case, after snowflaking with $\omega$, our models in~\eqref{eq:formula__realizable} becomes:
\begin{equation}
\label{eq:formula__realizable___snowflaked}
\begin{aligned}
    \hat{f}_{\theta:\omega}(x|G^{\beta})
     & = 
        \tfrac{1}{2}
        \Big(
            \min_{m \in [M]_+} \big(
                (aG^{(\beta)})_m+u_m
            \big)
            +  
            \max_{m \in [M]_+} \big(
                (aG^{(\beta)})_m-u_m
            \big)
        \Big)
\\
    u & = 
        \big(c_m\,\omega\circ \rho(x,b_m)\big)_{m=1}^M
.
\end{aligned}
\end{equation}
We state the natural consequence of our result here qualitatively for simplicity. 
\begin{cor}[Universal Approximation]
\label{cor:uniformlycontinuous}
Let $(X,\rho)$ be a compact metric space, $\omega$ be a concave, continuous, and strictly increasing modulus of continuity, and let $f:(X,\rho)\to (\mathbb{R},|\cdot|)$ be $\omega$-uniformly continuous.  
%%%
Then, for every $\varepsilon>0$ there exists some $M\in \mathbb{N}_+$, and some $\theta \in [\mathbb{R}\times X\times \times \{1\}]^M$ such that $\hat{f}_{\theta}(\cdot|G^0)$ in~\eqref{eq:formula__realizable___snowflaked} satisfies
\[
    \sup_{x\in X}
    \,
    \big|
        \hat{f}_{\theta}(x|G^0)
        -
        f(x)
    \big|
    <\varepsilon
.
\]
Moreover, $\hat{f}_{\theta}(\cdot|G^0):(X,\rho)\to (\mathbb{R},|\cdot|)$ is $\omega$-uniformly continuous.
\end{cor}

\subsection{Universality of DeepSets with Prescribed Lipschitz Regularity}
\label{s:NonEucExample}
Many machine-learning problems, particularly in science and engineering, involve target functions that are invariant under known symmetries. A central example is permutation invariance of point clouds, where the input is naturally an unordered set rather than an ordered sequence, motivating the DeepSets architecture of~\cite{zaheer2017deep}. DeepSets are known to be universal for continuous permutation-invariant functions under suitable latent-dimensional conditions~\cite[Theorems~15 and~20]{wagstaff2022universal}; see also~\cite{syed2026embedding}. As we will see, a universal and highly structured subclass of DeepSets falls within our setting and admits an explicit realization by our formula. In particular, universality can be achieved while preserving the prescribed Lipschitz upper bound, supplementing existing universality guarantees with explicit control of the approximator's regularity and a closed-form description of what the resulting universal DeepSet computes.

Let $d,K\in\mathbb N_+$ and denote by
$
\mathcal P_{\operatorname{emp}}^{\le K}(\mathbb R^d)
$
the collection of uniform empirical probability measures
\begin{equation}
\label{eq:emp_measure}
    \nu
    =
    \frac{1}{\widetilde K}
    \sum_{k=1}^{\widetilde K}\delta_{z_k},
\end{equation}
where $1\le\widetilde K\le K$ and $z_1,\ldots,z_{\widetilde K}\in\mathbb R^d$ are pairwise distinct. Thus, $\mathcal P_{\operatorname{emp}}^{\le K}(\mathbb R^d)$ may equivalently be viewed as the collection of non-empty point clouds in $\mathbb R^d$ containing at most $K$ points.
% \hfill\\
% \noindent
Fix $F\in\mathbb N_+$ and a \texttt{ReLU}-MLP
$
    \phi:\mathbb R^d\to\mathbb R^F.
$
For $\nu$ as in~\eqref{eq:emp_measure}, define its feature-mean embedding by
\begin{equation}
\label{eq:feature_mean_embedding}
    T_\phi(\nu)
    \eqdef
    \int_{\mathbb R^d}\phi(x)\,\nu(dx)
    =
    \frac{1}{\widetilde K}
    \sum_{k=1}^{\widetilde K}\phi(z_k)
    \in\mathbb R^F.
\end{equation}
Let $
    \mathcal{K}
\subseteq
    \mathcal P_{\operatorname{emp}}^{\le K}(\mathbb R^d)
$, with $2\le \#\mathcal{K}$, and suppose that $T_\phi$ is injective on $\mathcal{K}$. Then
\[
    d_\phi(\nu,\widetilde\nu)
    \eqdef
    \big\|
        T_\phi(\nu)-T_\phi(\widetilde\nu)
    \big\|_\infty,
\]
for each $\nu,\widetilde\nu\in\mathcal{K}$; defines a metric on $\mathcal{K}$.  
Suppose that $\mathcal{K}$ is compact with respect to $d_\phi$. For reference point clouds $
    \nu^{(m)}
    =
    K_m^{-1}\sum_{k=1}^{K_m}\delta_{z_k^{(m)}}
    \in\mathcal{K}
$ and $\beta=0$, formula~\eqref{eq:formula__realizable} takes the form
\begin{equation}
\label{eq:formula__realizable___feature_mean}
\begin{aligned}
    \hat{f}_\theta(\nu|G^0)
& =
    \frac12
    \left(
        \min_{m\in[M]_+}
        \left\{
            a_m+c_m
            \big\|
                T_\phi(\nu)-T_\phi(\nu^{(m)})
            \big\|_\infty
        \right\}
    \right.
\\
    &
    \left.
    +
        \max_{m\in[M]_+}
        \left\{
            a_m-c_m
            \big\|
                T_\phi(\nu)-T_\phi(\nu^{(m)})
            \big\|_\infty
        \right\}
    \right).
\end{aligned}
\end{equation}
Indeed, setting $
    b_m\eqdef T_\phi(\nu^{(m)})\in\mathbb R^F $ and
\[
    \Psi_\theta(z)
    \eqdef
    \frac12
    \left(
        \min_{m\in[M]_+}
        \{a_m+c_m\|z-b_m\|_\infty\}
        +
        \max_{m\in[M]_+}
        \{a_m-c_m\|z-b_m\|_\infty\}
    \right),
\]
we have
\begin{equation}
\label{eq:DeepSets_factorization}
    \hat{f}_\theta(\nu|G^0)
    =
    \Psi_\theta\big(T_\phi(\nu)\big)
    =
    \Psi_\theta\left(
        \frac{1}{\widetilde K}
        \sum_{k=1}^{\widetilde K}\phi(z_k)
    \right).
\end{equation}
Thus,~\eqref{eq:formula__realizable___feature_mean} is a mean-pooled instance of the \textsc{Deep Sets} architecture of~\cite[Section~3.1]{zaheer2017deep}; cf.~\cite[Theorem~2]{zaheer2017deep} for the classical sum-decomposition characterization. Moreover, by the same absolute-value, parallelization, and min/max construction used in Proposition~\ref{prop:representation}; cf.~\cite[Lemmata~5.3 and~5.11]{petersen2024mathematical}, $\Psi_\theta$ admits an exact \texttt{ReLU}-MLP realization of depth
$
    \mathcal{O}(\log F+\log M)
$
and with $\mathcal{O}(MF)$ width and non-zero parameters.
We therefore obtain the following direct consequence of Corollary~\ref{cor:uniformlycontinuous}.

\begin{cor}[Universality of Lipschitz Deep Sets with Prescribed Regularity]
\label{cor:DeepSets}
Let $d,K,F\in\mathbb N_+$, let
$
    \mathcal{K}
    \subseteq
    \mathcal P_{\operatorname{emp}}^{\le K}(\mathbb R^d)
$
be compact with respect to $d_\phi$, and let
$
    \phi:\mathbb R^d\to\mathbb R^F
$
be a \texttt{ReLU}-MLP such that $T_\phi$ is injective on $\mathcal{K}$. 
\hfill\\
\noindent
For every $L>0$, every $L$-Lipschitz map
$
    f:(\mathcal{K},d_\phi)\to\mathbb R,
$
and every $\varepsilon>0$, there exists a \texttt{ReLU}-MLP
$
    \Psi:\mathbb R^F\to\mathbb R
$
such that
\begin{equation}
\label{eq:deep_sets}
    \sup_{
        \nu=
        \frac{1}{\widetilde K}
        \sum_{k=1}^{\widetilde K}\delta_{z_k}
        \in\mathcal{K}
    }
    \left|
        f(\nu)
        -
        \Psi\left(
            \frac{1}{\widetilde K}
            \sum_{k=1}^{\widetilde K}\phi(z_k)
        \right)
    \right|
    <\varepsilon
\end{equation}
and the resulting Deep Set map $\big(
\mathcal P_{\operatorname{emp}}^{\le K}(\mathbb R^d),d_{\phi}\big)\ni 
    \nu
    \mapsto 
    \Psi\big(T_\phi(\nu)\big)
\in (\mathbb{R},|\cdot|)$
is $L$-Lipschitz.
\end{cor}

\section{Numerical Ablations}
\label{s:NumericalIllustrations}

We numerically examine the dependence of the estimator on its bandwidth,
sample size, target regularity, and noise level, and compare its observed
performance with that of selected \texttt{ReLU}-MLPs trained by gradient-based
methods. All code and scripts used to
generate the numerical results and figures reported in this section are publicly available.\footnote{\url{https://github.com/hradghoukasian/closed-form-nn}}

\paragraph{Numerical implementation convention}
The finite-sample guarantee in
Theorem~\ref{thrm:nonparametric_noisy_case} is proved for the split-sample,
fixed-cardinality construction in~\eqref{eq:formula__realizable}: the first
sample block supplies the reference points, and exactly
$q_\varepsilon$ observations from the second block are averaged at each
reference point.

In the experiments, we instead use the natural all-sample implementation.
Every training input acts as a reference point, and every observation in its
$\beta$-neighbourhood contributes to its locally averaged label.  Thus the
same observations are used both as reference points and as averaging
observations, and the number of observations averaged may vary between
neighbourhoods.  The subsequent central McShane-Whitney midpoint operation
is unchanged.
While  Theorem~\ref{thrm:nonparametric_noisy_case} is proved only for the split-sample, fixed-cardinality estimator, the experiments show that the formula can also be successfully applied using the all-sample implementation.  In addition, when $\beta$ is selected by
hold-out validation, the theoretical choice $\beta=\varepsilon/2$ is not
being used.

\subsection{Closed-Form Reconstruction on $\mathbb{R}$}

We first illustrate the noiseless closed-form reconstruction formula in the one
-dimensional setting. The input space is the compact interval
$X=[-1,1]\subset \mathbb{R}$,
equipped with the standard metric
$\rho(x,y)=|x-y|$.
We sample $N_{\mathrm{train}}=100$  or $N_{\mathrm{train}}=50$ training points uniformly from $[-1,1]$ and set
$Y_n=f(X_n)$,
where the target function is
\[
f(x)=e^{-x^2}\cos(5x)+\max\{0,x\}+\min\{0,-x^2\}.
\]
The Lipschitz constant used in the estimator is chosen empirically as
$
L=1.05\,L_{\mathcal{D}}$ and 
$L_{\mathcal{D}}
=
\max_{i\neq j}
\frac{|Y_i-Y_j|}{|X_i-X_j|}$; where, the factor $1.05$ provides a small safety margin.
Given the training data, we evaluate the noiseless version of our formula, i.e.\ with $\beta=0$.
The resulting plot (Figure \ref{fig:closed-form-r1-main}) compares the target function, the lower and upper Lipschitz envelopes, and the closed-form estimator.

\begin{figure}[ht]
    \centering

    \begin{subfigure}{0.48\textwidth}
        \centering
        \includegraphics[width=\textwidth]{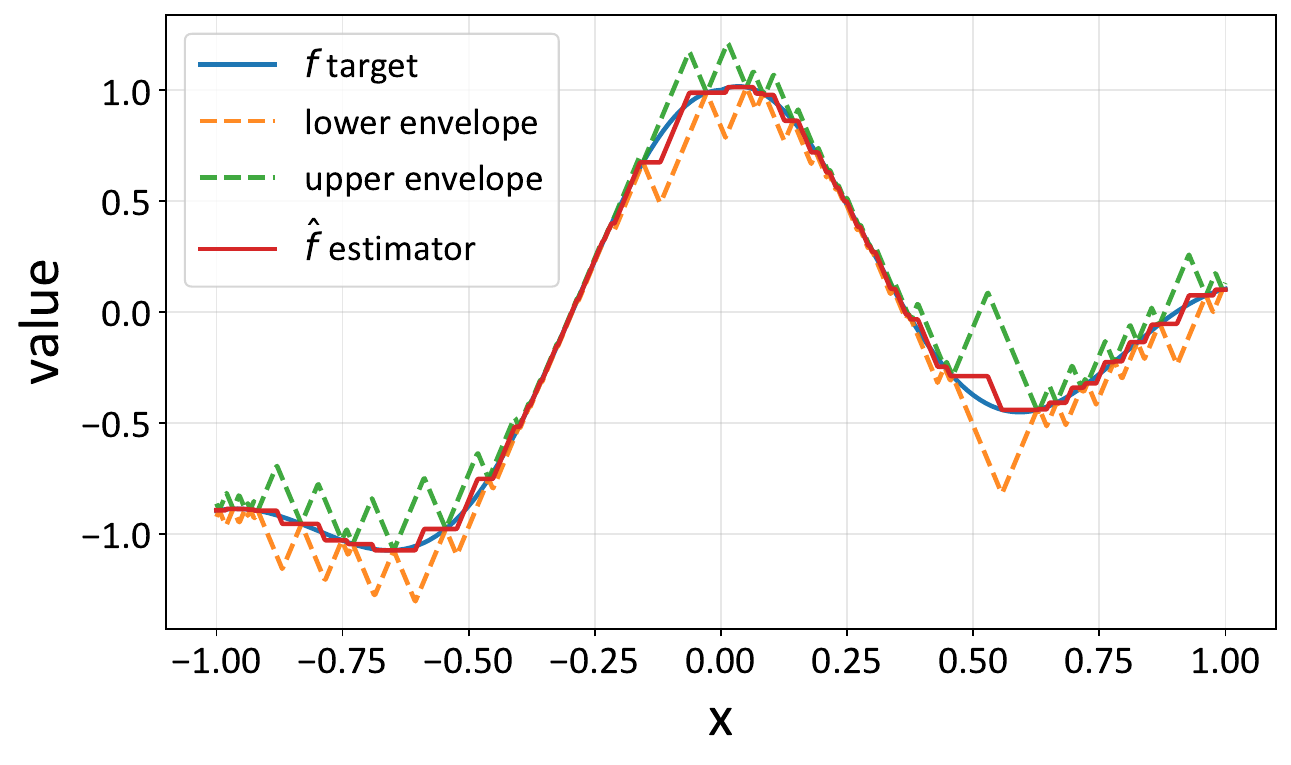}
        \caption{$N_{\mathrm{train}}=50$}
        \label{fig:closed-form-r1-n50}
    \end{subfigure}
    \hfill
    \begin{subfigure}{0.48\textwidth}
        \centering
        \includegraphics[width=\textwidth]{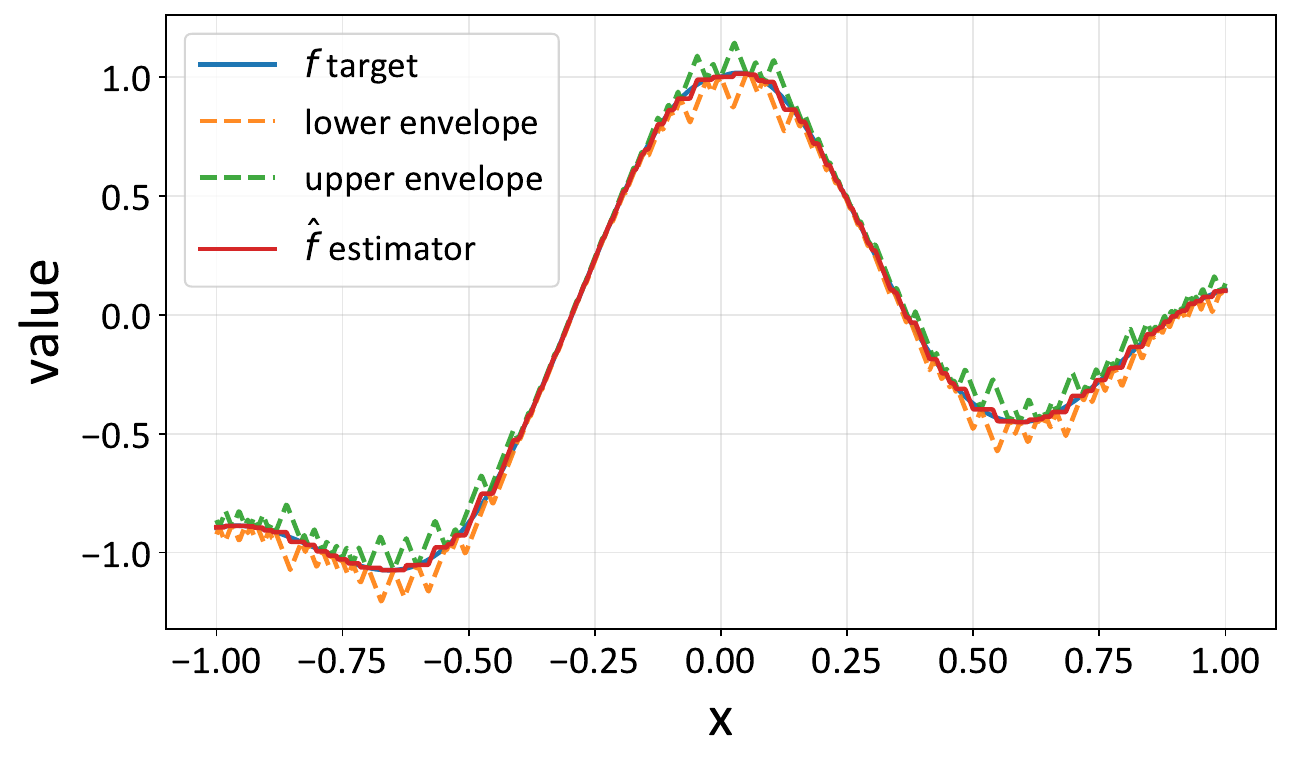}
        \caption{$N_{\mathrm{train}}=100$}
        \label{fig:closed-form-r1-n100}
    \end{subfigure}

    \caption{Closed-form Whitney-McShane midpoint reconstruction on $[-1,1]$ with noiseless samples. The lower and upper envelopes correspond to the Whitney and McShane bounds, and the estimator is their midpoint.}
    \label{fig:closed-form-r1-main}
\end{figure}

\subsection{Closed-Form Reconstruction on $\mathbb{R}^2$}

We next illustrate the closed-form reconstruction formula on a two-dimensional domain. The input space is
$
X=[-1,1]^2\subset \mathbb{R}^2,
$
equipped with the $\ell_1$ metric
$\rho(x,y)=\|x-y\|_1$.
We consider the smooth target function
$
f(x_1,x_2)=\sin(2x_1+x_2).
$
For each training size
$N_{\mathrm{train}}\in\{50,100,1000\},
$
we sample training points uniformly from $[-1,1]^2$ and set $Y_n=f(X_n)$. The Lipschitz constant used in the estimator is chosen as
$
L=1.05\,L_{\mathcal{D}}
$ where $L_{\mathcal{D}}$ is as in Remark~\ref{rmk:EmpLipschitz}.   
%%%
The closed-form estimator is then evaluated on a $60\times 60$ uniform grid over $[-1,1]^2$.   
Figures ~\ref{fig:closed-form-r2-n50_ell1}, \ref{fig:closed-form-r2-n100-ell1}, and \ref{fig:closed-form-r2-n1000-ell1} show the target function, the closed-form estimator, and the absolute reconstruction error. To emphasize the reconstructed surfaces, the training sample locations are not displayed.

\begin{figure}[ht]
    \centering
    \includegraphics[width=0.95\textwidth]{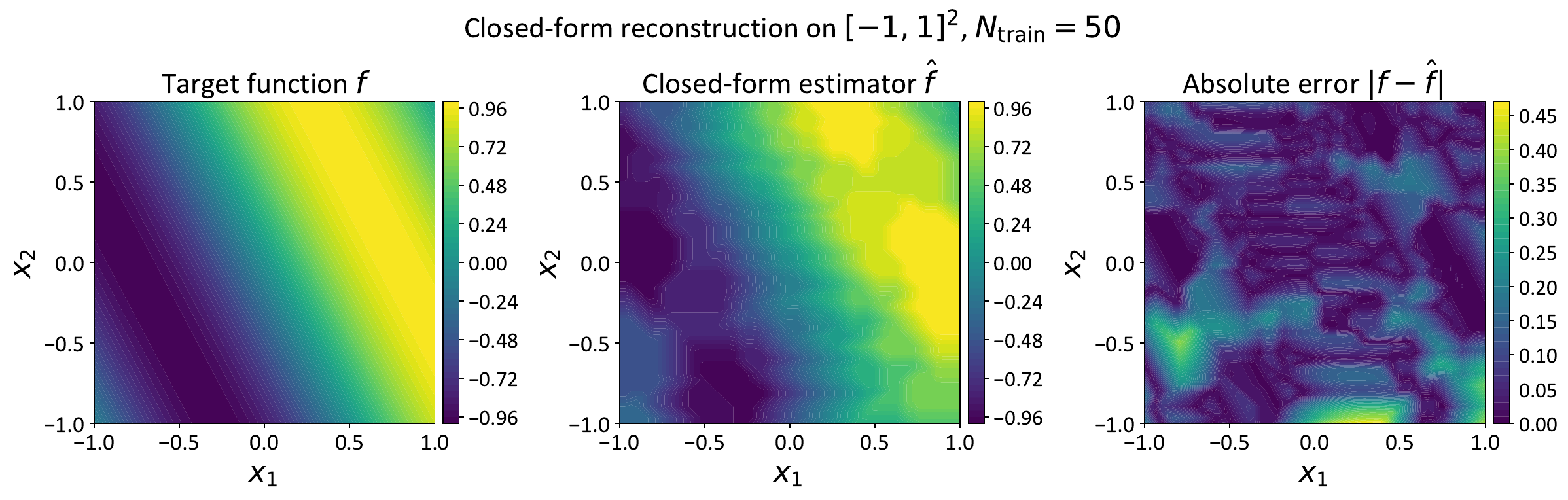}
    \caption{Closed-form reconstruction on $[-1,1]^2$ with $N_{\mathrm{train}}=50$ noiseless samples and $\ell_1$-distance as a metric.}
    \label{fig:closed-form-r2-n50_ell1}
\end{figure}

\begin{figure}[ht]
    \centering
    \includegraphics[width=0.95\textwidth]{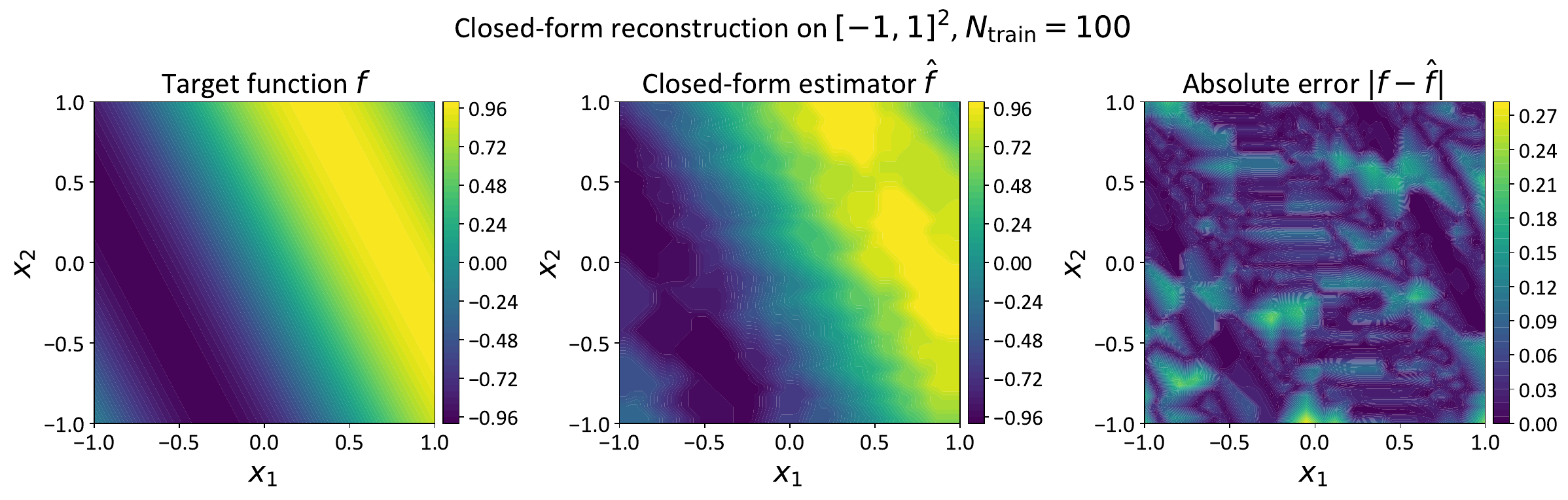}
    \caption{Closed-form reconstruction on $[-1,1]^2$ with $N_{\mathrm{train}}=100$ noiseless samples and $\ell_1$-distance as a metric.}
    \label{fig:closed-form-r2-n100-ell1}
\end{figure}

\begin{figure}[ht]
    \centering
    \includegraphics[width=0.95\textwidth]{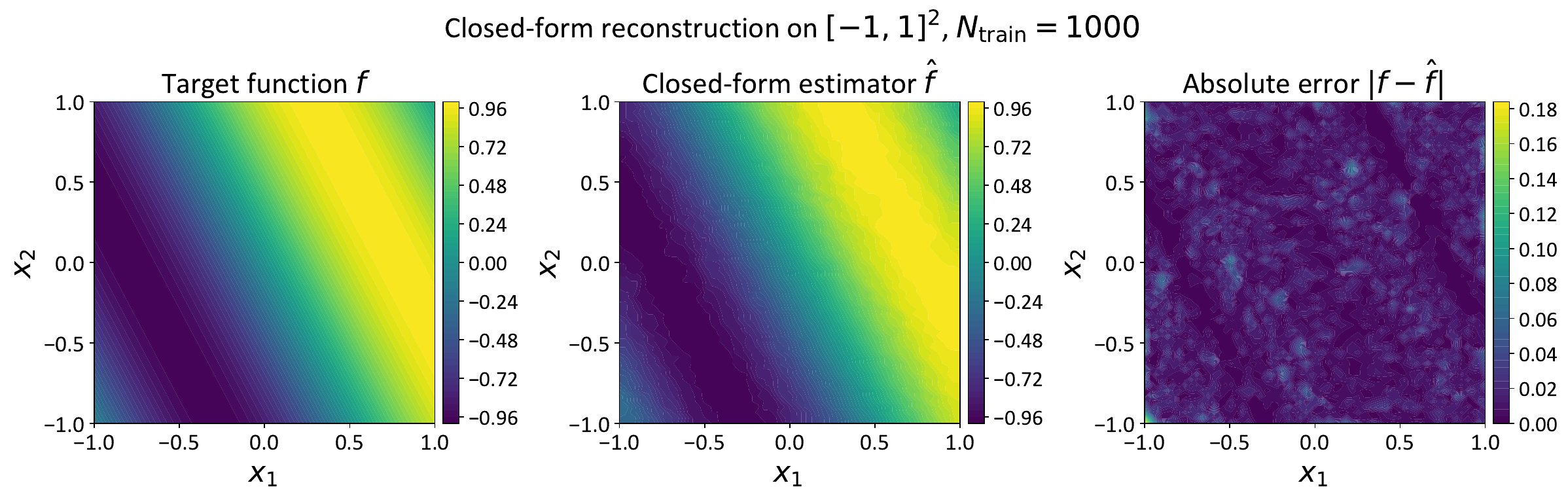}
    \caption{Closed-form reconstruction on $[-1,1]^2$ with $N_{\mathrm{train}}=1000$ noiseless samples and $\ell_1$-distance as a metric.}
    \label{fig:closed-form-r2-n1000-ell1}
\end{figure}

\subsection{Comparison with a Vanilla ReLU MLP}

We compare the closed-form Lipschitz estimator with two vanilla ReLU MLP baselines on a synthetic regression task over $\mathbb{R}^d$. In each run, we generate a random exactly $1$-Lipschitz target function given for $x\in[-1,1]^d$ by
\[
f(x)=\tanh(W_{\mathrm{norm}}^\top x+b)
\]
where $W_{\mathrm{norm}}$ is normalized according to the input metric. For the $\ell_2$ metric, we use
\[
W_{\mathrm{norm}}=\tfrac{W}{\|W\|_2},
\]
and choose $b=-W_{\mathrm{norm}}^\top x_0$ for a random interior point $x_0\in[-1,1]^d$, so that the maximum derivative of $\tanh$ is attained inside the domain. We then sample noiseless training data $Y_i=f(X_i)$ and compare the methods on an independently sampled test set.
For the closed-form method, we estimate the Lipschitz constant from the training data by $
L_{\mathcal{D}}
=
\max_{i\neq j}
\tfrac{|Y_i-Y_j|}{\|X_i-X_j\|_2}
$
and use the safety factor $
L_{\mathrm{used}}=1.05L_{\mathcal{D}}
$.
The closed-form estimator is then
\[
\hat{f}_{\mathrm{CF}}(x)
=
\frac12
\left[
\max_i \{Y_i-L_{\mathrm{used}}\|x-X_i\|_2\}
+
\min_i \{Y_i+L_{\mathrm{used}}\|x-X_i\|_2\}
\right].
\]
For the MLP baselines, we use ReLU networks with architecture with widths
\linebreak
$
(d ,W ,W, W, 1),
$
with widths $W=64$ and $W=128$. We report the test mean-squared error and the empirical Lipschitz estimate
$
\widehat{\mathrm{Lip}}(\hat{f})
=
\max_{i\neq j}
\frac{|\hat{f}(Z_i)-\hat{f}(Z_j)|}{\|Z_i-Z_j\|_2}
$,
computed on an independent set of Lipschitz-evaluation points ${Z_i}{i=1}^{N{\mathrm{lip}}}$; see Table~\ref{tab:mlp-closed-form-rd}.

\begin{table}[ht]
\centering
\caption{Comparison of the closed-form estimator and ReLU MLP baselines on random exactly $1$-Lipschitz target functions over $[-1,1]^d$. Hyperparameters: $d=2$, metric $\ell_2$, $N_{\mathrm{train}}=50$, $N_{\mathrm{test}}=2000$, $N_{\mathrm{lip}}=1000$, $20$ independent runs, MLP depth $3$, learning rate $10^{-3}$, and $1000$ training epochs. The MLP architecture is $d\to W\to W\to W\to 1$.}
\begin{tabular}{lcc}
\hline
\textbf{Method} & \textbf{Test MSE $\pm$ std} & \textbf{Empirical Lip. $\pm$ std} \\
\hline
Closed form & $2.605\times 10^{-3} \pm 1.854\times 10^{-3}$ & $1.034 \pm 0.011$ \\
MLP $W=64$ & $6.030\times 10^{-5} \pm 3.420\times 10^{-5}$ & $1.108\pm 0.050$ \\
MLP $W=128$ & $7.437\times 10^{-5} \pm 1.249\times 10^{-4}$ & $1.085 \pm 0.040$ \\
\hline
\end{tabular}
\label{tab:mlp-closed-form-rd}
\end{table}

Both methods achieve small test MSE on this smooth synthetic task, with the MLP baselines attaining lower error, while the closed-form estimator exhibits an empirical Lipschitz constant closer to the target value \(1\). This highlights the closed-form estimator’s advantage in controlled regularity and explicit construction, while still maintaining competitive predictive accuracy.

\subsection{Closed-Form Reconstruction on $\mathbb{R}$ in the Noisy Case}

We study the recovery of a clean one-dimensional target function from noisy observations. The domain is
$
    X=[-1,1]$ and $\rho(x,x')=|x-x'|$, 
and the ground truth is $
    f(x)=\sin(3x)
$.
For each noise level \(\sigma\), we sample
\[
    X_i\sim \mathrm{Unif}([-1,1]), 
    \qquad 
    Y_i=f(X_i)+\varepsilon_i,
    \qquad
    \varepsilon_i\sim \mathcal N(0,\sigma^2).
\]
We use \(N_{\mathrm{train}}=2000\), \(N_{\mathrm{test}}=2000\), and average over \(10\) random seeds, with
$
    \sigma\in\{0.0, 0.05, 0.1, 0.2, 0.4, 0.6,0.8,1,1.2,1.5\}.
$
We compare two estimators. The first is a standard ReLU MLP trained on the noisy labels by minimizing the empirical squared loss using Adam. The network has input dimension \(1\), output dimension \(1\), three hidden layers of width \(128\), ReLU activations, and is trained using Adam with learning rate \(10^{-3}\), batch size equal to the full training set, 
% no weight decay, 
and \(3000\) epochs.

The second estimator is the noisy universal formula. For each bandwidth \(\beta\), we locally average the noisy labels by
\[
    \bar Y_i^{(\beta)}
    =
    \frac{1}{|\{j:|X_j-X_i|\leq \beta\}|}
    \sum_{j:|X_j-X_i|\leq \beta} Y_j,
\]
and then apply
$
    \hat{f}_\beta(x)
    =
    \frac12
    \left[
    \min_i \{\bar Y_i^{(\beta)}+L_\beta |x-X_i|\}
    +
    \max_i \{\bar Y_i^{(\beta)}-L_\beta |x-X_i|\}
    \right]
$.
Here $L_\beta$ is the empirical Lipschitz constant of the averaged labels,
multiplied by $1.05$. The bandwidth $\beta$ is selected by cross-validation:
$20\%$ of the training data is held out for validation, and $\beta$ is chosen
from a logarithmic grid of 30 values in
$[\Delta_{\min},\operatorname{diam}(\mathcal X)]$, where $\Delta_{\min}$
is the minimum spacing between sorted training points and
$\operatorname{diam}(\mathcal X)=2$.
After selecting $\beta^\star$, we recompute the local averages and the final
estimator using the full training sample.
Both methods are evaluated against the clean target $f$, using clean test MSE.

\begin{figure}[ht]
    \centering
    \includegraphics[width=0.5\textwidth]{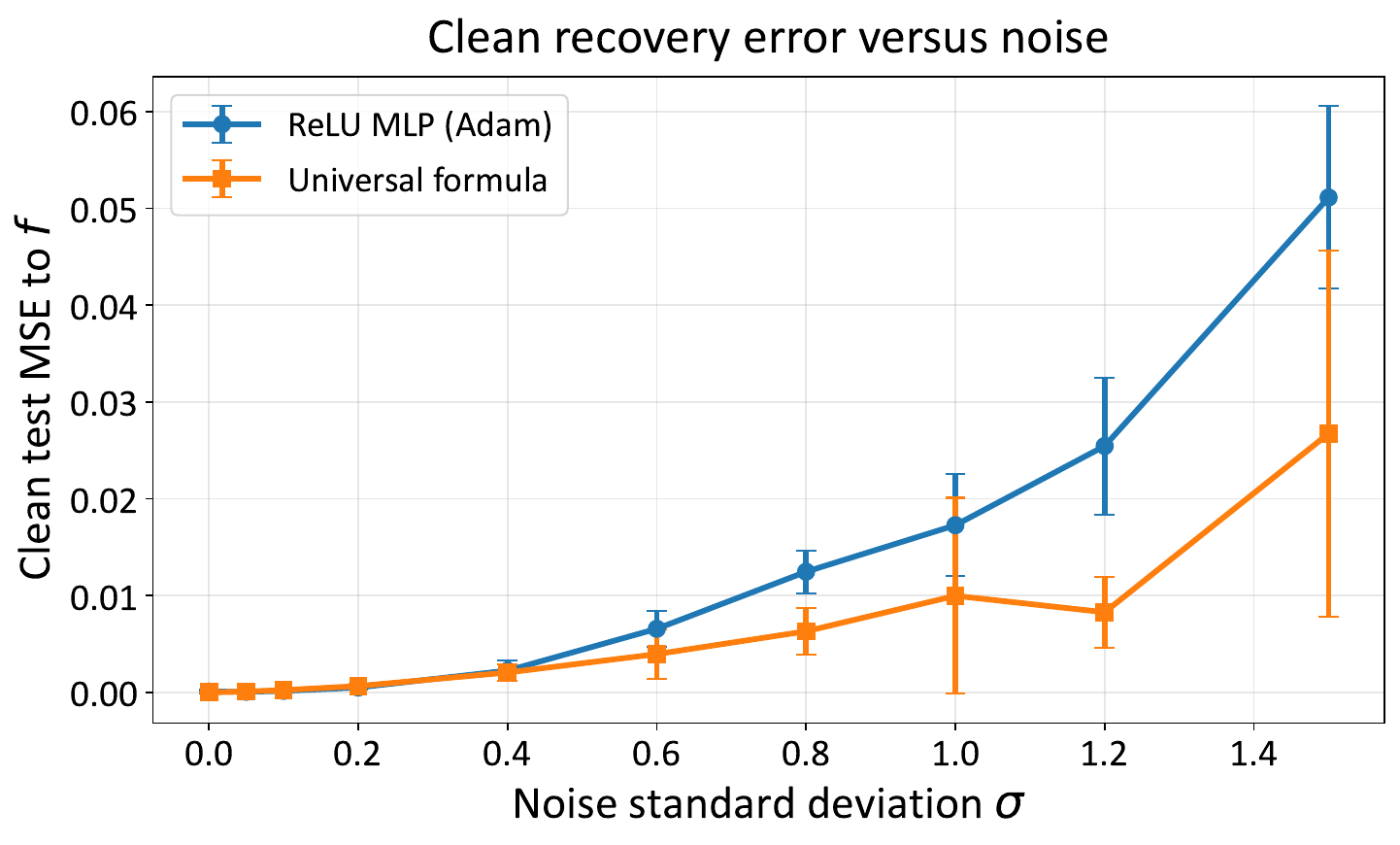}
    \caption{
    Clean test MSE versus noise standard deviation \(\sigma\). Both methods are trained on noisy labels but evaluated against the clean ground truth.
    }
    \label{fig:noisy-mse-sigma}
\end{figure}

Figure~\ref{fig:beta-vs-sigma} shows how the bandwidth selected by hold-out
validation changes with the noise level. For each value of $\sigma$, we report
the mean selected bandwidth $\beta^\star$ over 10 independent runs, with error
bars indicating one standard deviation. As the observation noise increases,
the selected bandwidth tends to increase, reflecting the need for stronger
local averaging to suppress noise before applying the Lipschitz extension.
Thus, $\beta^\star$ adapts to the noise level by controlling the bias-variance
tradeoff of the local smoothing step. Figure~\ref{fig:noisy-recovery-curves} shows ground truth, the ReLU MLP  estimate, the noisy universal-formula estimate, and noisy training samples for two different values of $\sigma$.

\begin{figure}[t]
    \centering
    \includegraphics[width=0.5\linewidth]
{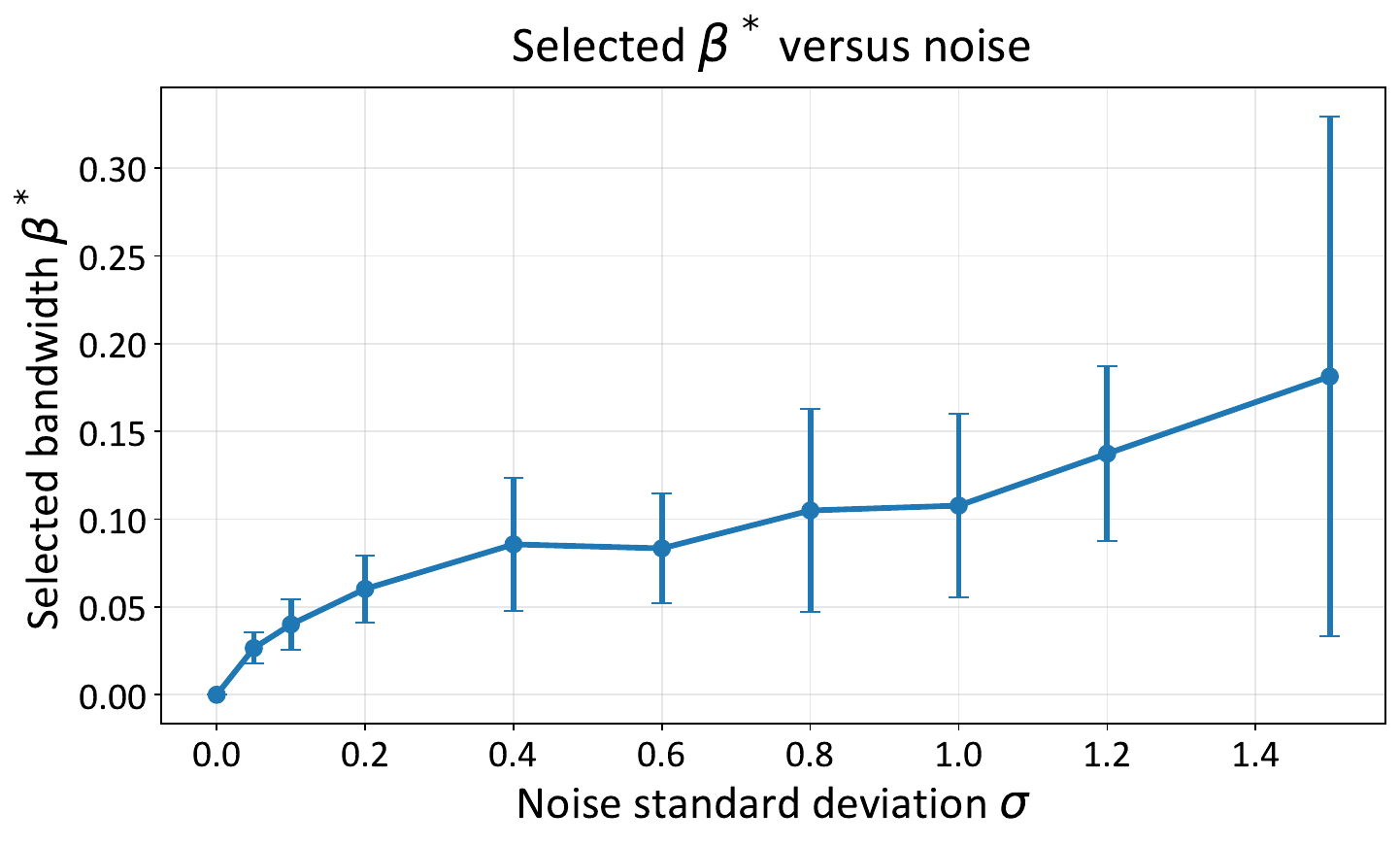}
    \caption{
    Selected bandwidth $\beta^\star$ as a function of the noise standard
    deviation $\sigma$. For each $\sigma$, $\beta^\star$ is selected by
    hold-out validation on the noisy training data. The curve shows the mean
    over 10 independent runs, and the error bars indicate one standard
    deviation.
    }
    \label{fig:beta-vs-sigma}
\end{figure}

\begin{figure}[ht]
    \centering

    \begin{subfigure}[t]{0.475\textwidth}
        \centering
        \includegraphics[width=\textwidth]{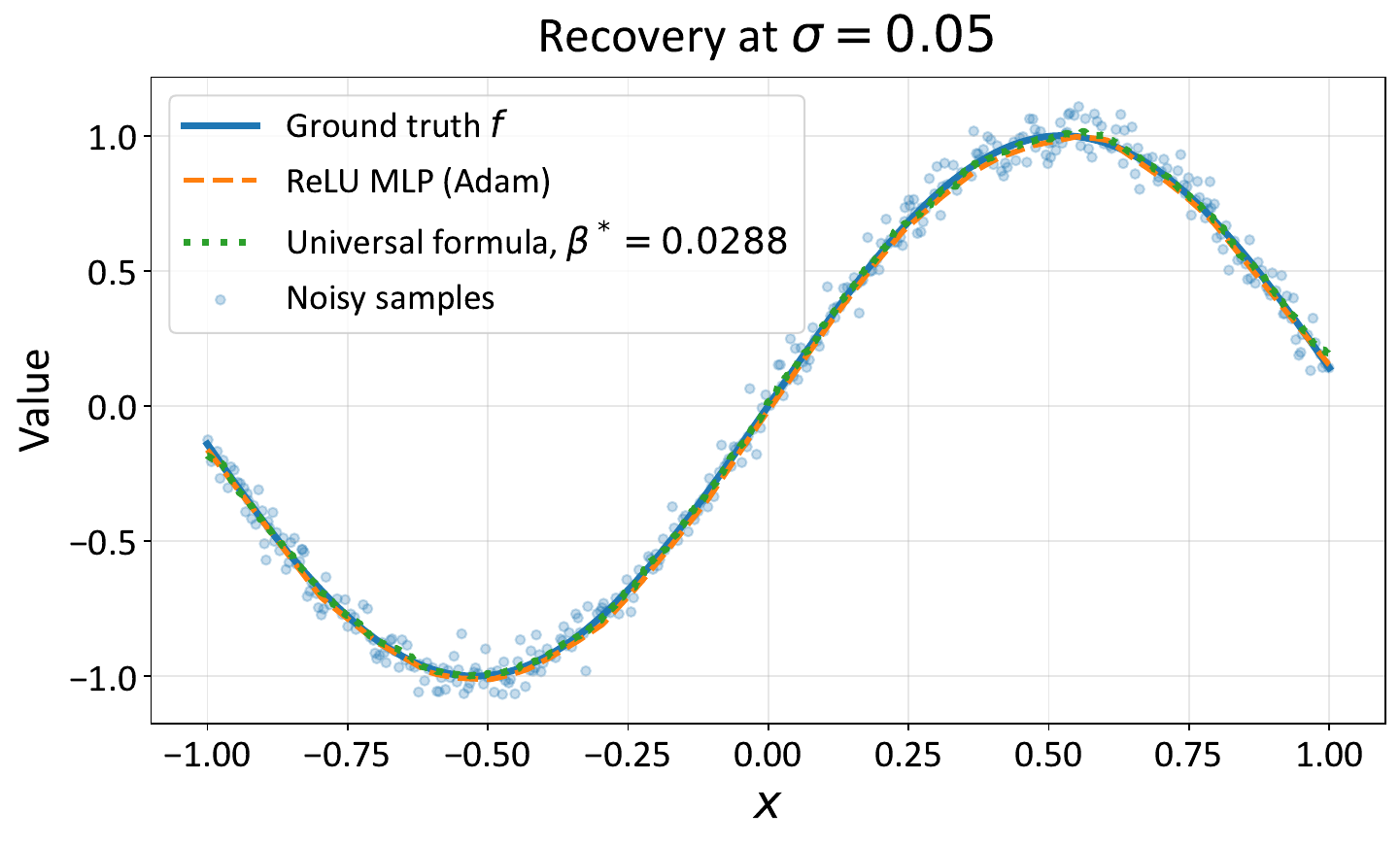}
        \caption{\(\sigma=0.05\)}
    \end{subfigure}
    \hfill
    \begin{subfigure}[t]{0.475\textwidth}
        \centering
        \includegraphics[width=\textwidth]{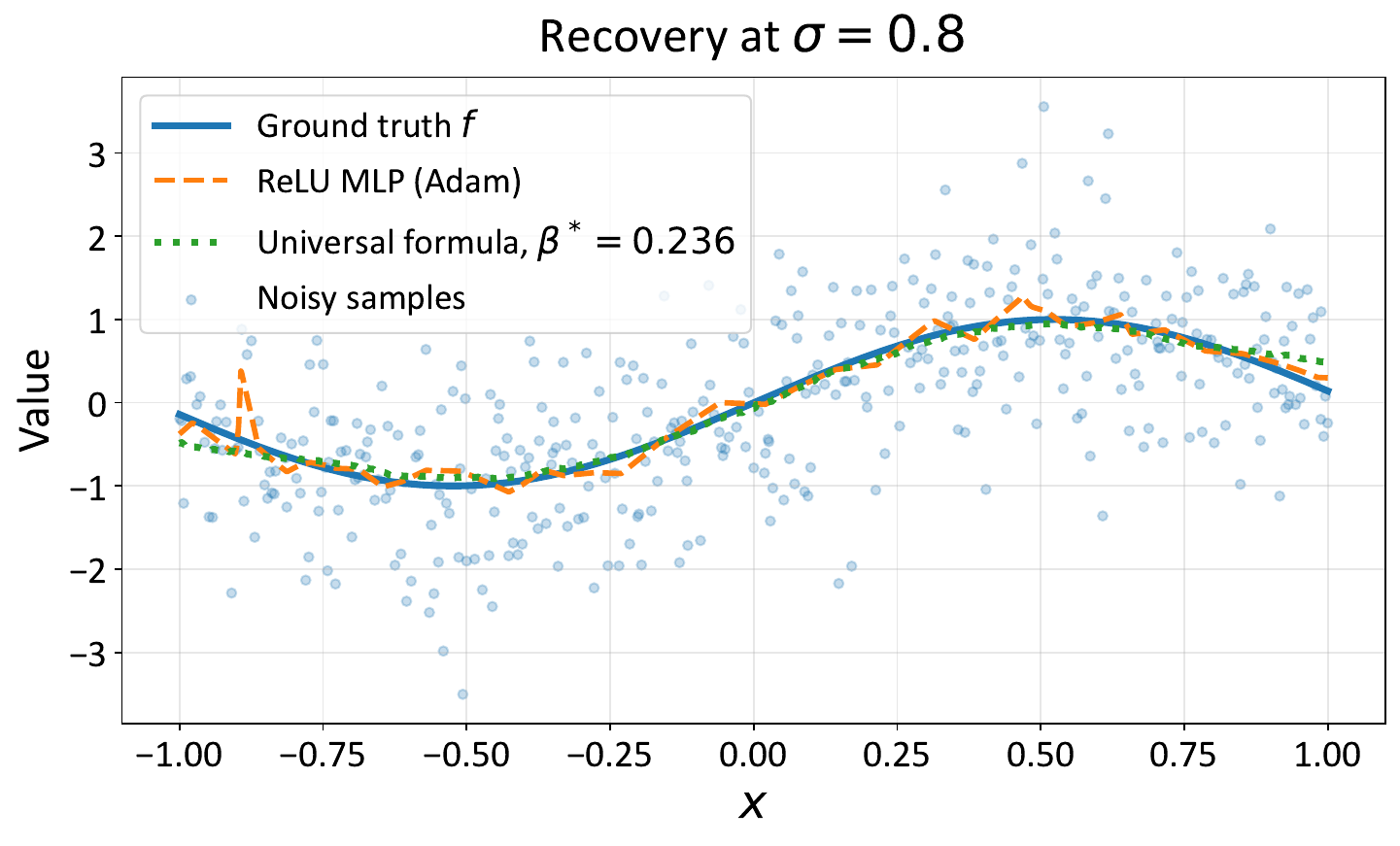}
        \caption{\(\sigma=0.8\)}
    \end{subfigure}

    \caption{
    Recovery of \(f(x)=\sin(3x)\) for representative noise levels. The plots show the ground truth, the ReLU MLP  estimate, the noisy universal-formula estimate, and noisy training samples.
    }
    \label{fig:noisy-recovery-curves}
\end{figure}
We next study how the recovery performance changes as the target function becomes more oscillatory. The domain and metric are again
$([-1,1],|\cdot|)$.
We fix the noise standard deviation at
    $\sigma=0.1$,
and consider the family of target functions
    $f_\omega(x)=\sin(\omega x)$,
where
    $\omega\in\{1,5,10,15,20,25,30,35,40,45,50\}$.
For each \(\omega\), we generate noisy observations
$
    X_i\sim \mathrm{Unif}([-1,1])
$, $Y_i=f_\omega(X_i)+\varepsilon_i$, $\varepsilon_i\sim \mathcal N(0,\sigma^2)$.  
We use \(N_{\mathrm{train}}=2000\), \(N_{\mathrm{test}}=2000\), and average the results over \(10\) random seeds. As \(\omega\) increases, the target function oscillates more rapidly, and the recovery problem becomes harder.

We use the same MLP architecture, optimization settings, bandwidth-selection procedure, and all-sample universal-formula implementation as in the preceding noise experiment.
Both methods are evaluated against the clean target \(f_\omega\), not against the noisy labels. The main error metric is the clean test MSE
$
    \tfrac{1}{N_{\mathrm{test}}}
    \sum_{j=1}^{N_{\mathrm{test}}}
    \big(
    \hat{f}(X_j^{\mathrm{test}})
    -
    f_\omega(X_j^{\mathrm{test}})
    \big)^2
$.
Figure~\ref{fig:noisy-frequency-mse} reports this error as a function of the frequency \(\omega\). Figure~\ref{fig:noisy-frequency-recovery} shows representative recovery plots for several frequencies.

\begin{figure}[H]%[ht]
    \centering
    \begin{subfigure}[t]{0.48\textwidth}
        \centering
        \includegraphics[width=\linewidth]{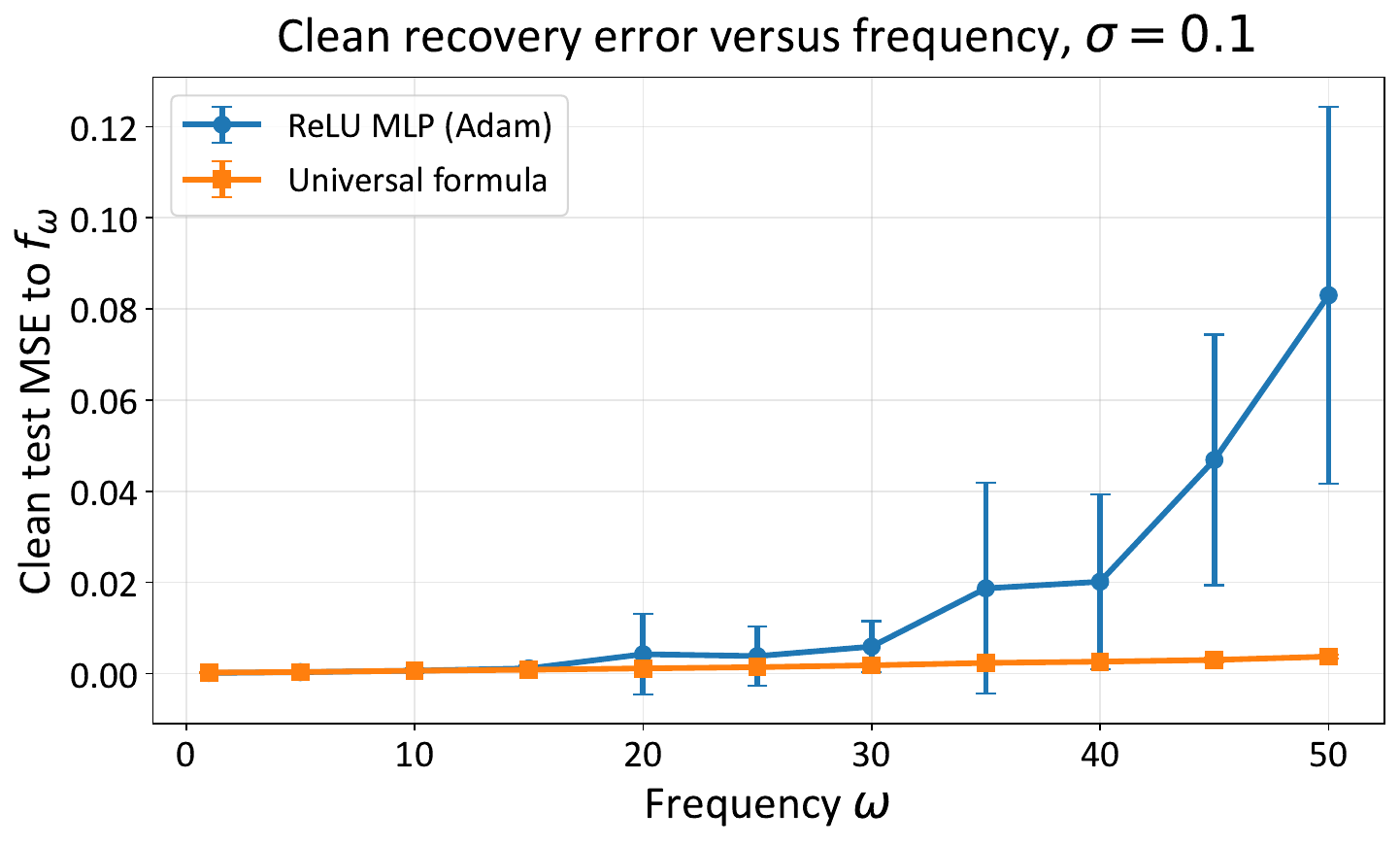}
        \caption{
        Clean test MSE versus target frequency $\omega$ under fixed noise level $\sigma=0.1$.
        Larger $\omega$ corresponds to a more oscillatory target function and hence a harder recovery problem.
        }
        \label{fig:noisy-frequency-mse}
    \end{subfigure}
    \hfill
    \begin{subfigure}[t]{0.48\textwidth}
        \centering
        \includegraphics[width=\linewidth]{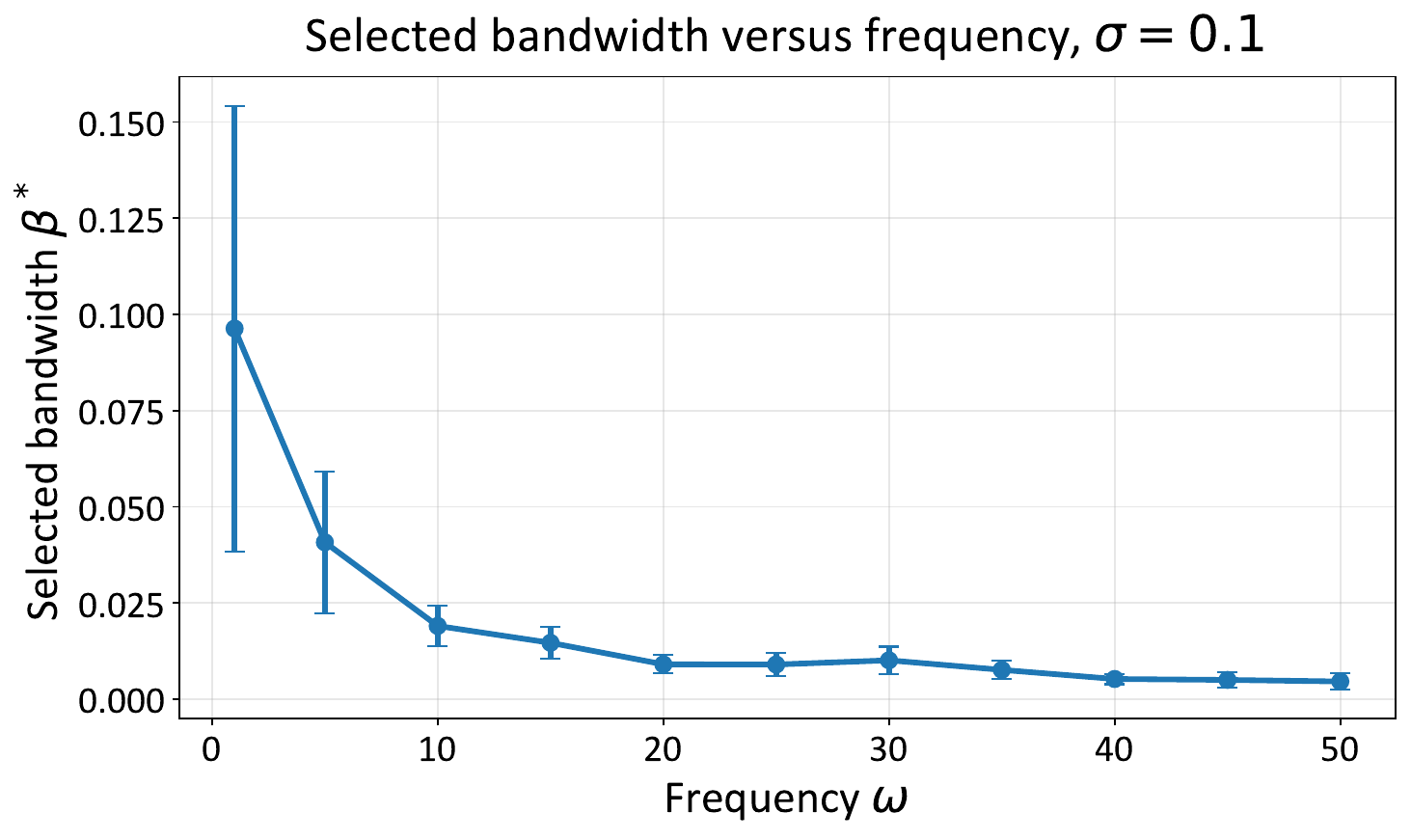}
        \caption{
        Selected bandwidth $\beta^\star$ as a function of the target frequency $\omega$ at fixed noise level $\sigma=0.1$.
        The curve shows the mean over 10 independent runs, with error bars indicating one standard deviation.
        }
        \label{fig:beta-vs-omega}
    \end{subfigure}
\end{figure}

Figure~\ref{fig:beta-vs-omega} shows how the bandwidth selected by hold-out validation changes with the frequency of the target function, under a fixed noise level $\sigma=0.1$. For each frequency $\omega$, we report the mean selected bandwidth $\beta^\star$ over 10 independent runs, with error bars indicating one standard deviation. As the frequency increases, the selected bandwidth generally decreases. This behaviour is natural: higher-frequency target functions vary more rapidly, so the local averaging step must use a smaller neighbourhood in order to avoid oversmoothing the signal. Thus, $\beta^\star$ adapts to the complexity of the target function by balancing noise reduction against preservation of local structure.

\begin{figure}[ht]
    \centering

    \begin{subfigure}[t]{0.475\textwidth}
        \centering
        \includegraphics[width=\textwidth]{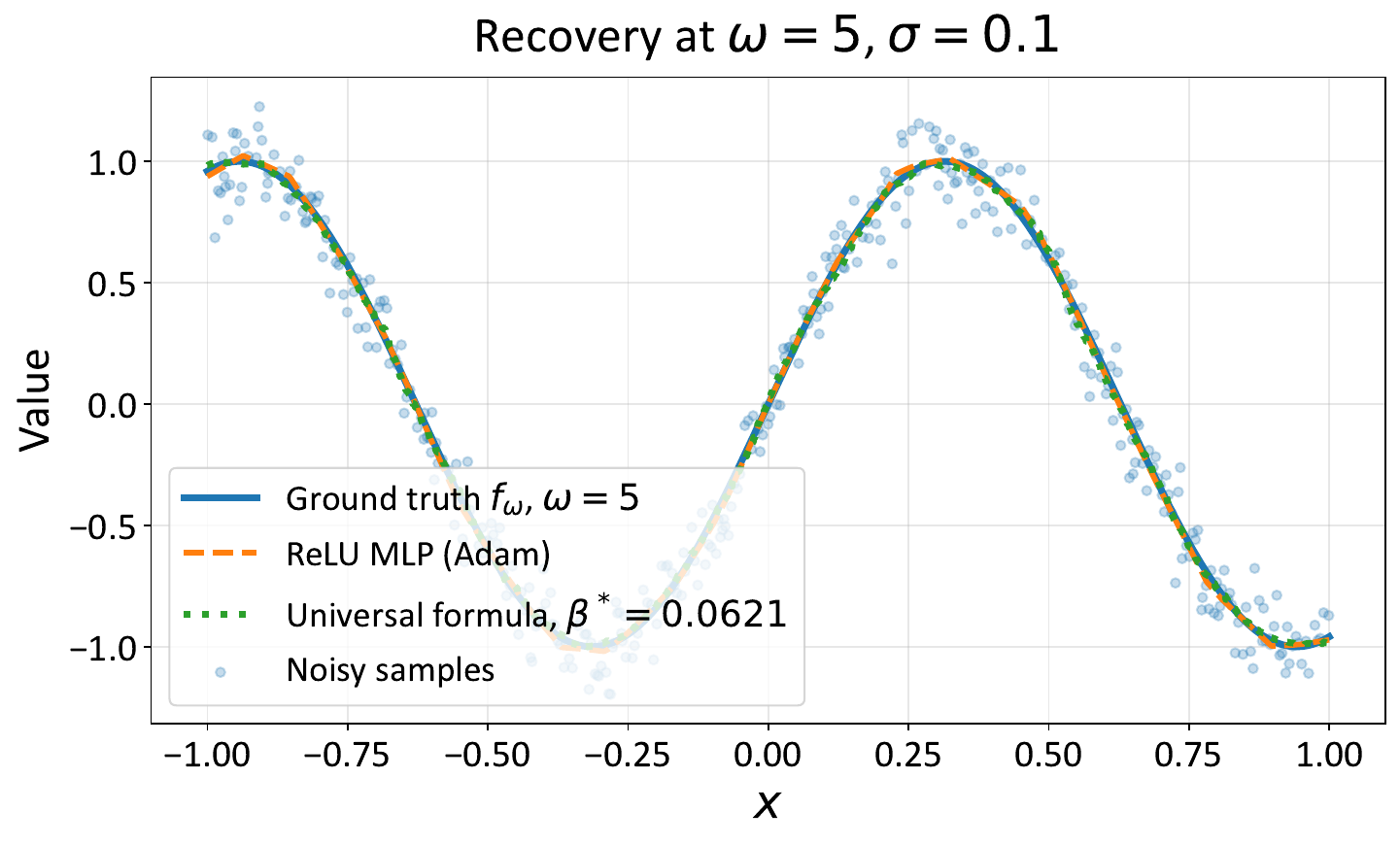}
        \caption{\(\omega=5\)}
    \end{subfigure}
    \hfill
    \begin{subfigure}[t]{0.475\textwidth}
        \centering
        \includegraphics[width=\textwidth]{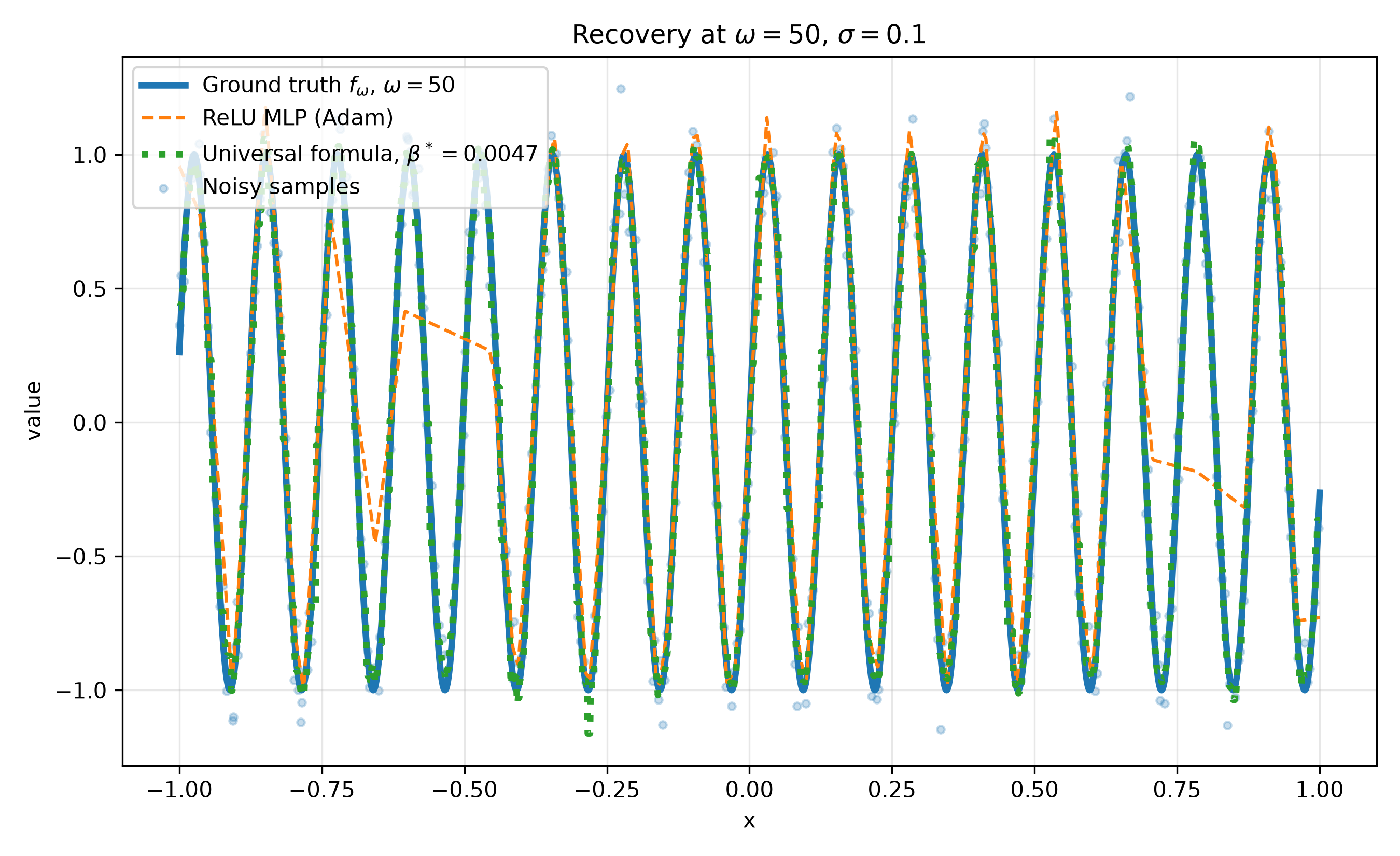}
        \caption{\(\omega=50\)}
    \end{subfigure}

    \caption{
    Recovery of \(f_\omega(x)=\sin(\omega x)\) for representative frequencies under fixed noise level \(\sigma=0.1\).
    The plots show the clean ground truth, the ReLU MLP estimate, the noisy universal-formula estimate, and the noisy training samples.
    }
    \label{fig:noisy-frequency-recovery}
\end{figure}

\section{Conclusion}
\label{s:Conclusion}

We introduce an explicit split-sample McShane--Whitney estimator for Lipschitz regression on general metric spaces and prove high-probability uniform recovery, attaining the minimax Euclidean polynomial exponent up to logarithmic factors without assuming access to an exact or approximate ERM (Theorem~\ref{thrm:nonparametric_noisy_case}).
On the Euclidean domains covered by Proposition~\ref{prop:representation} and Corollary~\ref{cor:realizability}, we give an exact sparse \texttt{ReLU}-MLP realization with explicitly assigned weights and a recovery bound containing no separate optimization-error term.
We further establish prescribed $L$-Lipschitz forward regularity (Proposition~\ref{prop:formula_lipschitz}), optimal-order fat-shattering complexity at the stated approximation scale (Theorem~\ref{thrm:fatshatteringotimality}), and stable near-minimax Lipschitz-width parameterization up to logarithmic factors in the oracle regime (Theorem~\ref{thm:stable_near_minimax_optimality_oracle_regime}).
Section~\ref{s:NumericalIllustrations} evaluates an all-sample implementation, whereas the finite-sample guarantee concerns the split-sample, fixed-cardinality estimator.

\appendix
\section{Additional Technical Background}
\label{s:TechDefs_Extras}
This appendix collects several technical definitions and notations that are not essential to the main exposition but will be useful in the proofs that follow.

\begin{definition}[Entropy number]
    Let $(X,\rho)$ be a metric space, and $\mathcal{M}\subseteq X$ be compact. For every fixed $n \in \mathbb{N}$, the entropy number $\varepsilon_n(\mathcal{M})_X$ is the infimum of all $\varepsilon > 0$ for which $2^n$ balls with centers from $X$ and radius $\varepsilon$ cover $\mathcal{M}$. Formally, we write
    \[
        \varepsilon_n(\mathcal{M})_X
        =
        \inf\left\{
            \varepsilon > 0 :
            \exists
            g_1,\cdots,g_{2^n} \in X,
            \
            \mathcal{M}
            \subseteq
            \bigcup_{j=1}^{2^n} B(g_j,\varepsilon)
        \right\}.
    \]
\end{definition}

\begin{definition}[Metric Entropy]
    Let $(X,\rho)$ be a metric space and let $\varepsilon>0$.  A set
    $\mathcal X_\varepsilon\subseteq X$ is called an $\varepsilon$-net of $X$ if
    $
    X
    \subseteq
    \bigcup_{x\in\mathcal X_\varepsilon}
        B(x,\varepsilon).
    $
    The $\varepsilon$-covering number of $X$ is
    \[
        \mathcal N(X,\rho,\varepsilon)
        \eqdef
        \inf
        \left\{
            |\mathcal X_\varepsilon|
            :
            \mathcal X_\varepsilon\subseteq X
            \text{ is an } \varepsilon\text{-net of }X
        \right\}.
    \]
    The metric entropy of $(X,\rho)$ is defined as 
    \[
        \log_2
        \mathcal N(X,\rho,\varepsilon)
        .
    \]
\end{definition}

\section{Proofs}
\label{s:Proofs}

\paragraph{Symmetry convention for the envelope arguments}
Before beginning the proofs, it is worth recording a notational convention.  For $\alpha\in[0,1]$, consider the convex combination of the upper and lower McShane--Whitney envelopes
\begin{equation}
\label{eq:formula__realizable__simplified}
\begin{aligned}
    \hat{f}_{\theta,\alpha}(x\mid G^{(\beta)})
    &\eqdef
        \alpha
        \min_{m\in[M]_+}
        \big\{
            V_m+\chi_m(x)
        \big\}
        +
        (1-\alpha)
        \max_{m\in[M]_+}
        \big\{
            V_m-\chi_m(x)
        \big\},
    \\
    V
    &\eqdef
        aG^{(\beta)}
        =
        (V_m)_{m=1}^M,
    \\
    \chi(x)
    &\eqdef
        \big(
            c_m\rho(x,b_m)
        \big)_{m=1}^M.
\end{aligned}
\end{equation}
The formula studied in the main text corresponds to
$\alpha=\tfrac12$, namely the central McShane--Whitney midpoint.  For notational simplicity, the elementary envelope arguments below are written for $\alpha=0$, i.e.\ for the lower envelope
$
    W_V(x)
\eqdef
    \max_{m\in[M]_+}
    \big\{
        V_m-\chi_m(x)
    \big\}
$.
The corresponding arguments for the upper envelope
$
    U_V(x)
    \eqdef
    \min_{m\in[M]_+}
    \big\{
        V_m+\chi_m(x)
    \big\}
$
are identical after simultaneously replacing $V$ and the target function
$f$ by $-V$ and $-f$, since $
    U_V(x)
    =
    -
    W_{-V}(x)
$.
Consequently, whenever a conclusion is preserved under convex combinations-such as well-definedness, measurability, or Lipschitz regularity---it follows for every $\alpha\in[0,1]$, and in particular for $\alpha=\tfrac12$, without repeating the same min/max calculation twice.

The use of both envelopes matters quantitatively only for the sharper central-midpoint approximation estimate.  Indeed, in the notation of Lemma~\ref{lem:central_midpoint_stability}, the complementary one-sided bounds for $U_V$ and $W_V$ imply that
\[
    f(x)
    -
    \Delta
    -
    2(1-\alpha)L\eta
    \le
    \hat{f}_{\theta,\alpha}(x\mid G^{(\beta)})
    \le
    f(x)
    +
    \Delta
    +
    2\alpha L\eta.
\]
Therefore,
$
    \big|
        \hat{f}_{\theta,\alpha}(x\mid G^{(\beta)})
        -
        f(x)
    \big|
    \le
    \Delta
    +
    2\max\{\alpha,1-\alpha\}L\eta.
$
This worst-case bound is minimized at $\alpha=\tfrac12$, for which it
becomes
$
    \big|
        \hat{f}_{\theta,1/2}(x\mid G^{(\beta)})
        -
        f(x)
    \big|
    \le
    \Delta+L\eta
$; as stated in
Lemma~\ref{lem:central_midpoint_stability}.  Thus, whenever this sharper central-midpoint estimate is required, we invoke Lemma~\ref{lem:central_midpoint_stability} (below); otherwise, we display only one of the two symmetric envelope calculations for notational simplicity.

With that said, our objective here is to analyze the following object, trained on noisy data corrupted with measurement noise.  We consider the i.i.d.\ LSI regime.
\begin{definition}[$N$-Point Whitney-McShane Operator]
\label{defn:Whitney_McShane_operator}
Let $(X,\rho)$ be a bounded metric space and fix $N\in \mathbb{N}_+$.  The \textit{$N$-point Whitney-McShane operator} is the map
\[
\begin{aligned}
    \mathcal{E}:
    (\operatorname{Lip}((X,\rho);1),\|\cdot\|_\infty)
    \times
    ([X\times \mathbb{R}]^N,\rho_1)
&\longrightarrow
    (\operatorname{Lip}((X,\rho);1),\|\cdot\|_\infty),
\\
    (f,\mathcal{D}_N)
&\longmapsto
    \mathcal{E}(f,\mathcal{D}_N),
\end{aligned}
\]
where, for every $\mathcal{D}_N\eqdef (x_n,y_n)_{n=1}^N\in [X\times \mathbb{R}]^N$, for each $z\in X$ by
\[
        \mathcal{E}(f,\mathcal{D}_N)(z)
    \eqdef
        \max_{n\in[N]_+}
        \,
            \big\{
                [f(x_n)+y_n]-\rho(z,x_n)
            \big\}
.
\]
\end{definition}

\subsection{Step 1: estimator in the Noiseless Case}
\label{s:Step1Noiseless}

\begin{proposition}[Regular Uniform Approximation: Noiseless Case - $\tilde{O}(\varepsilon^{-2(s+1)})$ Samples Suffice]
\label{prop:random_apprx_noisy}
Suppose that Assumption~\ref{assumption:main_1Lip} holds.  For every $\varepsilon>0$, 
and each $f\in \operatorname{Lip}((X,\rho);1)$ and every
$\delta_{\operatorname{net}},\delta_{\operatorname{noise}}\in(0,1)$, if
\begin{equation}
\label{eq:required_no_samples_noisy_Ahlfors}
    N
    \ge
    \frac{1}{a}
    \left(\frac{4}{\varepsilon}\right)^s
    \log\left(\frac{N_{\varepsilon/4}(X,\rho)}{\delta_{\operatorname{net}}}\right)
,
\end{equation}
then the following holds with probability at-least $-\delta_{\mathrm{net}}-\delta_{\mathrm{noise}}$
% have
\[
\begin{aligned}
% \mathbb{P}\Bigg(
% &
\sup_{x\in\mathcal X}
 \bigl|f(x)-\mathcal E(f\mid\mathbb D_N)(x)\bigr|
 < \varepsilon
 + \sqrt{C\log\!\left(\frac{2N}{\delta_{\mathrm{noise}}}\right)}
% \\[-1mm]
% &
% \quad\text{and}\quad
\,\mbox{ and }\,
 \{X_n\}_{n=1}^{N}
 \text{ is an $\varepsilon$-net of }(\mathcal X,\rho)
% \Bigg)
% \ge 1-\delta_{\mathrm{net}}-\delta_{\mathrm{noise}} .
\end{aligned}
\]
\end{proposition}
\noindent
We now directly deduce Theorem~\ref{thrm:random_apprx_noisy}.
\begin{proof}[{Proof of Theorem~\ref{thrm:random_apprx_noisy}}]
Since $Y_n=f(X_n)$ for each $n\in [N]_+$, then $C=0$.  Setting $\delta_{\operatorname{net}}=\delta_{\operatorname{noise}}\eqdef \tfrac{\delta}{2}$ in Proposition~\ref{prop:random_apprx_noisy} yields the conclusion.
\end{proof}

\begin{lem}[Probability of Seeing $\varepsilon$-Most of A Compact Metric Space]
\label{lem:coupon_collector_random_net}
Let $(X,\rho)$ be a compact metric space and let $\mu\in\mathcal{P}(X)$.  Fix
$\varepsilon>0$, and let $x_1^\star,\ldots,x_M^\star\in X$ be such that
\begin{equation}
\label{eq:covering_condition}
    X
    =
    \bigcup_{m=1}^{N_{\varepsilon/2}(X,\rho)}
    B_\rho(x_m^\star,\varepsilon/2)
.
\end{equation}
Assume that
$
    p_\varepsilon
    \eqdef
    \min_{m\in[M]_+}
    \mu\big(B_\rho(x_m^\star,\varepsilon/2)\big)
    >
    0.
$
Let $X_1,\ldots,X_N$ be i.i.d.\ with law $\mu$; then, for every $\delta\in(0,1)$, if
\begin{equation}
\label{eq:required_no_samples}
    N
    \ge
    \frac{1}{p_\varepsilon}
    \log\left(\frac{M}{\delta}\right)
,
\end{equation}
then
$
    \mathbb{P}
    \left(
        \{X_n\}_{n=1}^N
        \mbox{ is an } \varepsilon\mbox{-net of }(X,\rho)
    \right)
    \ge
    1-\delta
$.
\end{lem}
\begin{proof}[{Proof of Lemma~\ref{lem:coupon_collector_random_net}}]
For each $m\in[M]_+$, set
$
    B_m\eqdef B_\rho(x_m^\star,\varepsilon/2).
$
Let $\mathcal{G}$ denote the event that every element of the cover $(B_m)_{m=1}^M$ is hit by at least one of the samples $X_1,\ldots,X_N$; that is,
\[
    \mathcal{G}
    \eqdef
    \bigcap_{m=1}^M
    \left\{
        \exists n\in[N]_+
        \mbox{ such that }
        X_n\in B_m
    \right\}.
\]
We first show that $\mathcal{G}$ implies that $\{X_n\}_{n=1}^N$ is an $\varepsilon$-net of $(X,\rho)$.  Fix $x\in X$.  Since $(B_m)_{m=1}^M$ covers $X$, there exists $m\in[M]_+$ such that $x\in B_m$.  On $\mathcal{G}$, there exists $n\in[N]_+$ such that $X_n\in B_m$.  Hence, by the triangle inequality,
\[
    \rho(x,X_n)
    \le
    \rho(x,x_m^\star)
    +
    \rho(x_m^\star,X_n)
    <
    \frac{\varepsilon}{2}
    +
    \frac{\varepsilon}{2}
    =
    \varepsilon.
\]
Thus, on $\mathcal{G}$, every $x\in X$ lies within distance $\varepsilon$ of some sample point $X_n$.
It remains to estimate $\mathbb{P}(\mathcal{G}^c)$.  By the union bound,
\[
\begin{aligned}
    \mathbb{P}(\mathcal{G}^c)
    &=
    \mathbb{P}
    \left(
        \bigcup_{m=1}^M
        \left\{
            X_n\notin B_m
            \mbox{ for every }n\in[N]_+
        \right\}
    \right)
    \\
    &\le
    \sum_{m=1}^M
    \mathbb{P}
    \left(
        X_n\notin B_m
        \mbox{ for every }n\in[N]_+
    \right)
    \\
    &=
    \sum_{m=1}^M
    \big(1-\mu(B_m)\big)^N
    \\
    &\le
    \sum_{m=1}^M
    e^{-N\mu(B_m)}
    \\
    &\le
    M e^{-Np_\varepsilon}.
\end{aligned}
\]
If \eqref{eq:required_no_samples} holds, then $M e^{-Np_\varepsilon}\le \delta$.  Therefore $\mathbb{P}(\mathcal{G})\ge 1-\delta$.  Since $\mathcal{G}$ implies that $\{X_n\}_{n=1}^N$ is an $\varepsilon$-net of $(X,\rho)$, the claim follows.
\end{proof}

\begin{lem}[Uniform approximation by the estimator]
\label{lem:Whitney_McShane_uniform_approximation}
Let $(X,\rho)$ be a metric space and let $f\in \operatorname{Lip}((X,\rho);1)$.  Let
$x_1,\ldots,x_N\in X$ be an $\varepsilon/2$-net of $(X,\rho)$, and set
$
    \mathcal{D}^{\star}\eqdef ((x_n,0))_{n=1}^N.
$
Then the Whitney-McShane \textit{approximator}
$
    \mathcal{E}(f|\mathcal{D}^{\star})
$
satisfies
\begin{equation}
\label{eq:unif_approx}
    \sup_{x\in X}
    \,
    \big|
        f(x)
        -
        \mathcal{E}(f|\mathcal{D}^{\star})(x)
    \big|
    \le
    \varepsilon
.
\end{equation}
\end{lem}

\begin{proof}[{Proof of Lemma~\ref{lem:Whitney_McShane_uniform_approximation}}]
Fix $x\in X$.  First, for every $n\in[N]_+$, since $f$ is $1$-Lipschitz,
$
    f(x_n)-f(x)\le \rho(x_n,x).
$
Hence
$
    f(x_n)-\rho(x,x_n)\le f(x).
$
Taking the maximum over $n\in[N]_+$ gives
\[
    \mathcal{E}(f|\mathcal{D}^{\star})(x)
    =
    \max_{n\in[N]_+}
    \big\{
        f(x_n)-\rho(x,x_n)
    \big\}
    \le
    f(x).
\]
Thus
$
    |f(x)-\mathcal{E}(f|\mathcal{D}^{\star})(x)|
    =
    f(x)-\mathcal{E}(f|\mathcal{D}^{\star})(x).
$

Since $\{x_1,\ldots,x_N\}$ is an $\varepsilon/2$-net, there exists $n_x\in[N]_+$ such that
$
    \rho(x,x_{n_x})\le\varepsilon/2.
$
Using again that $f$ is $1$-Lipschitz,
$
    f(x_{n_x})\ge f(x)-\rho(x,x_{n_x}).
$
Therefore
\[
\begin{aligned}
    \mathcal{E}(f|\mathcal{D}^{\star})(x)
    \ge
    f(x_{n_x})-\rho(x,x_{n_x})
    \ge
    f(x)-2\rho(x,x_{n_x})
    \ge
    f(x)-\varepsilon
.
\end{aligned}
\]
Combining the two estimates yields
$
    0\le f(x)-\mathcal{E}(f|\mathcal{D}^{\star})(x)\le\varepsilon.
$
Since $x\in X$ was arbitrary, \eqref{eq:unif_approx} follows.
\end{proof}
We are now equipped to prove the randomized version of the previous result.

\begin{lem}[Randomized Uniform Approximation: Noiseless Case]
\label{lem:random_apprx_noisless}
Let $(X,\rho)$ be a compact metric space, let $\mu\in\mathcal{P}(X)$, and let
$X_1,\ldots,X_N$ be i.i.d.\ with law $\mu$.  Fix $\varepsilon>0$, set
$M_\varepsilon\eqdef N_{\varepsilon/4}(X,\rho)$, and let
$x_1^\star,\ldots,x_{M_\varepsilon}^\star\in X$ be such that
$X=\bigcup_{m=1}^{M_\varepsilon}B_\rho(x_m^\star,\varepsilon/4)$.
Assume that
$
    p_\varepsilon
    \eqdef
    \min_{m\in[M_\varepsilon]_+}
    \mu\big(B_\rho(x_m^\star,\varepsilon/4)\big)
    >
    0.
$
Then, for every $f\in \operatorname{Lip}((X,\rho);1)$ and every $\delta\in(0,1)$, if
\begin{equation}
\label{eq:required_no_samples_noiseless}
    N
    \ge
    \frac{1}{p_\varepsilon}
    \log\left(\frac{M_\varepsilon}{\delta}\right)
,
\end{equation}
then the noiseless estimator
$\mathcal{E}(f|((X_n,0))_{n=1}^N)$ satisfies
\begin{equation}
\label{eq:PAC_Guarantee}
    \mathbb{P}\left(
        \sup_{x\in X}\,
        \big|
            f(x)
            -
            \mathcal{E}(f|((X_n,0))_{n=1}^N)(x)
        \big|
        <
        \varepsilon
    \right)
    \ge
    1-\delta
.
\end{equation}
\end{lem}

\begin{proof}[{Proof of Lemma~\ref{lem:random_apprx_noisless}}]
We first check measurability.  For $\mathbf{x}=(x_n)_{n=1}^N\in X^N$, define
$\mathcal{D}_{\mathbf{x}}^\star\eqdef ((x_n,0))_{n=1}^N$ and
$
    \Psi(\mathbf{x})
\eqdef
    \sup_{z\in X}
    \,
        \big|
            f(z)
            -
            \mathcal{E}(f|\mathcal{D}_{\mathbf{x}}^\star)(z)
        \big|
$.
Since $X$ is compact, the 
\linebreak
supremum is finite.  Moreover, for any
$\mathbf{x},\widetilde{\mathbf{x}}\in X^N$, the reverse triangle inequality and the stability of the Whitney-McShane operator give
\[
\begin{aligned}
    |\Psi(\mathbf{x})-\Psi(\widetilde{\mathbf{x}})|
    % &
    \le
    \big\|
        \mathcal{E}(f|\mathcal{D}_{\mathbf{x}}^\star)
        -
        \mathcal{E}(f|\mathcal{D}_{\widetilde{\mathbf{x}}}^\star)
    \big\|_\infty
    % \\
    % &
    \le
    2\,\rho^{(1)}(\mathbf{x},\widetilde{\mathbf{x}})
\end{aligned}
\]
where the last inequality held by Lemma~\ref{s:Optimality__ss:ProofStep2}.
Thus $\Psi:X^N\to\mathbb{R}$ is continuous, hence Borel measurable.  Since
$(X_1,\ldots,X_N)$ is an $X^N$-valued random variable, then we have
\[
    \sup_{x\in X}
    \,
    \big|
        f(x)
        -
        \mathcal{E}(f|((X_n,0))_{n=1}^N)(x)
    \big|
    =
    \Psi((X_n)_{n=1}^N)
\]
is a real-valued random variable.
We now prove the probability estimate.  Apply Lemma \ref{lem:coupon_collector_random_net} with $\varepsilon/2$ in place of $\varepsilon$.  Since
$X=\bigcup_{m=1}^{M_\varepsilon}B_\rho(x_m^\star,\varepsilon/4)$, $M_\varepsilon=N_{\varepsilon/4}(X,\rho)$, and
$
    p_\varepsilon
    =
    \min_{m\in[M_\varepsilon]_+}
    \mu\big(B_\rho(x_m^\star,\varepsilon/4)\big),
$
condition~\eqref{eq:required_no_samples_noiseless} implies that
$
    \mathbb{P}\big(
        \{X_n\}_{n=1}^N
        \mbox{ is an } \varepsilon/2\mbox{-net of }
        \linebreak
        (X,\rho)
    \big)
    \ge
    1-\delta.
$
On this event, Lemma~\ref{lem:Whitney_McShane_uniform_approximation} gives
$
    \sup_{x\in X}\,
    \big|
        f(x)
        -
        \mathcal{E}(f|((X_n,0))_{n=1}^N)(x)
    \big|
    \linebreak
    \le
    \varepsilon.
$
Therefore \eqref{eq:PAC_Guarantee} follows.
\end{proof}

\subsection{Step 2: Analysis Of Whitney-McShane with Noisy Measurements}
\label{s:Step2Noisy}
We now incorporate the role of noise into our estimate in order to obtain our conclusion; when we have access to noisy training data (corrupted with measurement noise also).

Unfortunately, in the form in Proposition~\ref{prop:random_apprx_noisy}, the estimator is still sensitive to unlikely but large errors in the measurement data; i.e.\ the random fluctuations of 
\linebreak
$\max_{n\in [N]}\, |\varepsilon_n|$ dominate at a scale of $\mathcal{O}(\sqrt{\log(N)})$ with high probability.  While there is no issue if there is no measurement noise (e.g.\ in classical approximation theorems) there is no reason why this should be the case in practice. 
This problem with the Whitney-McShane formula can be overcome by first locally-averaging the training data; thereby nearly points have their output values averaged so as to cancel-out measurement noise. 

\begin{algorithm}[ht]
\caption{Split-Sample Local Averaging}
\label{alg:LocalAveragingRobustification}
\DontPrintSemicolon

\KwIn{
An observed dataset
$
    \mathbb{D}_N=((X_n,Y_n))_{n=1}^{N}
$
with $N=N_0+N_1$, a bandwidth $\beta>0$, and an integer
$q\in[N_1]_+$.
}

\KwOut{
An effective dataset
$
    \mathbb{D}_{N_0}^{(\beta,q)}
    =
    ((X_k,V_k^{(\beta,q)}))_{k=1}^{N_0}.
$
}

Initialize
$
    \mathbb{D}_{N_0}^{(\beta,q)}
    \gets
    \varnothing
$.\;

\For{$k=1,\ldots,N_0$}{
    Set $
        \mathcal I_k^{(\beta)}
        \gets
        \{
            i\in[N_1]_+
            :
            \rho(X_{N_0+i},X_k)\le\beta
        \}$.

    \eIf{$\#\mathcal I_k^{(\beta)}\ge q$}{
        Let
        $
            i_{k,1}<\cdots<i_{k,q}
        $
        be the first $q$ elements of
        $\mathcal I_k^{(\beta)}$ in increasing order
        \;
        Set $
            V_k^{(\beta,q)}
            \gets
            \frac{1}{q}
            \sum_{\ell=1}^{q}
                Y_{N_0+i_{k,\ell}}.
        $
    }{
        Set
        $
            V_k^{(\beta,q)}
            \gets
            0
        $.
    }

    Update
    $
        \mathbb{D}_{N_0}^{(\beta,q)}
        \gets
        \mathbb{D}_{N_0}^{(\beta,q)}
        \cup
        \{
            (X_k,V_k^{(\beta,q)})
        \}.
    $
}

\Return{$\mathbb{D}_{N_0}^{(\beta,q)}$}\;
\end{algorithm}

We will make use of the previous lemma as well as the following concentration of measure results; assuming that the measurements (outputs) are corrupted by log-Sobolev-type noise.

\begin{lem}[Concentration of the estimator]
\label{lem:Whitney_McShane_LSI_deviation}
Let $(X,\rho)$ be a bounded separable metric space and equip
$X\times\mathbb{R}$ with the $\ell^1$-product metric; denoted here by $\overline{\rho}$.
Let $\mu\in\mathcal{P}(X\times\mathbb{R})$ satisfy a log-Sobolev inequality with constant $C>0$ with respect to $\overline{\rho}$; that is, for every locally Lipschitz
$h:X\times\mathbb{R}\to\mathbb{R}$,
\[
    \operatorname{Ent}_{\mu}(h^2)
    \le
    C\int_{X\times\mathbb{R}}
        |\nabla h|^2\,d\mu,
\]
where
$
    |\nabla h|(z)
    \eqdef
    \limsup_{\widetilde z\to z}
        \frac{|h(z)-h(\widetilde z)|}{\overline{\rho}(z,\widetilde z)}.
$
For every $n\in\mathbb{N}_+$, let
$Z_n\eqdef (X_n,\varepsilon_n)$ be i.i.d.\ with law $\mu$, with
$\mathbb{E}[\varepsilon_1]=0$, and define the training data and fictional clean training data by
\[
        \mathbb{D}_N
    \eqdef
        \big((X_n,\varepsilon_n)\big)_{n=1}^N
\mbox{ and }
        \mathbb{D}_N^\star
    \eqdef
        \big((x_n^\star,0)\big)_{n=1}^N,
\]
where $x_1^\star,\ldots,x_N^\star\in X$ are deterministic.  Then, for every
$f\in \operatorname{Lip}((X,\rho);1)$, the quantity
\[
    \Phi_N
    \eqdef
    \big\|
        \mathcal{E}(f,\mathbb{D}_N)
        -
        \mathcal{E}(f,\mathbb{D}_N^\star)
    \big\|_\infty
\]
is a real-valued random variable, and for every $t>0$,
\[
    \mathbb{P}
    \left(
        \Phi_N
        \ge
        \mathbb{E}[\Phi_N]+t
    \right)
    \le
    \exp\left(
        -\frac{t^2}{4CN}
    \right).
\]
\end{lem}
\begin{proof}[{Proof of Lemma~\ref{lem:Whitney_McShane_LSI_deviation}}]
Abbreviate $   (Z_n)_{n=1}^N
    =
        \big((X_n,\varepsilon_n)\big)_{n=1}^N$.
Let
$
    \Psi(\mathbb{D})
\eqdef
    \big\|
        \mathcal{E}(f,\mathbb{D})
        -
        \mathcal{E}(f,\mathbb{D}_N^\star)
    \big\|_\infty
$ where $
    \mathbb{D}\in [X\times\mathbb{R}]^N$.  
Since $X$ is separable and the functions
$\mathcal{E}(f,\mathbb{D})-\mathcal{E}(f,\mathbb{D}_N^\star)$ are continuous on $X$, the supremum norm may be computed over a fixed countable dense subset of $X$.  Hence $\Psi$ is Borel measurable, and therefore
\[
    \Phi_N=\Psi(\mathbb{D}_N)
\]
is a real-valued random variable.
We next prove that $\Psi$ is $2$-Lipschitz with respect to the $\ell^1$-product metric on $[X\times\mathbb{R}]^N$.  Let
\[
    \mathbb{D}=((x_n,y_n))_{n=1}^N,
    \qquad
    \widetilde{\mathbb{D}}
    =
    ((\widetilde{x}_n,\widetilde{y}_n))_{n=1}^N.
\]
By the reverse triangle inequality and the stability estimate for the
$N$-point Whitney-McShane operator,
\[
\begin{aligned}
    |\Psi(\mathbb{D})-\Psi(\widetilde{\mathbb{D}})|
    % &
    \le
    \big\|
        \mathcal{E}(f,\mathbb{D})
        -
        \mathcal{E}(f,\widetilde{\mathbb{D}})
    \big\|_\infty
    % \\
    % &
    \le
    2\,\rho^{(1)}(\mathbb{D},\widetilde{\mathbb{D}}),
\end{aligned}
\]
where, by the $\ell^1$-product-metric convention. 
Thus $\Psi$ is $2$-Lipschitz.

By tensorization of the log-Sobolev inequality (\cite[Proposition 4]{naor2008concentration}, $\mu^{\otimes N}$ satisfies a log-Sobolev inequality with constant $CN$ with respect to the $\ell^1$-product metric $\rho^{(1)}$ on $[X\times\mathbb{R}]^N$.  Applying the Herbst theorem, cf.~\cite[Theorem 21]{naor2008concentration}, to the $2$-Lipschitz function $\Psi$ gives, for every $t>0$,
\[
    \mathbb{P}
    \left(
        \Psi(\mathbb{D}_N)
        \ge
        \mathbb{E}[\Psi(\mathbb{D}_N)]+t
    \right)
    \le
    \exp\left(
        -\frac{t^2}{4CN}
    \right).
\]
Since $\Phi_N=\Psi(\mathbb{D}_N)$, the claimed concentration estimate follows.
\end{proof}
Using the previous result, we are thus able to deduce the following.
\begin{lem}[Randomized Uniform Approximation: Noisy Case]
\label{lem:random_apprx_noisy_random_threshold}
Let $(X,\rho)$ be a compact metric space and equip $X\times\mathbb{R}$ with the
$\ell^1$-product metric $\overline{\rho}$.  Let $\mu\in\mathcal{P}(X\times\mathbb{R})$
and let $\mu_X\eqdef(\pi_X)_\#\mu$.  Let $Z_n\eqdef (X_n,\varepsilon_n)$,
$n\in[N]_+$, be i.i.d.\ with law $\mu$.  Fix $\varepsilon>0$, set
$M_\varepsilon\eqdef N_{\varepsilon/4}(X,\rho)$, and let
$x_1^\star,\ldots,x_{M_\varepsilon}^\star\in X$ be such that
$X=\bigcup_{m=1}^{M_\varepsilon}B_\rho(x_m^\star,\varepsilon/4)$.  Assume that
$
    p_\varepsilon
    \eqdef
    \min_{m\in[M_\varepsilon]_+}
    \mu_X\big(B_\rho(x_m^\star,\varepsilon/4)\big)
    >
    0.
$
Define
$\mathbb{D}_N\eqdef ((X_n,\varepsilon_n))_{n=1}^N$ and
$\mathbb{D}_N^0\eqdef ((X_n,0))_{n=1}^N$.  Then, for every
$f\in \operatorname{Lip}((X,\rho);1)$ and every $\delta\in(0,1)$, if
\begin{equation}
\label{eq:required_no_samples_noisy_random_threshold}
    N
    \ge
    \frac{1}{p_\varepsilon}
    \log\left(\frac{M_\varepsilon}{\delta}\right)
,
\end{equation}
upon defining $
    \Xi_N
    \eqdef
    \big\|
        \mathcal{E}(f|\mathbb{D}_N)
        -
        \mathcal{E}(f|\mathbb{D}_N^0)
    \big\|_\infty
$ we have
\begin{equation}
\label{eq:PAC_Guarantee_noisy_random_threshold}
    \mathbb{P}\left(
        \sup_{x\in X}\,
        \big|
            f(x)
            -
            \mathcal{E}(f|\mathbb{D}_N)(x)
        \big|
        \le
        \varepsilon
        +
        \Xi_N
    \right)
    \ge
    1-\delta
.
\end{equation}
\end{lem}
\begin{proof}[{Proof of Lemma~\ref{lem:random_apprx_noisy_random_threshold}}]
By the coupon-collector lemma applied at scale $\varepsilon/2$, condition
\eqref{eq:required_no_samples_noisy_random_threshold} implies that, with
probability at least $1-\delta$, the set $\{X_n\}_{n=1}^N$ is an
$\varepsilon/2$-net of $(X,\rho)$.  On this event, Lemma~\ref{lem:Whitney_McShane_uniform_approximation} gives
$
    \sup_{x\in X}
    |f(x)-\mathcal{E}(f|\mathbb{D}_N^0)(x)|
    \le
    \varepsilon.
$
Therefore, on the same event, the triangle inequality gives
\[
\begin{aligned}
    \sup_{x\in X}\,
    \big|
        f(x)
        -
        \mathcal{E}(f|\mathbb{D}_N)(x)
    \big|
    &\le
    \sup_{x\in X}\,
    \big|
        f(x)
        -
        \mathcal{E}(f|\mathbb{D}_N^0)(x)
    \big|
    +
    \big\|
        \mathcal{E}(f|\mathbb{D}_N^0)
        -
        \mathcal{E}(f|\mathbb{D}_N)
    \big\|_\infty
    \\
    &\le
    \varepsilon+\Xi_N.
\end{aligned}
\]
This proves \eqref{eq:PAC_Guarantee_noisy_random_threshold}. 
\end{proof}

\subsubsection{Controlling the Fluctuations $\Xi_N$}
It remains to estimate the error term $\Xi_N$ in Lemma~\ref{lem:random_apprx_noisy_random_threshold}.
\begin{lem}[Noise propagation through the estimator]
\label{lem:Whitney_McShane_noise_propagation}
Let $(X,\rho)$ be a metric space, let $f\in \operatorname{Lip}((X,\rho);1)$, and let
$\mathbb{D}_N=((x_n,\varepsilon_n))_{n=1}^N\in [X\times\mathbb{R}]^N$.  Define
$\mathbb{D}_N^0\eqdef ((x_n,0))_{n=1}^N$.  Then
\[
    \Xi_N
\eqdef
    \big\|
        \mathcal{E}(f|\mathbb{D}_N)
        -
        \mathcal{E}(f|\mathbb{D}_N^0)
    \big\|_\infty
\]
satisfies
$
    \Xi_N\le \max_{n\in[N]_+}|\varepsilon_n|.
$
\end{lem}
\begin{proof}[{Proof of Lemma~\ref{lem:Whitney_McShane_noise_propagation}}]
For each $z\in X$, set
$
    a_n(z)\eqdef f(x_n)-\rho(z,x_n).
$
Then
\linebreak
$
    \mathcal{E}(f|\mathbb{D}_N)(z)=\max_{n\in[N]_+}\{a_n(z)+\varepsilon_n\}
$
and
$
    \mathcal{E}(f|\mathbb{D}_N^0)(z)=\max_{n\in[N]_+}a_n(z).
$
Let $r_N\eqdef \max_{n\in[N]_+}|\varepsilon_n|$.  Since
$
    a_n(z)-r_N\le a_n(z)+\varepsilon_n\le a_n(z)+r_N
$
for every $n\in[N]_+$, taking maxima gives
$
    \max_{n\in[N]_+}a_n(z)-r_N
    \le
    \max_{n\in[N]_+}\{a_n(z)+\varepsilon_n\}
    \le
    \linebreak
    \max_{n\in[N]_+}a_n(z)+r_N.
$
Therefore
$
    \big|
        \mathcal{E}(f|\mathbb{D}_N)(z)
        -
        \mathcal{E}(f|\mathbb{D}_N^0)(z)
    \big|
    \le r_N.
$
Taking the 
\linebreak
supremum over $z\in X$ yields $\Xi_N\le r_N$, as claimed.
\end{proof}

\begin{lem}[Log-Sobolev Maximal Noise Bound]
\label{lem:LSI_maximal_noise}
Let $\varepsilon_1,\ldots,\varepsilon_N$ be independent random variables
with common law $\nu$.  Assume that $\nu$ is centred and satisfies
\eqref{eq:noise_LSI} with constant $C>0$.  Then, for every
$\delta\in(0,1)$,
\[
    \mathbb{P}\left(
        \max_{n\in[N]_+}|\varepsilon_n|
        \le
        \sqrt{
            C
            \log\left(
                \frac{2N}{\delta}
            \right)
        }
    \right)
    \ge
    1-\delta.
\]
\end{lem}
\begin{proof}[{Proof of Lemma~\ref{lem:LSI_maximal_noise}}]
The coordinate projection $\pi_{\mathbb{R}}:X\times\mathbb{R}\to\mathbb{R}$,
defined by $\pi_{\mathbb{R}}(x,y)=y$, is $1$-Lipschitz from
$(X\times\mathbb{R},\overline{\rho})$ to $(\mathbb{R},|\cdot|)$.  Hence, by the
log-Sobolev concentration estimate applied to $\pi_{\mathbb{R}}$, and since
$\mathbb{E}[\varepsilon_1]=0$, for every $t>0$,
$
        \mathbb{P}(|\varepsilon_1|\ge t)
    \le
        2\exp(-t^2/C)
$.
Therefore, by the union bound,
\[
\begin{aligned}
    \mathbb{P}
    \left(
        \max_{n\in[N]_+}|\varepsilon_n|\ge t
    \right)
    % &
    =
    \mathbb{P}
    \left(
        \bigcup_{n=1}^N
        \{|\varepsilon_n|\ge t\}
    \right)
    % \\
    % &
    \le
    \sum_{n=1}^N
    \mathbb{P}(|\varepsilon_n|\ge t)
    % \\
    % &
    \le
    2N\exp(-t^2/C).
\end{aligned}
\]
Taking $t=\sqrt{C\log(2N/\delta)}$ gives the claimed high-probability bound.
\end{proof}
Upon combining the previous two technical lemmata we obtain.
\begin{lem}[High-Probability Noise Propagation through the estimator]
\label{lem:Whitney_McShane_noise_HP}
Let $(X,\rho)$ be a metric space and equip $X\times\mathbb{R}$ with the
$\ell^1$-product metric $\overline{\rho}$.  Let
$\mu\in\mathcal{P}(X\times\mathbb{R})$ satisfy a log-Sobolev inequality with
constant $C>0$ with respect to $\overline{\rho}$.  Let
$Z_n\eqdef (X_n,\varepsilon_n)$, $n\in[N]_+$, be i.i.d.\ with law $\mu$, and
assume that $\mathbb{E}[\varepsilon_1]=0$.  For every
$f\in \operatorname{Lip}((X,\rho);1)$, define
$\mathbb{D}_N\eqdef ((X_n,\varepsilon_n))_{n=1}^N$ and
$\mathbb{D}_N^0\eqdef ((X_n,0))_{n=1}^N$.  Then, for every $\delta\in(0,1)$,
\begin{equation}
\label{eq:fluctuators}
        \mathbb{P}
        \Big(
            \Xi_N
            \le
            \sqrt{C\log(2N/\delta)}
        \Big)
    \ge
        1-\delta
.
\end{equation}
\end{lem}

\begin{proof}[{Proof of Lemma~\ref{lem:Whitney_McShane_noise_HP}}]
Set
$
    \Xi_N
    \eqdef
    \big\|
        \mathcal{E}(f|\mathbb{D}_N)
        -
        \mathcal{E}(f|\mathbb{D}_N^0)
    \big\|_\infty.
$
By Lemma~\ref{lem:Whitney_McShane_noise_propagation}, for every realization of
$\mathbb{D}_N$,
$
    \Xi_N\le \max_{n\in[N]_+}|\varepsilon_n|.
$
Therefore,
\begin{equation}
\label{eq:lala1}
        \biggl\{\max_{n\in[N]_+}|\varepsilon_n|
            \le
            \sqrt{C\log(2N/\delta)}
        \biggr\}
\subseteq
    \biggl\{
        \Xi_N
        \le
        \sqrt{C\log(2N/\delta)}
    \biggr\}
.
\end{equation}
By Lemma~\ref{lem:LSI_maximal_noise},
we have 
\begin{equation}
\label{eq:lala2}
        \mathbb{P}
        \biggl(
            \max_{n\in[N]_+}|\varepsilon_n|
            \le
            \sqrt{C\log(2N/\delta)}
        \biggr)
    \ge
        1-\delta
.
\end{equation}
Combining~\eqref{eq:lala1} and~\eqref{eq:lala2} we deduce~\eqref{eq:fluctuators}.
\end{proof}

\subsubsection{Completing The Argument: Non-Robust Version (Proposition~\ref{prop:random_apprx_noisy})}
Using the previous lemmata, we deduce the theorem~\ref{prop:random_apprx_noisy}.
\begin{proof}[{Proof of Theorem~\ref{prop:random_apprx_noisy}}]
Write $M_\varepsilon\eqdef N_{\varepsilon/4}(X,\rho)$.
Let $x_1^\star,\ldots,x_{M_\varepsilon}^\star\in X$ be such that
$X=\bigcup_{m=1}^{M_\varepsilon}B_\rho(x_m^\star,\varepsilon/4)$.  By lower
Ahlfors regularity,
$
    \mu_X(B_\rho(x_m^\star,\varepsilon/4))
    \ge
    a(\varepsilon/4)^s
$
for every $m\in[M_\varepsilon]_+$.  Hence
$
    p_\varepsilon
    \eqdef
    \min_{m\in[M_\varepsilon]_+}
    \mu_X(B_\rho(x_m^\star,\varepsilon/4))
    \ge
    a(\varepsilon/4)^s.
$
Therefore \eqref{eq:required_no_samples_noisy_Ahlfors} implies
$
    N\ge p_\varepsilon^{-1}
    \log(M_\varepsilon/\delta_{\operatorname{net}}).
$
Define the clean training dataset
$
    \mathbb{D}_N^0\eqdef ((X_n,0))_{n=1}^N
$
and set
$
    \Xi_N
    \eqdef
    \|\mathcal{E}(f|\mathbb{D}_N)-\mathcal{E}(f|\mathbb{D}_N^0)\|_\infty.
$
By Lemma~\ref{lem:random_apprx_noisy_random_threshold},
\begin{equation}
\label{eq:uni1}
    \mathbb{P}\left(
        \sup_{x\in X}\,
        \big|
            f(x)
            -
            \mathcal{E}(f|\mathbb{D}_N)(x)
        \big|
        <
        \varepsilon+\Xi_N
    \right)
    \ge
    1-\delta_{\operatorname{net}}.
\end{equation}
By Lemma~\ref{lem:Whitney_McShane_noise_HP}, we find that
\begin{equation}
\label{eq:uni2}
    \mathbb{P}\left(
        \Xi_N
        \le
        \sqrt{C\log(2N/\delta_{\operatorname{noise}})}
    \right)
    \ge
    1-\delta_{\operatorname{noise}}.
\end{equation}
On the intersection of these two events,
$
    \sup_{x\in X}|f(x)-\mathcal{E}(f|\mathbb{D}_N)(x)|
    <
    \varepsilon+
    \linebreak
    \sqrt{C\log(2N/\delta_{\operatorname{noise}})}.
$
Upon taking union bound,~\eqref{eq:uni1} and~\eqref{eq:uni2} yield the conclusion.
\end{proof}

\subsection{neighbourhood-averaging: Robustification}
\label{s:robustification}
We now robustify our formula via local averaging procedure.  We first present the general lemma, before specializing it to our setting.
\begin{lem}[Local averaging before Whitney-McShane]
\label{lem:local_averaging_before_Whitney}
Let $(X,\rho)$ be a compact metric space, let $f\in \operatorname{Lip}((X,\rho);1)$,
and let $x_1^\star,\ldots,x_K^\star\in X$ be an $\eta$-net of $X$.  Fix
$r>0$.  Suppose that, for every $k\in[K]_+$, we are given points
$X_{k,1},\ldots,X_{k,M}\in B_\rho(x_k^\star,r)$ and observations
$Y_{k,j}\eqdef f(X_{k,j})+\varepsilon_{k,j}$, where the noises
$\varepsilon_{k,j}$ are centred and satisfy, for every $t>0$ and every
$k\in[K]_+$,
$
    \mathbb{P}\big(|M^{-1}\sum_{j=1}^M\varepsilon_{k,j}|\ge t\big)
    \le
    2\exp(-Mt^2/C).
$
Define the locally averaged labels
$
    \overline{Y}_k\eqdef M^{-1}\sum_{j=1}^M Y_{k,j}.
$
Then, for every $\delta\in(0,1)$,
\[
        \mathbb{P}\biggl(
            \max_{k\in[K]_+}|\overline{Y}_k-f(x_k^\star)|
            \le
            r+\sqrt{C M^{-1}\log(2K/\delta)}
        \biggr)
    \ge 
        1-\delta
.
\]
\end{lem}

\begin{proof}[{Proof of Lemma~\ref{lem:local_averaging_before_Whitney}}]
For every $k\in[K]_+$,
$
    \overline{Y}_k-f(x_k^\star)
    =
    M^{-1}\sum_{j=1}^M\big(f(X_{k,j})-f(x_k^\star)\big)
    \linebreak
    +
    M^{-1}\sum_{j=1}^M\varepsilon_{k,j}.
$
Since $f$ is $1$-Lipschitz and $X_{k,j}\in B_\rho(x_k^\star,r)$,
$
    |M^{-1}\sum_{j=1}^M(f(X_{k,j})-f(x_k^\star))|\le r.
$
By the sub-Gaussian estimate and a union bound over $k\in[K]_+$, we have
\[
    \max_{k\in[K]_+}
    \left|
        M^{-1}\sum_{j=1}^M\varepsilon_{k,j}
    \right|
    \le
    \sqrt{C M^{-1}\log(2K/\delta)}
\]
with probability at least $1-\delta$.  Hence, on this event,
$
    \max_{k\in[K]_+}|\overline{Y}_k-f(x_k^\star)|
    \le
    r+\sqrt{C M^{-1}\log(2K/\delta)}.
$
\end{proof}
We need to estimate the stability of the Whitney-McShane formula as a function of the input data.
We also deduce the following Lipschitz stability bound.

\begin{lem}[Stability after local averaging]
\label{lem:stability_after_local_averaging}
Assume the hypotheses of Lemma~\ref{lem:local_averaging_before_Whitney}.  Define
$
    \overline{\mathbb{D}}_K
    \eqdef
    ((x_k^\star,\overline{Y}_k-f(x_k^\star)))_{k=1}^K
$
and
$
    \mathbb{D}_K^\star
    \eqdef
    ((x_k^\star,0))_{k=1}^K.
$
Then, for every $\delta\in(0,1)$,
\[
    \mathbb{P}\left(
        \big\|
            \mathcal{E}(f|\overline{\mathbb{D}}_K)
            -
            \mathcal{E}(f|\mathbb{D}_K^\star)
        \big\|_\infty
        \le
        r+\sqrt{C M^{-1}\log(2K/\delta)}
    \right)
    \ge
    1-\delta
.
\]
\end{lem}
\begin{proof}[{Proof of Lemma~\ref{lem:stability_after_local_averaging}}]
Set
$
    \Delta_K
    \eqdef
    \max_{k\in[K]_+}|\overline{Y}_k-f(x_k^\star)|.
$
For each $z\in X$, write
\linebreak
$
    a_k(z)\eqdef f(x_k^\star)-\rho(z,x_k^\star).
$
Then
$
    \mathcal{E}(f|\overline{\mathbb{D}}_K)(z)
    =
    \max_{k\in[K]_+}\{a_k(z)+\overline{Y}_k-f(x_k^\star)\}
$
and
$
    \mathcal{E}(f|\mathbb{D}_K^\star)(z)
    =
    \max_{k\in[K]_+}a_k(z).
$
Since
$
    -\Delta_K
    \le
    \overline{Y}_k-f(x_k^\star)
    \le
    \Delta_K
$
for every $k\in[K]_+$, taking maxima gives
$
    \big|
        \mathcal{E}(f|\overline{\mathbb{D}}_K)(z)
        -
        \mathcal{E}(f|\mathbb{D}_K^\star)(z)
    \big|
    \le
    \Delta_K.
$
Taking the supremum over $z\in X$ yields
$
    \|\mathcal{E}(f|\overline{\mathbb{D}}_K)
        -
        \mathcal{E}(f|\mathbb{D}_K^\star)\|_\infty
    \le
    \Delta_K.
$
By Lemma~\ref{lem:local_averaging_before_Whitney}, with probability at least $1-\delta$,
$
    \Delta_K
    \le
    r+\sqrt{C M^{-1}\log(2K/\delta)}
$.
\end{proof}

We now derive our main full-version of our reconstruction theorem; which relies on the following technical lemma.

\begin{lem}[Uniform Stability of Central Envelopes]
\label{lem:central_midpoint_stability}
Let $(X,\rho)$ be a metric space, let $L>0$, and let $f:X\to\mathbb{R}$ be $L$-Lipschitz. Let $K\in\mathbb{N}_+$, let $x_1,\ldots,x_K\in X$, and suppose that there exists $\eta\ge0$ such that
\begin{equation}
\label{eq:central_midpoint_covering_condition}
    \sup_{x\in X}
    \min_{k\in[K]_+}
    \rho(x,x_k)
    \le
    \eta.
\end{equation}
Let $v_1,\ldots,v_K\in\mathbb{R}$ and suppose that, for some $\Delta\ge0$, one has $\max_{k\in[K]_+}|v_k-f(x_k)|\le\Delta$.
For every $x\in X$, define
\begin{equation*}
\begin{aligned}
    U(x)
    &\eqdef
    \min_{k\in[K]_+}
    \left\{
        v_k+L\rho(x,x_k)
    \right\},
    \\
    W(x)
    &\eqdef
    \max_{k\in[K]_+}
    \left\{
        v_k-L\rho(x,x_k)
    \right\},
    \\
    \hat{f}(x)
    &\eqdef
    \frac{U(x)+W(x)}{2}.
\end{aligned}
\end{equation*}
Then $U$, $W$, and $\hat{f}$ are $L$-Lipschitz and for every $x\in X$
\begin{equation}
\label{eq:central_midpoint_stability}
    |\hat{f}(x)-f(x)|
\le
    \Delta+L\eta.
\end{equation}
\end{lem}
\begin{proof}[{Proof of Lemma~\ref{lem:central_midpoint_stability}}]
Since $K\in\mathbb{N}_+$, each minimum and maximum in the definitions of $U$ and $W$ is taken over a finite non-empty set of real numbers. Hence $U$, $W$, and $\hat{f}$ are well defined and real valued on $X$.
We first verify the Lipschitz claims. Fix $x,y\in X$. Since the index set is finite, there exist $k_y^+\in[K]_+$ and $k_x^-\in[K]_+$ such that
\[
    U(y)
    =
    v_{k_y^+}
    +
    L\rho(y,x_{k_y^+})
    \qquad\text{and}\qquad
    W(x)
    =
    v_{k_x^-}
    -
    L\rho(x,x_{k_x^-}).
\]
Using the definitions of $U$ and $W$, followed by the reverse triangle inequality, gives
\begin{equation*}
\begin{aligned}
    U(x)-U(y)
    &\le
    L\left(
        \rho(x,x_{k_y^+})
        -
        \rho(y,x_{k_y^+})
    \right)
    \le
    L\rho(x,y),
    \\
    W(x)-W(y)
    &\le
    L\left(
        \rho(y,x_{k_x^-})
        -
        \rho(x,x_{k_x^-})
    \right)
    \le
    L\rho(x,y).
\end{aligned}
\end{equation*}
Interchanging $x$ and $y$ in these inequalities yields $|U(x)-U(y)|\le L\rho(x,y)$ and $|W(x)-W(y)|\le L\rho(x,y)$. Therefore,
\[
    |\hat{f}(x)-\hat{f}(y)|
    \le
    \frac{|U(x)-U(y)|+|W(x)-W(y)|}{2}
    \le
    L\rho(x,y),
\]
so $U$, $W$, and $\hat{f}$ are $L$-Lipschitz.
We next prove the following pointwise bounds: for every $x\in X$,
\begin{equation}
\label{eq:central_envelope_bounds}
\begin{aligned}
    f(x)-\Delta
    &\le
    U(x)
    \le
    f(x)+\Delta+2L\eta,
    \\
    f(x)-\Delta-2L\eta
    &\le
    W(x)
    \le
    f(x)+\Delta.
\end{aligned}
\end{equation}
Fix $x\in X$. For every $k\in[K]_+$, the assumption $|v_k-f(x_k)|\le\Delta$ and the $L$-Lipschitz continuity of $f$ imply
\begin{equation*}
\begin{aligned}
    v_k+L\rho(x,x_k)
    &\ge
    f(x_k)-\Delta+L\rho(x,x_k)
    \ge
    f(x)-\Delta,
    \\
    v_k-L\rho(x,x_k)
    &\le
    f(x_k)+\Delta-L\rho(x,x_k)
    \le
    f(x)+\Delta.
\end{aligned}
\end{equation*}
Taking the minimum in the first inequality and the maximum in the second gives
\[
    U(x)\ge f(x)-\Delta
    \qquad\text{and}\qquad
    W(x)\le f(x)+\Delta.
\]
By~\eqref{eq:central_midpoint_covering_condition}, choose $k_x\in[K]_+$ attaining the finite minimum, so that $\rho(x,x_{k_x})\le\eta$. Evaluating the minimum and maximum at this index and again using $|v_{k_x}-f(x_{k_x})|\le\Delta$ and the Lipschitz continuity of $f$, we obtain
\begin{equation*}
\begin{aligned}
    U(x)
    &\le
    v_{k_x}+L\rho(x,x_{k_x})
    \le
    f(x)+\Delta+2L\rho(x,x_{k_x})
    \le
    f(x)+\Delta+2L\eta,
    \\
    W(x)
    &\ge
    v_{k_x}-L\rho(x,x_{k_x})
    \ge
    f(x)-\Delta-2L\rho(x,x_{k_x})
    \ge
    f(x)-\Delta-2L\eta.
\end{aligned}
\end{equation*}
This proves~\eqref{eq:central_envelope_bounds}.
Averaging the lower bounds for $U(x)$ and $W(x)$ in~\eqref{eq:central_envelope_bounds}, and then averaging their upper bounds, yields
\[
    f(x)-\Delta-L\eta
    \le
    \hat{f}(x)
    \le
    f(x)+\Delta+L\eta.
\]
Thus $|\hat{f}(x)-f(x)|\le\Delta+L\eta$. Since $x\in X$ was arbitrary, taking the supremum over $x\in X$ proves~\eqref{eq:central_midpoint_stability}.
\end{proof}

%%%
The following is the technical version of our main result.
\begin{theorem}[Split-Sample Uniform Recovery]
\label{thrm:random_apprx_noisy_robustified}
Suppose that Assumption~\ref{assumption:main_1Lip} holds.  Let
$N_0,N_1\in\mathbb{N}_+$ and set $N\eqdef N_0+N_1$.  Fix
$\varepsilon\in(0,1)$ and
$\delta_{\operatorname{net}},
 \delta_{\operatorname{noise}}\in(0,1)$
such that
$
    \varepsilon/2\le\operatorname{diam}(X)
$.
Set
\[
    \eta\eqdef\frac{\varepsilon}{4},
    \qquad
    \beta\eqdef\frac{\varepsilon}{2},
    \qquad
    q_\varepsilon
    \eqdef
    \left\lceil\varepsilon^{-2}\right\rceil.
\]
Use the first $N_0$ observations as the reference block and the remaining
$N_1$ observations as the averaging block.  Assume that
\begin{equation}
\label{eq:technical_robustified_sample_condition}
\begin{aligned}
    N_0
    &\ge
    \frac{1}{a(\varepsilon/8)^s}
    \log\left(
        \frac{
            2\mathcal N(X,\rho,\varepsilon/8)
        }{
            \delta_{\operatorname{net}}
        }
    \right),
    \\
    N_1
    &\ge
    \frac{1}{a(\varepsilon/2)^s}
    \max\left\{
        2q_\varepsilon,
        8\log\left(
            \frac{
                2N_0
            }{
                \delta_{\operatorname{net}}
            }
        \right)
    \right\}.
\end{aligned}
\end{equation}

For each $k\in[N_0]_+$, set $\xi_k\eqdef X_k$.  If at least
$q_\varepsilon$ indices
$j\in\{N_0+1,\ldots,N_0+N_1\}$ satisfy
$
    \rho(X_j,\xi_k)\le\beta,
$
let
$
    j_{k,1}<\cdots<j_{k,q_\varepsilon}
$
be the first $q_\varepsilon$ such indices and define
\[
    \overline Y_k
    \eqdef
    \frac{1}{q_\varepsilon}
    \sum_{\ell=1}^{q_\varepsilon}
        Y_{j_{k,\ell}}.
\]
If fewer than $q_\varepsilon$ such indices exist, set
$
    \overline Y_k\eqdef0
$.

Define, for every $x\in X$,
\begin{equation}
\label{eq:technical_robustified_estimator}
    \hat{f}(x)
    \eqdef
    \frac{1}{2}
    \left[
        \min_{k\in[N_0]_+}
        \left\{
            \overline Y_k+\rho(x,\xi_k)
        \right\}
        +
        \max_{k\in[N_0]_+}
        \left\{
            \overline Y_k-\rho(x,\xi_k)
        \right\}
    \right].
\end{equation}
Then, for every fixed
$f\in\operatorname{Lip}((X,\rho);1)$:

\begin{enumerate}
    \item[(i)] $\hat{f}:X\to\mathbb{R}$ is $1$-Lipschitz for every
    realization of the sample;

    \item[(ii)] the random variable
    $
        \|\hat{f}-f\|_\infty
    $
    is measurable and satisfies
    \begin{equation}
    \label{eq:technical_robustified_uniform_bound}
        \mathbb{P}\left(
            \|\hat{f}-f\|_\infty
            \le
            \varepsilon
            \left[
                1
                +
                \sqrt{
                    C
                    \log\left(
                        \frac{
                            2N_0
                        }{
                            \delta_{\operatorname{noise}}
                        }
                    \right)
                }
            \right]
        \right)
        \ge
        1
        -
        \delta_{\operatorname{net}}
        -
        \delta_{\operatorname{noise}}.
    \end{equation}
\end{enumerate}
\end{theorem}
In the finite-dimensional Euclidean random-design setting covered by the
classical minimax theory, the polynomial exponent in our rate agrees with
the minimax sup-norm exponent for first-order Lipschitz regression, up to
logarithmic factors; see~\cite{Stone1982OptimalGlobalRates}.  We use this as
a rate comparison.  We do not claim here a matching minimax lower bound for
the full class of lower-Ahlfors-regular metric-space models appearing in
Theorem~\ref{thrm:random_apprx_noisy_robustified}.
Moreover, Theorem~\ref{thrm:random_apprx_noisy_robustified} gives a high-probability
uniform guarantee for $\hat{f}_{\mathcal{D}}$, and hence a high-probability
population-risk guarantee.  The bias-variance decomposition is a weaker
second-moment summary of the same estimation problem; cf.\ bias-variance decompositions for kernel ridge regression~\cite{pmlr-v235-cheng24g,cheng2024comprehensive}.
\begin{proof}[{Proof of Theorem~\ref{thrm:random_apprx_noisy_robustified}}]
We first construct the reference net.  Apply
Lemma~\ref{lem:coupon_collector_random_net} with target accuracy
$\eta=\varepsilon/4$.  Thus the covering balls used in that lemma have radius
$\eta/2=\varepsilon/8$.  By lower Ahlfors $s$-regularity, every such ball has
$\mu_X$-mass at least $a(\varepsilon/8)^s$.  Therefore the first inequality in
\eqref{eq:technical_robustified_sample_condition} implies
\[
    \mathbb{P}\left(
        \{\xi_k\}_{k=1}^{N_0}
        \mbox{ is an } \eta\mbox{-net of }(X,\rho)
    \right)
    \ge
    1-\frac{\delta_{\operatorname{net}}}{2}.
\]
Denote this event by $\mathcal{G}_{\operatorname{net}}$.
We next show that the second block supplies $M$ nearby samples around each
reference point.  Condition on the first block $(X_k)_{k=1}^{N_0}$ and set
$B_k\eqdef B_\rho(\xi_k,r)$, where $r=\varepsilon/2$.  Since $r\le
\operatorname{diam}(X)$, lower Ahlfors regularity gives
$
    \mu_X(B_k)\ge ar^s=a(\varepsilon/2)^s
$
for every $k\in[N_0]_+$.  Define
$
    S_k
    \eqdef
    \#\{j\in\{N_0+1,\ldots,N_0+N_1\}:X_j\in B_k\}.
$
Conditional on the first block, $S_k$ is binomial with parameters $N_1$ and
$q_k\eqdef\mu_X(B_k)$, where $q_k\ge a(\varepsilon/2)^s$.  By
\eqref{eq:technical_robustified_sample_condition}, $N_1q_k\ge 2M$.  
Conditional on $(X_1,\ldots,X_{N_0})$, the random variable $S_k$ is
binomial with parameters $N_1$ and
$
    p_k
\eqdef
    \mu_X(B_\rho(\xi_k,\beta))$.
Let
$
    \mu_k\eqdef N_1p_k
$
denote its conditional mean.  By lower Ahlfors regularity and
\eqref{eq:technical_robustified_sample_condition}, $
    \mu_k
    \ge
    N_1a\beta^s
    \ge
    2q_\varepsilon$
and hence $
    \{S_k<q_\varepsilon\}
    \subseteq
    \{S_k\le\mu_k/2\}.
$.
For every $\lambda>0$, conditional Markov's inequality gives
\[
\begin{aligned}
    \mathbb{P}\left(
        S_k\le\frac{\mu_k}{2}
        \,\middle|\,
        X_1,\ldots,X_{N_0}
    \right)
    &=
    \mathbb{P}\left(
        e^{-\lambda S_k}
        \ge
        e^{-\lambda\mu_k/2}
        \,\middle|\,
        X_1,\ldots,X_{N_0}
    \right)
    \\
    &\le
    e^{\lambda\mu_k/2}
    \mathbb{E}\left[
        e^{-\lambda S_k}
        \,\middle|\,
        X_1,\ldots,X_{N_0}
    \right]
    \\
    &=
    e^{\lambda\mu_k/2}
    \left(
        1-p_k+p_ke^{-\lambda}
    \right)^{N_1}
    \\
    &\le
    \exp\left(
        \mu_k
        \left[
            \frac{\lambda}{2}
            +
            e^{-\lambda}
            -
            1
        \right]
    \right).
\end{aligned}
\]
Taking $\lambda=\log 2$ and using $\log 2\le3/4$ yields
$
    \tfrac{\log 2}{2}
    +
    \tfrac{1}{2}
    -
    1
    \le
    -\frac{1}{8}$; therefore
\begin{equation}
\label{eq:Chernoffed}
\begin{aligned}
    \mathbb{P}\left(
        S_k<q_\varepsilon
        \,\middle|\,
        X_1,\ldots,X_{N_0}
    \right)
    &\le
    \exp\left(
        -\frac{\mu_k}{8}
    \right)
    \\
    &\le
    \exp\left(
        -\frac{
            N_1a(\varepsilon/2)^s
        }{8}
    \right).
\end{aligned}
\end{equation}
By the second inequality in
\eqref{eq:technical_robustified_sample_condition},
$
    \exp\big(
        -\tfrac{
            N_1a(\varepsilon/2)^s
        }{8}
    \big)
    \le
    \tfrac{
        \delta_{\operatorname{net}}
    }{
        2N_0
    }
$.    Taking a union bound over $k\in[N_0]_+$ gives
\[
    \mathbb{P}\left(
        \min_{k\in[N_0]_+}S_k
        <
        q_\varepsilon
        \,\middle|\,
        X_1,\ldots,X_{N_0}
    \right)
    \le
    \frac{\delta_{\operatorname{net}}}{2}.
\]
Upon taking a union bound over $k\in[N_0]_+$ and using the second inequality in
\eqref{eq:technical_robustified_sample_condition}, we obtain $
    \mathbb{P}\big(
        \exists k\in[N_0]_+:S_k<M
        \,\big|\,
        X_1,\ldots,X_{N_0}
    \big)
    \le
    \delta_{\operatorname{net}}/2
$.  Consequently, with probability at least $1-\delta_{\operatorname{net}}$, both $\mathcal{G}_{\operatorname{net}}$ holds and every ball $B_k$ contains at least
$M$ samples from the second block.  Denote this intersection by $\mathcal{G}_{\operatorname{occ}}$. Let
\[
    \mathscr X_N
    \eqdef
    \sigma(X_1,\ldots,X_{N_0+N_1})
\]
be the sigma-algebra generated by all sample locations.  Define the occupancy event
\[
    \mathcal G_{\operatorname{occ}}
    \eqdef
    \bigcap_{k=1}^{N_0}
    \left\{
        S_k\ge q_\varepsilon
    \right\}.
\]
Because the selected indices are chosen in increasing sample-index order,
each map
\[
    (X_1,\ldots,X_{N_0+N_1})
    \longmapsto
    j_{k,\ell}
\]
is $\mathscr X_N$-measurable on $\mathcal G_{\operatorname{occ}}$.  By the product assumption $
    \mu=\mu_X\otimes\nu
$ and independence across observations, conditional on $\mathscr X_N$ the
random variables $
    \varepsilon_1,\ldots,\varepsilon_{N_0+N_1}
$ are independent and each has law $\nu$.  Consequently, for every fixed $k\in[N_0]_+$, on $\mathcal G_{\operatorname{occ}}$ the selected vector
\[
    \left(
        \varepsilon_{j_{k,1}},
        \ldots,
        \varepsilon_{j_{k,q_\varepsilon}}
    \right)
\]
has conditional law $\nu^{\otimes q_\varepsilon}$. By tensorization of~\eqref{eq:noise_LSI}, the measure $\nu^{\otimes q_\varepsilon}$ satisfies the same log-Sobolev inequality with constant $C$ with respect to the Euclidean product metric.  The map
\[
    (z_1,\ldots,z_{q_\varepsilon})
    \longmapsto
    \frac{1}{q_\varepsilon}
    \sum_{\ell=1}^{q_\varepsilon}z_\ell
\]
has Lipschitz constant $q_\varepsilon^{-1/2}$ (with respect to the $\ell^2$ metric).  The Herbst
argument therefore yields, for every $t>0$,
\[
\begin{aligned}
    &\mathbb{P}\left(
        \left|
            \frac{1}{q_\varepsilon}
            \sum_{\ell=1}^{q_\varepsilon}
                \varepsilon_{j_{k,\ell}}
        \right|
        \ge t
        \,\middle|\,
        \mathscr X_N
    \right)
    \\
    &\hspace{4cm}\le
    2\exp\left(
        -\frac{
            q_\varepsilon t^2
        }{C}
    \right)
\end{aligned}
\]
on $\mathcal G_{\operatorname{occ}}$.
 The selected index sets may overlap for different values of $k$.  Hence the corresponding local averages need not be conditionally independent. No such independence is used: a conditional union bound gives
\[
\begin{aligned}
    &\mathbb{P}\left(
        \max_{k\in[N_0]_+}
        \left|
            \frac{1}{q_\varepsilon}
            \sum_{\ell=1}^{q_\varepsilon}
                \varepsilon_{j_{k,\ell}}
        \right|
        \ge t
        \,\middle|\,
        \mathscr X_N
    \right)
    \\
    &\hspace{2cm}\le
    2N_0
    \exp\left(
        -\frac{
            q_\varepsilon t^2
        }{C}
    \right)
\end{aligned}
\]
on $\mathcal G_{\operatorname{occ}}$.  
Set
$
    t_\varepsilon
    \eqdef
    \sqrt{
        \tfrac{C}{q_\varepsilon}
        \log\big(
            \tfrac{
                2N_0
            }{
                \delta_{\operatorname{noise}}
            }
        \big)
    }$, and define $
    \mathcal G_{\operatorname{noise}}
\eqdef
    \{
        \max_{k\in[N_0]_+}
        \big|
            \tfrac{1}{q_\varepsilon}
            \sum_{\ell=1}^{q_\varepsilon}
                \varepsilon_{j_{k,\ell}}
        \big|
        \le
        t_\varepsilon
    \}
$. 
Then
\[
    \mathbb{P}\left(
        \mathcal G_{\operatorname{occ}}
        \cap
        \mathcal G_{\operatorname{noise}}^c
    \right)
    =
    \mathbb{E}\left[
        \mathbbm{1}_{\mathcal G_{\operatorname{occ}}}
        \mathbb{P}\left(
            \mathcal G_{\operatorname{noise}}^c
            \,\middle|\,
            \mathscr X_N
        \right)
    \right]
    \le
    \delta_{\operatorname{noise}}.
\]
This shows that $
    \mathbb{P}\left(
        \mathcal G_{\operatorname{net}}^c
    \right)
    \le
    \tfrac{\delta_{\operatorname{net}}}{2}
$,
and the occupancy argument gives
$
    \mathbb{P}\left(
        \mathcal G_{\operatorname{occ}}^c
    \right)
    \le
    \tfrac{\delta_{\operatorname{net}}}{2}
$.
Therefore
\[
\begin{aligned}
    &\mathbb{P}\left(
        \mathcal G_{\operatorname{net}}
        \cap
        \mathcal G_{\operatorname{occ}}
        \cap
        \mathcal G_{\operatorname{noise}}
    \right)
    \\
    &\quad\ge
    1
    -
    \mathbb{P}\left(
        \mathcal G_{\operatorname{net}}^c
    \right)
    -
    \mathbb{P}\left(
        \mathcal G_{\operatorname{occ}}^c
    \right)
    -
    \mathbb{P}\left(
        \mathcal G_{\operatorname{occ}}
        \cap
        \mathcal G_{\operatorname{noise}}^c
    \right)
    \\
    &\quad\ge
    1
    -
    \delta_{\operatorname{net}}
    -
    \delta_{\operatorname{noise}}.
\end{aligned}
\]

Since $|e_k|\le \max_{\ell\in[N_0]_+}|e_\ell|$ for every $k$, taking maxima gives
$
    |\hat{f}(x)-f_{N_0}^{\circ}(x)|
    \le
    \max_{k\in[N_0]_+}|\overline{Y}_k-f(\xi_k)|.
$
Therefore,
\[
    \|\hat{f}-f_{N_0}^{\circ}\|_\infty
    \le
    r+
    \sqrt{
        \frac{C}{M}
        \log\left(
            \frac{2N_0}{\delta_{\operatorname{noise}}}
        \right)
    }.
\]
Since $\{\xi_k\}_{k=1}^{N_0}$ is an $\eta$-net on
$\mathcal{G}_{\operatorname{net}}$, Lemma~\ref{lem:Whitney_McShane_uniform_approximation}
applied with accuracy $2\eta$ gives
\linebreak
$
    \sup_{x\in X}|f(x)-f_{N_0}^{\circ}(x)|
    \le
    2\eta.
$
We thus have
\[
\begin{aligned}
    \sup_{x\in X}
    |f(x)-\hat{f}(x)|
\le
    \sup_{x\in X}|f(x)-f_{N_0}^{\circ}(x)|
    +
    \|f_{N_0}^{\circ}-\hat{f}\|_\infty
\le
    2\eta+r+
    \sqrt{
        \frac{C}{M}
        \log\left(
            \frac{2N_0}{\delta_{\operatorname{noise}}}
        \right)
    }.
\end{aligned}
\]
Since $2\eta+r=2(\varepsilon/4)+\varepsilon/2=\varepsilon$ and
$M=\lceil\varepsilon^{-2}\rceil$ implies $M^{-1}\le \varepsilon^2$, we obtain \eqref{eq:relu_robustified_uniform_bound} since
\[
    \sup_{x\in X}
    \,
    |f(x)-\hat{f}(x)|
\le
    \varepsilon
    +
    \varepsilon
    \sqrt{
        C
        \log\left(
            \frac{2N_0}{\delta_{\operatorname{noise}}}
        \right)
    }.
\]
It remains only to verify that $\hat{f}$ is $1$-Lipschitz.  For each
$k\in[N_0]_+$, the map $x\mapsto \overline{Y}_k-\rho(x,\xi_k)$ is
$1$-Lipschitz, because
$
    |\rho(x,\xi_k)-\rho(x',\xi_k)|\le \rho(x,x')
$
for all $x,x'\in X$.  The pointwise maximum of finitely many $1$-Lipschitz
functions is again $1$-Lipschitz.  Hence $\hat{f}$ is %$\mathbb{P}$-a.s.\
$1$-Lipschitz.
\end{proof}

We now deduce convergence.
\begin{lem}[Lower Ahlfors regularity implies polynomial metric entropy]
\label{lem:lower_Ahlfors_implies_entropy_bound}
Suppose that $(X,\rho)$ is compact and that $\mu_X$ is lower Ahlfors
$s$-regular with constant $a>0$; namely,
\[
    \mu_X(B_\rho(x,r))
    \ge
    a r^s
\]
for every $x\in X$ and every $0<r\le \operatorname{diam}(X)$.  Then, for every
$0<r\le \operatorname{diam}(X)$,
\begin{equation}
\label{eq:Ahlfors_entropy_bound}
    N_r(X,\rho)
    \le
    \frac{2^s}{a}\,r^{-s}.
\end{equation}
\end{lem}
\begin{proof}[{Proof of Lemma~\ref{lem:lower_Ahlfors_implies_entropy_bound}}]
Let $\{x_1,\ldots,x_K\}\subseteq X$ be a maximal $r$-separated set; that is,
$
    \rho(x_i,x_j)>r
$ for every $i,j\in [N]_+$ with $i\neq j$, 
and the set is maximal with respect to inclusion.  By maximality,
$\{x_1,\ldots,x_K\}$ is an $r$-net of $X$.  Hence, $
    N_r(X,\rho)\le K
$; while, on the other hand, the balls
$
    B_\rho(x_1,r/2),\ldots,B_\rho(x_K,r/2)
$
are pairwise disjoint.  Therefore, using that $\mu_X$ is a probability measure
and lower Ahlfors $s$-regular, we find that
\[
    1
\ge
    \sum_{k=1}^K
        \mu_X(B_\rho(x_k,r/2))
\ge
    K\,a\left(\frac{r}{2}\right)^s
.
\]
Thus
$
    K
\le
    \frac{2^s}{a}\,r^{-s}
$.  Since $N_r(X,\rho)\le K$, this proves~\eqref{eq:Ahlfors_entropy_bound}.
\end{proof}

We now deduce our consistency result, which concludes step 2 of our problem.

\begin{cor}[Consistency and rate of the estimator]
\label{cor:robustified_consistency_total_sample}
Suppose that Assumption~\ref{assumption:main_1Lip} holds and that
$\operatorname{diam}(X)>0$.  Let $(Z_n)_{n=1}^{\infty}$ be an i.i.d.\ sequence
with law $\mu=\mu_X\otimes\nu$, where $Z_n=(X_n,\varepsilon_n)$.  
Fix any $\kappa>1$.  For every total sample budget
$\mathsf N\in\mathbb N_+$, set
\[
    \delta_{\operatorname{net},\mathsf N}
    =
    \delta_{\operatorname{noise},\mathsf N}
    \eqdef
    \mathsf N^{-\kappa}.
\]
Choose $A>0$ as specified below and set
\[
    \varepsilon_{\mathsf N}
    \eqdef
    A\mathsf N^{-1/(s+2)},
    \qquad
    \beta_{\mathsf N}
    \eqdef
    \frac{\varepsilon_{\mathsf N}}{2},
    \qquad
    q_{\mathsf N}
    \eqdef
    \left\lceil
        \varepsilon_{\mathsf N}^{-2}
    \right\rceil.
\]
For all sufficiently large $\mathsf N$, choose
$N_{0,\mathsf N},N_{1,\mathsf N}\in\mathbb N_+$ satisfying
$
    N_{0,\mathsf N}+N_{1,\mathsf N}\le\mathsf N
$
and the sample-size conditions below:
and choose integers $N_{0,\mathsf{N}},N_{1,\mathsf{N}}\in\mathbb N_+$ satisfying
$N_{0,\mathsf{N}}+N_{1,\mathsf{N}}\le \mathsf{N}$ and
\begin{align}
\label{eq:N0_N_choice}
    N_{0,\mathsf{N}}
  &   \ge
    \frac{1}{a(\varepsilon_{\mathsf{N}}/8)^s}
    \log\left(
        \frac{
            2N_{\varepsilon_{\mathsf{N}}/8}(X,\rho)
        }{
            \delta_{\operatorname{net},\mathsf{N}}
        }
    \right)
\quad \mbox{ and }\quad
    \\
    N_{1,\mathsf{N}}
  &   \ge
    \frac{1}{a(\varepsilon_{\mathsf{N}}/2)^s}
    \max\left\{
        2q_{\mathsf{N}},
        8\log\left(
            \frac{
                2N_{0,\mathsf{N}}
            }{
                \delta_{\operatorname{net},\mathsf{N}}
            }
        \right)
    \right\}
.
\end{align}
Then, for every fixed
$f\in\operatorname{Lip}((X,\rho);1)$, there exists a constant
$K\in(0,\infty)$, depending only on
$a,s,C,A,\kappa$, such that, for all sufficiently large $\mathsf N$,
\begin{equation}
\label{eq:robustified_rate_OP}
    \mathbb{P}\left(
        \|\hat{f}_{\mathsf N}-f\|_\infty
        >
        K
        \mathsf N^{-1/(s+2)}
        \sqrt{\log\mathsf N}
    \right)
    \le
    2\mathsf N^{-\kappa}.
\end{equation}
\end{cor}
\begin{proof}[{Proof of Corollary~\ref{cor:robustified_consistency_total_sample}}]
Set $\hat{f}_{\mathsf{N}}:X\to\mathbb R$ be the estimator constructed in
Theorem~\ref{thrm:random_apprx_noisy_robustified} with
$
    \varepsilon=\varepsilon_{\mathsf{N}}
$ and $M=q_{\mathsf{N}}
$, $N_0=N_{0,\mathsf{N}}$, $N_1=N_{1,\mathsf{N}}$.  
We divide the proof into four steps.
By Lemma~\ref{lem:lower_Ahlfors_implies_entropy_bound}, for every
$0<r\le\operatorname{diam}(X)$,
$
    N_r(X,\rho)
\le
    \frac{2^s}{a}\,r^{-s}
$.
Since $\varepsilon_{\mathsf{N}}\downarrow0$, for all sufficiently large
$\mathsf{N}$ we have
$
    0<\varepsilon_{\mathsf{N}}<1
$ and $\frac{\varepsilon_{\mathsf{N}}}{2}
    \le
    \operatorname{diam}(X)$.  
Thus Theorem~\ref{thrm:random_apprx_noisy_robustified} is applicable once
\eqref{eq:N0_N_choice} holds.
\hfill\\
We next verify that such choices can be made under the total budget
$N_{0,\mathsf{N}}+N_{1,\mathsf{N}}\le\mathsf{N}$, provided that $A>0$ is chosen
large enough.  By the entropy bound,
$
    N_{\varepsilon_{\mathsf{N}}/8}(X,\rho)
\le
    \frac{2^s}{a}
    \left(
        \frac{\varepsilon_{\mathsf{N}}}{8}
    \right)^{-s}
=
    \frac{16^s}{a}
    \varepsilon_{\mathsf{N}}^{-s}
$. 
Thus, the right-hand side of the first inequality in~\eqref{eq:N0_N_choice} is bounded, up to a
constant depending only on $a,s,A,\beta$, by
$
    \varepsilon_{\mathsf{N}}^{-s}
    \log\big(
        \varepsilon_{\mathsf{N}}^{-s}
        \mathsf{N}^\beta
    \big)
$.
Using $\varepsilon_{\mathsf{N}}=A\mathsf{N}^{-1/(s+2)}$, we obtain
\begin{equation}
\label{eq:N0_growth_bound}
    N_{0,\mathsf{N}}
    \lesssim
    \mathsf{N}^{s/(s+2)}
    \log \mathsf{N}.
\end{equation}
In particular, $
    N_{0,\mathsf{N}}=o(\mathsf{N})
$.
\hfill\\
Now consider the lower bound for $N_{1,\mathsf{N}}$.  Since
$q_{\mathsf{N}}=\lceil\varepsilon_{\mathsf{N}}^{-2}\rceil$, for all sufficiently
large $\mathsf{N}$,
$
    q_{\mathsf{N}}
\le
    2\varepsilon_{\mathsf{N}}^{-2}
$.  
Hence
$
    2q_{\mathsf{N}}
\le
    4\varepsilon_{\mathsf{N}}^{-2}
$.  
The part of~\eqref{eq:N0_N_choice} coming from the term $2q_{\mathsf{N}}$ is
therefore bounded by
\[
    \tfrac{1}{a(\varepsilon_{\mathsf{N}}/2)^s}
    4\varepsilon_{\mathsf{N}}^{-2}
    =
    \tfrac{2^{s+2}}{a}
    \varepsilon_{\mathsf{N}}^{-(s+2)}
    =
    \frac{2^{s+2}}{aA^{s+2}}\,
    \mathsf{N}.
\]
The remaining logarithmic part in~\eqref{eq:N0_N_choice} is, by
\eqref{eq:N0_growth_bound},
\[
\begin{aligned}
    \tfrac{1}{a(\varepsilon_{\mathsf{N}}/2)^s}
    8\log\big(
        \frac{
            2N_{0,\mathsf{N}}
        }{
            \delta_{\operatorname{net},\mathsf{N}}
        }
    \big)
\lesssim
    \varepsilon_{\mathsf{N}}^{-s}
    \log\left(
        N_{0,\mathsf{N}}\mathsf{N}^\beta
    \right)
\lesssim
        \mathsf{N}^{s/(s+2)}
        \log \mathsf{N}
    \in 
        o(\mathsf{N})
.
\end{aligned}
\]
Consequently,
$
    N_{1,\mathsf{N}}
    \le
    \big(
        \frac{2^{s+2}}{aA^{s+2}}
        +o(1)
    \big)
    \mathsf{N}.
$.
Choosing $A>0$ sufficiently large ensures that
$
    \frac{2^{s+2}}{aA^{s+2}}<1
$.
Since $N_{0,\mathsf{N}}=o(\mathsf{N})$, it follows that, for all sufficiently
large $\mathsf{N}$, the integers $N_{0,\mathsf{N}}$ and $N_{1,\mathsf{N}}$ may be
chosen so that
$    N_{0,\mathsf{N}}+N_{1,\mathsf{N}}
\le
    \mathsf{N}
$.  
It remains to obtain a high-probability bound; first we apply Theorem~\ref{thrm:random_apprx_noisy_robustified} with
$
    \varepsilon=\varepsilon_{\mathsf{N}}
$, $\delta_{\operatorname{net}}
    =
    \delta_{\operatorname{net},\mathsf{N}}
$, $\delta_{\operatorname{noise}}
    =
    \delta_{\operatorname{noise},\mathsf{N}}
$, 
and with the sample sizes $N_{0,\mathsf{N}},N_{1,\mathsf{N}}$ chosen above.  Then
there is an event $\mathcal G_{\mathsf{N}}$ satisfying
$
    \mathbb P(\mathcal G_{\mathsf{N}})
\ge
    1
    -
    \delta_{\operatorname{net},\mathsf{N}}
    -
    \delta_{\operatorname{noise},\mathsf{N}}
=
    1-2\mathsf{N}^{-\beta}
$
such that, on $\mathcal G_{\mathsf{N}}$,
\begin{equation}
\label{eq:pre_rate_bound}
    \|\hat{f}_{\mathsf{N}}-f\|_\infty
    \le
    \varepsilon_{\mathsf{N}}
    \left(
        1
        +
        \sqrt{
            C
            \log\left(
                \frac{
                    2N_{0,\mathsf{N}}
                }{
                    \delta_{\operatorname{noise},\mathsf{N}}
                }
            \right)
        }
    \right).
\end{equation}
Using~\eqref{eq:N0_growth_bound} and
$\delta_{\operatorname{noise},\mathsf{N}}=\mathsf{N}^{-\beta}$, we have
\[
\begin{aligned}
    \log\left(
        \frac{
            2N_{0,\mathsf{N}}
        }{
            \delta_{\operatorname{noise},\mathsf{N}}
        }
    \right)
=
    \log\left(
        2N_{0,\mathsf{N}}\mathsf{N}^\beta
    \right)
\lesssim
    \log\left(
        \mathsf{N}^{\beta+s/(s+2)}
        \log\mathsf{N}
    \right)
\lesssim
    \log\mathsf{N}.
\end{aligned}
\]
Substituting this into~\eqref{eq:pre_rate_bound} and using
$\varepsilon_{\mathsf{N}}=A\mathsf{N}^{-1/(s+2)}$, we obtain a constant
$K<\infty$, independent of $\mathsf{N}$, such that $
    \|\hat{f}_{\mathsf{N}}-f\|_\infty
\le
    K\,\mathsf{N}^{-1/(s+2)}
    \sqrt{\log\mathsf{N}}
$
on $\mathcal G_{\mathsf{N}}$, for all sufficiently large $\mathsf{N}$.
We finish by deducing convergence in probability.  
Thus, for all sufficiently large $\mathsf{N}$,
$
    \mathbb P\left(
        \|\hat{f}_{\mathsf{N}}-f\|_\infty
        >
        K\,\mathsf{N}^{-1/(s+2)}
        \sqrt{\log\mathsf{N}}
    \right)
\le
    2\mathsf{N}^{-\beta}
$.
This proves
$
    \|\hat{f}_{\mathsf{N}}-f\|_\infty
\lesssim_{\mathbb P}
    \\
        \mathsf{N}^{-1/(s+2)}
        \sqrt{\log \mathsf{N}}
$.  
Since
$
 \lim\limits_{N\uparrow \infty}\,   \mathsf{N}^{-1/(s+2)}
    \sqrt{\log \mathsf{N}}
=0$ we also obtain
$\lim\limits_{N \uparrow \infty}\,
    \|\hat{f}_{\mathsf{N}}-f\|_\infty
=0$ in probability.
\end{proof}

\subsection{Step 3: Graph Realization}
\label{s:graph_realization}
It only remains to show that Algorithm~\ref{alg:LocalAveragingRobustification} is being computed by the formula in~\eqref{eq:formula__realizable}, with aptly-chosen bandwidth parameter.
\begin{lem}[Selected-Graph Formula Identity]
\label{lem:graph_averaged_labels_equal_virtual_McShane_Whitney}
Let $(X,\rho)$ be a metric space, let
$N_0,N_1\in\mathbb{N}_+$, and set $N\eqdef N_0+N_1$.  Let $
    \mathbb{D}_N
\eqdef
    ((X_n,Y_n))_{n=1}^{N}
    \in
    (X\times\mathbb{R})^N$, fix $\beta>0$, $q\in[N_1]_+$, and $L>0$.  For each
$i\in[N_1]_+$ and $k\in[N_0]_+$, define
\[
\begin{aligned}
    H_{ik}^{(\beta)}
    & \eqdef
    \mathbbm{1}_{\{
        \rho(X_{N_0+i},X_k)\le\beta
    \}},
\,\,
    S_k^{(\beta)}
    \eqdef
    \sum_{i=1}^{N_1}H_{ik}^{(\beta)}
,\,\,
    R_{ik}^{(\beta)}
    \eqdef
    \sum_{\ell=1}^{i}H_{\ell k}^{(\beta)},
\\
    A_{ik}^{(\beta,q)}
    & \eqdef
    \mathbbm{1}_{\{S_k^{(\beta)}\ge q\}}
    H_{ik}^{(\beta)}
    \mathbbm{1}_{\{R_{ik}^{(\beta)}\le q\}}
,\,\,
    G^{(\beta,q)}
    \eqdef
    \frac{1}{q}A^{(\beta,q)}
\\a^{\operatorname{av}}
    & \eqdef
    (Y_{N_0+1},\ldots,Y_{N_0+N_1})
,\,\,
\mbox{ and }
    V^{(\beta,q)}
    \eqdef
    a^{\operatorname{av}}G^{(\beta,q)}.
\end{aligned}
\]
Let
$
    ((X_k,\overline Y_k))_{k=1}^{N_0}
$
be the output of
Algorithm~\ref{alg:LocalAveragingRobustification} with inputs
$(\mathbb{D}_N,\beta,q)$.  Then
\begin{equation}
\label{eq:selected_graph_average_identity}
    \overline Y_k
    =
    V_k^{(\beta,q)}
    =
    \left(
        a^{\operatorname{av}}
        G^{(\beta,q)}
    \right)_k
\end{equation}
for every $k\in[N_0]_+$.  Consequently, if
$
    \theta^{(\beta,q,L)}
\eqdef
    \left(
        \left(
            V_k^{(\beta,q)},X_k,L
        \right)
    \right)_{k=1}^{N_0}
$, then
\begin{equation}
\label{eq:selected_graph_formula_identity}
    \hat{f}_{\mathbb{D}_N}^{(\beta,q,L)}(x)
    =
    \hat{f}_{\theta^{(\beta,q,L)}}(x)
\end{equation}
for each $x\in \mathcal{X}$.
\end{lem}
\begin{proof}[{Proof of Lemma~\ref{lem:graph_averaged_labels_equal_virtual_McShane_Whitney}}]
Fix $k\in[N_0]_+$.  Suppose first that
$
    S_k^{(\beta)}<q
$.
Then the factor
$
    \mathbbm{1}_{\{S_k^{(\beta)}\ge q\}}
$
vanishes, so
$
    A_{ik}^{(\beta,q)}=0
$
for every $i\in[N_1]_+$.  Hence
$
    \big(
        a^{\operatorname{av}}
        G^{(\beta,q)}
    \big)_k
    =
    0$.  
Algorithm~\ref{alg:LocalAveragingRobustification} also sets
$
    \overline Y_k=0
$
in this case.  Thus
\eqref{eq:selected_graph_average_identity} holds when
$
    S_k^{(\beta)}<q
$.
Suppose now that
$
    S_k^{(\beta)}\ge q
$,
and let
$
    i_{k,1}<\cdots<i_{k,S_k^{(\beta)}}
$
be the increasing enumeration of the indices $i\in[N_1]_+$ for which
$
    H_{ik}^{(\beta)}=1
$.
For such an index $i=i_{k,\ell}$, one has
$
    R_{ik}^{(\beta)}=\ell$.  
It follows from the definition of
$
    A_{ik}^{(\beta,q)}
$
that
$
    A_{ik}^{(\beta,q)}
    =
    1
$  
precisely when
$
    i\in\{i_{k,1},\ldots,i_{k,q}\}
$,
and that it vanishes for every other $i$.  Therefore
\[
\begin{aligned}
    \big(
        a^{\operatorname{av}}
        G^{(\beta,q)}
    \big)_k
=
    \frac{1}{q}
    \sum_{i=1}^{N_1}
        Y_{N_0+i}
=
    \frac{1}{q}
    \sum_{\ell=1}^{q}
=
    \overline Y_k,
\end{aligned}
\]
where the last equality is the update rule in
Algorithm~\ref{alg:LocalAveragingRobustification}.  This proves
\eqref{eq:selected_graph_average_identity} in both cases.  Substituting
$
    V_k^{(\beta,q)}=\overline Y_k
$
and
$
    \xi_k=X_k
$
into~\eqref{eq:formula__realizable} gives
\[
\begin{aligned}
    \hat{f}_{\mathbb{D}_N}^{(\beta,q,L)}(x)
    &=
    \frac{1}{2}
    \left[
        \min_{k\in[N_0]_+}
        \left\{
            V_k^{(\beta,q)}
            +
            L\rho(x,X_k)
        \right\}
        +
        \max_{k\in[N_0]_+}
        \left\{
            V_k^{(\beta,q)}
            -
            L\rho(x,X_k)
        \right\}
    \right]
    \\
    &=
    \hat{f}_{\theta^{(\beta,q,L)}}(x)
\end{aligned}
\]
for every $x\in X$.  This proves
\eqref{eq:selected_graph_formula_identity}.
\end{proof}

\begin{proof}[{Proof of Theorem~\ref{thrm:nonparametric_noisy_case}}]
Set
$
    \delta_{\operatorname{net}}
    =
    \delta_{\operatorname{noise}}
    \eqdef
    \frac{\delta}{2}$; then, the first inequality in~\eqref{eq:main_robustified_sample_condition} becomes
\[
    N_0
    \ge
    \frac{1}{a(\varepsilon/8)^s}
    \log\left(
        \frac{
            2\mathcal N(X,\rho,\varepsilon/8)
        }{
            \delta_{\operatorname{net}}
        }
    \right),
\]
and the second becomes
\[
    N_1
    \ge
    \frac{1}{a(\varepsilon/2)^s}
    \max\left\{
        2q_\varepsilon,
        8\log\left(
            \frac{
                2N_0
            }{
                \delta_{\operatorname{net}}
            }
        \right)
    \right\}.
\]
Thus all hypotheses of
Theorem~\ref{thrm:random_apprx_noisy_robustified} hold with
$
    \beta=\frac{\varepsilon}{2}
$ and $
    q_\varepsilon
    =
    \left\lceil\varepsilon^{-2}\right\rceil$.  
That theorem shows that the midpoint estimator
$\hat{f}$ in~\eqref{eq:technical_robustified_estimator} is
$1$-Lipschitz for every sample realization and satisfies
\[
    \mathbb{P}\left(
        \|\hat{f}-f\|_\infty
        \le
        \varepsilon
        \left[
            1
            +
            \sqrt{
                C
                \log\left(
                    \frac{4N_0}{\delta}
                \right)
            }
        \right]
    \right)
    \ge
    1-\delta.
\]
By
Lemma~\ref{lem:graph_averaged_labels_equal_virtual_McShane_Whitney},
this estimator agrees pointwise with
$
    \hat{f}_{\mathbb{D}_N}^{(\beta,q_\varepsilon,1)}
$
defined in~\eqref{eq:formula__realizable}.  Therefore
\eqref{eq:main_robustified_uniform_bound} follows, together with the claimed
$1$-Lipschitz regularity.
\end{proof}

\subsection{Proofs of Optimality}
\label{s:Proofs__ss:Optim}
\subsubsection{Proof of Combinatorial Optimality}
\label{s:FatShattering__ss:LB}

Our lower bound generalizes~\cite[Theorem~4.2]{zhang2024deep} to general metric measure spaces. Again, Ahlfors regularity plays the key role. In particular, the result of~\cite[Theorem~4.2]{zhang2024deep} is recovered by equipping $[0,1]^d$ with the $\alpha$-snowflaked metric $\|x-y\|_\infty^\alpha$, for $0<\alpha\le1$, and Lebesgue measure, which is $(d/\alpha)$-Ahlfors regular with respect to this metric.

\begin{proposition}[Lower Bound on Fat-Shattering Dimension via Ahlfors Regularity]
\label{prop:uniform_approximation_fat_shattering}
Let $(X,\rho,\mathfrak{m})$ be a compact $d$-Ahlfors regular metric measure space, for some $d>0$, let $B,L>0$, and fix $\mathcal{F}\subseteq C(X,\mathbb{R})$. There exist constants $\varepsilon_0,c>0$ such that, if there is a fixed tolerance $0<\varepsilon<\min\{B/2,L\varepsilon_0\}$ such that, for every $f\in\operatorname{Lip}((X,\rho),[-B,B];L)$, there is some $\hat{f}\in\mathcal{F}$ satisfying
\begin{equation}
\label{eq:eps_approx}
    \|f-\hat{f}\|_\infty
\le
    \varepsilon,
\end{equation}
then, for every $\varepsilon\le\gamma<\min\{B/2,L\varepsilon_0\}$,
\begin{equation}
\label{eq:fat_shattering_lower_bound}
    \operatorname{fat}_{\gamma}(\mathcal{F})
\ge
    c
    \left(
        \frac{L}{\gamma}
    \right)^d.
\end{equation}
\end{proposition}

\begin{proof}[{Proof of Proposition~\ref{prop:uniform_approximation_fat_shattering}}]
Fix $\varepsilon\le\gamma<\min\{B/2,L\varepsilon_0\}$ and set $\eta\eqdef\gamma/L$. By the lower-bound construction in the proof of Lemma~\ref{lem:Nondegeneracy_metric_entropy}, together with~\eqref{eq:packing_lower_bound_Ahlfors}, there exist points $\{x_1,\ldots,x_m\}\subseteq X$, with
\[
    m
    \ge
    c\,\eta^{-d}
    =
    c
    \left(
        \frac{L}{\gamma}
    \right)^d,
\]
such that, for every $\sigma\in\{-1,1\}^m$, there exists $f_\sigma\in\operatorname{Lip}((X,\rho),[-B,B];L)$ satisfying
$
    f_\sigma(x_j)=2\gamma\sigma_j
$
for every $j\in[m]_+$.

By~\eqref{eq:eps_approx}, for each $\sigma\in\{-1,1\}^m$ there exists $\hat{f}_\sigma\in\mathcal F$ such that $\|f_\sigma-\hat{f}_\sigma\|_\infty\le\varepsilon$. Therefore, for every $j\in[m]_+$,
\[
    \sigma_j\hat{f}_\sigma(x_j)
    \ge
    2\gamma-\varepsilon
    \ge
    \gamma.
\]
Thus, taking $t_j\eqdef0$ for every $j\in[m]_+$, the set $\{x_1,\ldots,x_m\}$ is $\gamma$-fat shattered by $\mathcal F$. Consequently,
\[
    \operatorname{fat}_{\gamma}(\mathcal F)
    \ge
    m
    \ge
    c
    \left(
        \frac{L}{\gamma}
    \right)^d,
\]
which proves~\eqref{eq:fat_shattering_lower_bound}.
\end{proof}

We may now complete the proof of Theorem~\ref{thrm:fatshatteringotimality} by combining Proposition~\ref{prop:uniform_approximation_fat_shattering} with the Lipschitz-class fat-shattering bound of~\cite[Theorem~2]{gottlieb2014efficient}, the Ahlfors-regular packing estimates established in Lemma~\ref{lem:Nondegeneracy_metric_entropy}, and our oracle approximation guarantee.

\begin{proof}[{Proof of Theorem~\ref{thrm:fatshatteringotimality}}]
Suppose that~\eqref{eq:eps_approx__minimaxformulation} holds. Proposition~\ref{prop:uniform_approximation_fat_shattering}, applied with $\gamma=\varepsilon$, gives
$
    \operatorname{fat}_{\varepsilon}(\mathcal F)
    \ge
    c
    \big(
        \tfrac{L}{\varepsilon}
    \big)^d
$.  On the other hand, monotonicity of the fat-shattering dimension, the inclusion
$
    \mathcal F\subseteq\operatorname{Lip}((X,\rho),\mathbb R;L),
$
and~\cite[Theorem~2]{gottlieb2014efficient} give
$
    \operatorname{fat}_{\varepsilon}(\mathcal F)
\le
    \mathcal{N}
    \big(
        X,\rho,\frac{2\varepsilon}{L}
    \big)
$ and  the Ahlfors-regular packing estimate established in Lemma~\ref{lem:Nondegeneracy_metric_entropy} therefore yields $\mathcal{N}
    \big(
        X,\rho,\frac{2\varepsilon}{L}
    \big)\le C
    \left(
        \frac{L}{\varepsilon}
    \right)^d$; whence
\[
    \operatorname{fat}_{\varepsilon}(\mathcal F)
    \le
    C
    \left(
        \frac{L}{\varepsilon}
    \right)^d.
\]
By construction, we have that $
    \mathcal F_{M:L}
    \subseteq
    \operatorname{Lip}((X,\rho),\mathbb{R};L)
$.  
Finally, the estimate~\eqref{eq:MW_error_bound_Ahlfors_metric}, applied after rescaling from the unit-Lipschitz to the $L$-Lipschitz setting, gives
\[
    \sup_{f\in\operatorname{Lip}((X,\rho),[-B,B];L)}
    \inf_{\hat{f}\in\mathcal F_{M:L}}
    \|f-\hat{f}\|_\infty
    \le
    C_d L M^{-1/d}.
\]
Hence~\eqref{eq:eps_approx__minimaxformulation} holds for $\mathcal F_{M:L}$ whenever
$
    M\ge C_d'(L/\varepsilon)^d
$
for a sufficiently large constant $C_d'>0$ depending only on $(X,\rho,\mathfrak m)$. This proves the final assertion.
\end{proof}

\subsection{Proof of Optimal Parameter-Dependence}
\label{s:Optimality__ss:Result}
We now prove Theorem~\ref{thm:stable_near_minimax_optimality_oracle_regime} in the realizable regime, if $\{X_n\}_{n=1}^N$ can be selected by the user as in neural network approximation theory; then our reconstruction formula is minimax-optimal (up to a poly-log factor if we sample the data and do not have oracle access)
\begin{equation}
% \label{eq:formula}
\begin{aligned}
    \hat{f}_{\theta}(x)
     & = 
        \tfrac{1}{2}
        \Big(
            \min_{m \in [M]_+} (a_m+u_m)
            +  
            \max_{m \in [M]_+} (a_m-u_m)
        \Big)
\\
    x^{(0)} & = 
        \big(c_m\,\rho(x,b_m)\big)_{m=1}^M
.
\end{aligned}
\end{equation}
We now study the (local-)Lipschitz regularity of $\hat{f}_{\theta}$ as a function of $\theta$; i.e.\ of the map $[\mathbb{R}\times X\times \mathbb{R}]^M \ni \theta \mapsto \hat{f}_{\theta}\in C(X)$.  

\subsubsection{Step 1: Optimal Approximation with Oracle-Access to $f$}
\label{s:Optimality__ss:ProofStep1}

Let us study the approximation of $1$-Lipschitz functions into $[-1,1]$; which when we have oracle-access to the function values $f(x)$ any any queried point $x\in X$ yields.
\begin{lem}[Candidate optimal approximator with oracle access to \(f\)]
\label{lem:oracle_access__approximation}
Let $(X,\rho)$ be a totally bounded metric space with
$
    \operatorname{diam}(X)\le 1
$
and doubling dimension at most $d\in\mathbb N_+$.  Let
$
    \{x_n\}_{n=1}^N
$
be an $\tfrac{1}{N^{1/d}}$-net of minimal cardinality in $(X,\rho)$.  For any
$1$-Lipschitz map $f:X\to[-1,1]$, define
\[
    \theta_f
    \eqdef
    \big((f(x_n),x_n,1)\big)_{n=1}^N .
\]
Then the reconstruction formula $\hat{f}_{\theta_f}$ from~\eqref{eq:formula}
satisfies
\[
    \sup_{x\in X}
        |f(x)-\hat{f}_{\theta_f}(x)|
    \lesssim
        N^{-1/d}.
\]
\end{lem}
\begin{proof}[{Proof of Lemma~\ref{lem:oracle_access__approximation}}]
Let $0<\varepsilon<1$ be arbitrary. 
Since
$\operatorname{diam}(X)\le1$ and the doubling dimension of $X$ is at most $d$,
the standard iterated doubling estimate gives
$
    N_\varepsilon(X,\rho)
\le
    \left(\frac{2}{\varepsilon}\right)^d
$; where $N_\varepsilon(X,\rho)$ denotes the minimal cardinality of an
$\varepsilon$-net of $X$.  
Since $\{x_n\}_{n=1}^N$ is an $\varepsilon$-net of
minimal cardinality, $N=N_\varepsilon(X,\rho)$, and therefore
$
    N
\le
    \left(\frac{2}{\varepsilon}\right)^d 
$, or
$
    \varepsilon
\le
    2N^{-1/d}
$.

For $x\in X$, define the McShane and Whitney envelopes generated by the oracle
values $\{f(x_n)\}_{n=1}^N$ by
\[
    U_f(x)
    \eqdef
    \min_{n\in[N]_+}
        \big(f(x_n)+\rho(x,x_n)\big),
\, \mbox{ and }\,
    L_f(x)
    \eqdef
    \max_{n\in[N]_+}
        \big(f(x_n)-\rho(x,x_n)\big).
\]
With the choice
$
    \theta_f=((f(x_n),x_n,1))_{n=1}^N
$,
the reconstruction formula~\eqref{eq:formula} is precisely
$
    \hat{f}_{\theta_f}(x)
    =
    \frac{1}{2}\big(U_f(x)+L_f(x)\big)
$ for each $x\in X$.
\hfill\\
\noindent
We first show that $f$ is squeezed between these two envelopes.  Since $f$ is
$1$-Lipschitz, for every $x\in X$ and every $n\in[N]_+$,
$
    f(x)
    \le
    f(x_n)+\rho(x,x_n)
$ and $f(x)
    \ge
    f(x_n)-\rho(x,x_n)
$.  
Taking the minimum in the first inequality and the maximum in the second gives
\begin{equation}
\label{eq:envelope_us_bby}
    L_f(x)
    \le
    f(x)
    \le
    U_f(x)
\end{equation}
for each $x\in X$.
\hfill\\
\noindent
Next, fix $x\in X$.  Since $\{x_n\}_{n=1}^N$ is an $\varepsilon$-net, there
exists $n_x\in[N]_+$ such that
$
    \rho(x,x_{n_x})\le \varepsilon 
$.
By the definitions of $U_f$ and $L_f$,
$
    U_f(x)
\le
    f(x_{n_x})+\rho(x,x_{n_x})
$ and $
    L_f(x)
\ge
    f(x_{n_x})-\rho(x,x_{n_x})
$.  
Consequently,
$
    0
\le
    U_f(x)-L_f(x)
\le
    2\rho(x,x_{n_x})
\le
    2\varepsilon
$.
Since
$
    f(x)\in[L_f(x),U_f(x)]
$
and
$
    \hat{f}_{\theta_f}(x)
$
is the midpoint of this interval, we obtain
\[
    |f(x)-\hat{f}_{\theta_f}(x)|
\le
    \frac{1}{2}\big(U_f(x)-L_f(x)\big)
\le
    \varepsilon
.
\]
Taking the supremum over $x\in X$ yields
\[
    \sup_{x\in X}
        |f(x)-\hat{f}_{\theta_f}(x)|
    \le
        \varepsilon
    \le
        2N^{-1/d}
    \lesssim
        N^{-1/d}.
\]
\end{proof}

\subsubsection{Step 2: Stability of an Optimal Function to Parameter Map}
\label{s:Optimality__ss:ProofStep2}
\begin{lem}[Local Lipschitz stability in the parameters]
\label{lem:stability}
Let $(X,\rho)$ be a metric space with diameter $\operatorname{diam}(X)<\infty$.  
Let $M\in \mathbb{N}_+$ and $R\ge 0$; the ``encoder'' map
$
    \mathfrak{E}:
    \,
    \big([-R,R]\times X\times[-R,R]\big)^M
        \ni \theta
        \mapsto
    \hat{f}_\theta\in C_b(X)
$ is $\max\{1,R,\operatorname{diam}(X)\}$-Lipschitz.  
\hfill\\
In particular, $\mathfrak{E}|_{([-1,1]\times X\times \{1\})^M}$ is $1$-Lipschitz.
\end{lem}
Observe that, in the realizable \(1\)-Lipschitz case we take
\(a_m=f(b_m)\in[-1,1]\) and \(c_m=1\), the parameter vector belongs to
$
    \big([-1,1]\times X\times[-1,1]\big)^M$.  
Thus, Lemma~\ref{lem:stability} applies with \(R=1\), and the realization map is
\(\max\{1,\operatorname{diam}(X)\}\)-Lipschitz on this parameter set.  

\begin{proof}[{Proof of Lemma~\ref{lem:stability}}]
We will show that, for every
$
    \theta
    =
    \big(a,(b_m)_{m=1}^M,c\big)
,
\tilde{\theta}
    =
    \big(\tilde a,(\tilde b_m)_{m=1}^M,\tilde c\big)
\in [\mathbb{R}\times X\times \mathbb{R}]^M
$ we have
$
    \sup_{x\in X}\,
        \big|
                \hat{f}_{\theta}(x)
            -
                \hat{f}_{\tilde{\theta}}(x)
        \big|
    \le 
        \max\{1,R,D_X\}
        \,
        \sum_{m=1}^M
        \big(
                |a_m-\tilde a_m|
            +
                \rho(b_m,\tilde b_m)
            +
                |c_m-\tilde c_m|
        \big)
$.
Now, for $\theta=(a,(b_m)_{m=1}^M,c)$, write
$
    u_m^\theta(x)
    \eqdef
    c_m\,\rho(x,b_m)
$ and $m\in[M]_+$ and $x\in X$.  
Thus
\[
        \hat{f}_\theta(x)
    =
        \frac{1}{2}
        \left(
            \min_{m\in[M]_+}\big(a_m+u_m^\theta(x)\big)
            +
            \max_{m\in[M]_+}\big(a_m-u_m^\theta(x)\big)
        \right)
.
\]
Fix $x\in X$.  We first recall the elementary inequalities
$
    \big|
        \min_{m\in[M]_+} \alpha_m
        -
        \min_{m\in[M]_+} \beta_m
    \big|
    \le
    \max_{m\in[M]_+}|\alpha_m-\beta_m|
$, 
and
$
    \big|
        \max_{m\in[M]_+} \alpha_m
        -
        \max_{m\in[M]_+} \beta_m
    \big|
\le
    \max_{m\in[M]_+}|\alpha_m-\beta_m|
$. 
Applying these inequalities with $
    \alpha_m=a_m+u_m^\theta(x)
    $ and $
    \beta_m=\tilde a_m+u_m^{\tilde\theta}(x)$
and then with
$
    \alpha_m=a_m-u_m^\theta(x)
$ and $\beta_m=\tilde a_m-u_m^{\tilde\theta}(x)$, 
gives
\begin{align}
\label{eq:stability_minmax_step}
    \big|
        \hat{f}_\theta(x)-\hat{f}_{\tilde\theta}(x)
    \big|
    &\le
    \frac{1}{2}
    \max_{m\in[M]_+}
    \big|
        (a_m+u_m^\theta(x))
        -
        (\tilde a_m+u_m^{\tilde\theta}(x))
    \big|
\nonumber\\
    &\qquad
    +
    \frac{1}{2}
    \max_{m\in[M]_+}
    \big|
        (a_m-u_m^\theta(x))
        -
        (\tilde a_m-u_m^{\tilde\theta}(x))
    \big|.
\end{align}
For each $m\in[M]_+$, the reverse triangle inequality for the metric $\rho$
implies
$
    \big|
        \rho(x,b_m)-\rho(x,\tilde b_m)
    \big|
\le
    \rho(b_m,\tilde b_m)
$.
Hence, since $|c_m|\le R$ and $\rho(x,\tilde b_m)\le D_X$, we have
\begin{align}
\label{eq:cone_parameter_stability}
    \big|
        u_m^\theta(x)-u_m^{\tilde\theta}(x)
    \big|
    &=
    \big|
        c_m\rho(x,b_m)-\tilde c_m\rho(x,\tilde b_m)
    \big|
\nonumber\\
    &\le
    |c_m|
    \big|
        \rho(x,b_m)-\rho(x,\tilde b_m)
    \big|
    +
    |c_m-\tilde c_m|\,\rho(x,\tilde b_m)
\nonumber\\
    &\le
    R\,\rho(b_m,\tilde b_m)
    +
    D_X\,|c_m-\tilde c_m|.
\end{align}
Consequently,
$
    \big|
        (a_m\pm u_m^\theta(x))
        -
        (\tilde a_m\pm u_m^{\tilde\theta}(x))
    \big|
\le
        |a_m-\tilde a_m|
        +
        R\,\rho(b_m,\tilde b_m)
        +
        D_X\,|c_m-\tilde c_m|
$.  
Substituting this estimate into~\eqref{eq:stability_minmax_step} yields
\begin{equation*}
    \big|
        \hat{f}_\theta(x)-\hat{f}_{\tilde\theta}(x)
    \big|
\le
    \max_{m\in[M]_+}
    \left(
        |a_m-\tilde a_m|
        +
        R\,\rho(b_m,\tilde b_m)
        +
        D_X\,|c_m-\tilde c_m|
    \right)
    .
\end{equation*}  
Finally,
$
    |a_m-\tilde a_m|
        +
        R\,\rho(b_m,\tilde b_m)
        +
        D_X\,|c_m-\tilde c_m|
\le
    \max\{1,R,D_X\}
    \big(
        |a_m-\tilde a_m|
        +
        \rho(b_m,\tilde b_m)
        +
        |c_m-\tilde c_m|
    \big)
$.
Taking the maximum over $m$, then bounding the maximum by the sum over
$m\in[M]_+$, gives
\[
        \big|
            \hat{f}_\theta(x)-\hat{f}_{\tilde\theta}(x)
        \big|
    \le
        \max\{1,R,D_X\}
        \sum_{m=1}^M
        \big(
                |a_m-\tilde a_m|
            +
                \rho(b_m,\tilde b_m)
            +
                |c_m-\tilde c_m|
        \big)
.
\]
Since the right-hand side is independent of $x\in X$, taking the supremum over
$x\in X$ yields the conclusion.
\end{proof}

\subsubsection{Step 3: Identifying ``Good'' Lower Bounds on the Lipschitz Widths}
\label{s:Optimality__ss:ProofStep3}
Relying on the beautiful result in~\cite[Theorem 3.1 (i)]{petrova2023lipschitz} and~\cite[Theorem XV]{kolmogorov1959varepsilon} we obtain the following lower-bounds if $(X,\rho)=([0,1]^d,\|\cdot\|_2)$.  However, as we are interested in generality, we will have to derive a matching lower-bound to~\cite[Lemma 4.2]{gottlieb2016adaptive} (without their extra polylog factor); which require the underlying metric space $(X,\rho)$ not to be ``too thin'', as their doubling assumption only prevents it from being too thick.
%
%
%

% {\color{blue}{make a picture of the metric hat-functions}}
%
%
%
\begin{lem}[{Metric entropy of $\mathcal{F}_{1:d}$ scales like $\Theta(\varepsilon^{-d})$}]
\label{lem:Nondegeneracy_metric_entropy}
Let $(X,\rho,\mathfrak{m})$ be a compact metric measure space.
Assume that $(X,\rho,\mathfrak{m})$ is $0<d$-Ahlfors regular, and write
$
    \mathcal{F}_1
\eqdef
    \operatorname{Lip}((X,\rho),[-1,1];1)
$.  
Then, there are constants
$
    0<c\le C<\infty
$
and
$
    0<\varepsilon_0<1
$
(depending only on on $(X,\rho,\mathfrak{m})$) such that, for every
$
    0<\varepsilon<\varepsilon_0
$,
\begin{equation}
\label{eq:Lip_class_entropy_asymptotics}
    c\,\varepsilon^{-d}
    \le
    \log_2
    \mathcal N
    \bigl(
        \mathcal{F}_1,
        \|\cdot\|_{\infty},
        \varepsilon
    \bigr)
    \le
    C\,\varepsilon^{-d}.
\end{equation}
\end{lem}

\begin{proof}[{Proof of Lemma~\ref{lem:Nondegeneracy_metric_entropy}}]
Abbreviate $D_X\eqdef\operatorname{diam}(X)\in(0,\infty)$.
We prove the upper and lower bounds separately.
For each $k\in\mathbb N_0$, set
$
    r_k
\eqdef
    2^{-k}D_X
$.  
Since $X$ is compact, we may choose a nested sequence
$
    \mathcal{X}_0\subseteq \mathcal{X}_1\subseteq\cdots\subseteq X
$
such that each $\mathcal{X}_k$ is a maximal $r_k$-separated subset of $X$.
Thus $\mathcal{X}_k$ is also an $r_k$-net of $X$.
\hfill\\
We first record the standard packing estimate.  Since $\mathcal{X}_k$ is
$r_k$-separated, the balls
$
    \big\{
        B\big(x,\frac{r_k}{2}\big)
        :
        x\in\mathcal{X}_k
    \big\}
$
are pairwise disjoint.  By the lower $d$-Ahlfors condition 
\[
    \sum_{x\in\mathcal{X}_k}
        \mathfrak{m}\left(
            B\left(x,\frac{r_k}{2}\right)
        \right)
\ge
    |\mathcal{X}_k|\,c_{\rm A}\left(\frac{r_k}{2}\right)^d
\].  
On the other hand, this disjoint union is contained in $X$, and
$
    \mathfrak{m}(X)\le C_{\rm A}D_X^d
$.
Therefore
\begin{equation}
\label{eq:net_cardinality_upper_bound}
    |\mathcal{X}_k|
    \le
    C\,2^{kd}
\mbox{ for each }
    k\in\mathbb N_0,
\end{equation}
where $C$ depends only on the Ahlfors regularity constants and on $d$.
\hfill\\
Let $K\in\mathbb N$ be the unique integer such that
\begin{equation}
\label{eq:choice_of_K_entropy_proof}
    3r_K
\le
    \varepsilon
<
    6r_K .
\end{equation}
Since $\mathcal{X}_{k-1}$ is an $r_{k-1}$-net then: 
for every $k\ge1$ and every $x\in\mathcal{X}_k$, choose a parent point
$
    \pi_k(x)\in\mathcal{X}_{k-1}
$
such that
\begin{equation}
\label{eq:parent_map_entropy_proof}
    \rho(x,\pi_k(x))
\le
    r_{k-1}
=
    2r_k 
.
\end{equation}
We now construct a finite family of candidate traces on $\mathcal{X}_K$.
For $k\in\{0,\dots,K\}$, let
$
    \Lambda_k
\eqdef
    r_k\mathbb{Z}
$. 
A \textit{multiscale code} by which we mean functions
$
    q_k:\mathcal{X}_k\to\Lambda_k
$ for each $k\in [K]$
satisfying the following admissibility conditions:
$
    q_0(x)\in[-2,2]
$, $
    x\in\mathcal{X}_0
$ 
and, for every $1\le k\le K$ and every $x\in\mathcal{X}_k$,
\begin{equation}
\label{eq:admissible_multiscale_code}
    |q_k(x)-q_{k-1}(\pi_k(x))|
    \le
    5r_k .
\end{equation}
Let $\mathcal{Q}_K$ denote the set of all terminal traces
$
    q_K:\mathcal{X}_K\to\mathbb R
$
arising from such admissible multiscale codes.
\hfill\\
We claim that
\begin{equation}
\label{eq:terminal_trace_cardinality_bound}
    \log_2|\mathcal{Q}_K|
\le
    C\,2^{Kd}.
\end{equation}
Indeed, the number of possible values of $q_0(x)$ is bounded by a constant
depending only on $D_X$.  Moreover, once $q_{k-1}$ has been fixed, the
constraint~\eqref{eq:admissible_multiscale_code} and the lattice spacing
$r_k$ imply that each $q_k(x)$ has at most a universal number of admissible
choices.  Hence, 
$
    |\mathcal{Q}_K|
\le
    C^{|\mathcal{X}_0|}
    \prod_{k=1}^K
        C^{|\mathcal{X}_k|}
$; taking logarithms and combining with~\eqref{eq:net_cardinality_upper_bound} we find that
\[
    \log_2|\mathcal{Q}_K|
\le
    C
    \sum_{k=0}^K
        |\mathcal{X}_k|
\le
    C
    \sum_{k=0}^K
        2^{kd}
\le
    C\,2^{Kd}
.
\]
We next show that the family $\mathcal{Q}_K$ is rich enough to approximate the
trace of every $f\in\mathcal{F}_1$.  Fix $f\in\mathcal{F}_1$.  For every
$k\in\{0,\dots,K\}$ and every $x\in\mathcal{X}_k$, choose
$
    q_k(x)\in\Lambda_k
$
such that
\begin{equation}
\label{eq:quantized_trace_accuracy}
    |q_k(x)-f(x)|
    \le
    r_k .
\end{equation}
Such a choice is possible by definition of $\Lambda_k$.  We verify that the
resulting code is admissible.  For $k\ge1$ and $x\in\mathcal{X}_k$,
using~\eqref{eq:parent_map_entropy_proof},~\eqref{eq:quantized_trace_accuracy},
and the fact that $f$ is $1$-Lipschitz, we get
\[
\begin{aligned}
    |q_k(x)-q_{k-1}(\pi_k(x))|
    &\le
    |q_k(x)-f(x)|
    +
    |f(x)-f(\pi_k(x))|
    +
    |f(\pi_k(x))-q_{k-1}(\pi_k(x))|
\\
    &\le
    r_k
    +
    \rho(x,\pi_k(x))
    +
    r_{k-1}
\\
    &\le
    r_k+2r_k+2r_k
    =
    5r_k .
\end{aligned}
\]
Thus $q_K\in\mathcal{Q}_K$, and for each $z\in\mathcal{X}_K$ we have
\begin{equation}
\label{eq:terminal_trace_accuracy}
    |q_K(z)-f(z)|
    \le
    r_K
.
\end{equation}
For each terminal trace $q\in\mathcal{Q}_K$, define its McShane extension
$
    E_q(x)
    \eqdef
    \inf_{z\in\mathcal{X}_K}
        \bigl(q(z)+\rho(x,z)\bigr)
$ for each $x\in X$.
The function $E_q$ is $1$-Lipschitz on $X$.  We claim that the family
$
    \{E_q:q\in\mathcal{Q}_K\}
$
is an $\varepsilon$-cover of $\mathcal{F}_1$ in $\|\cdot\|_\infty$.
\hfill\\
Fix $f\in\mathcal{F}_1$, and let $q\in\mathcal{Q}_K$ be the terminal trace
constructed above.  By~\eqref{eq:terminal_trace_accuracy}, for every
$z\in\mathcal{X}_K$ and every $x\in X$,
$
    q(z)+\rho(x,z)
\ge
    f(z)-r_K+\rho(x,z)
\ge
    f(x)-r_K
$.
Taking the infimum over $z\in\mathcal{X}_K$ gives
\begin{equation}
\label{eq:McShane_lower_approx_entropy}
    E_q(x)
\ge
    f(x)-r_K 
.
\end{equation}
Conversely, since $\mathcal{X}_K$ is an $r_K$-net, for every $x\in X$ there
exists $z_x\in\mathcal{X}_K$ such that
$
    \rho(x,z_x)\le r_K
$.
Therefore
\[
\begin{aligned}
    E_q(x)
    &\le
    q(z_x)+\rho(x,z_x)
\\
    &\le
    f(z_x)+r_K+\rho(x,z_x)
\\
    &\le
    f(x)+2\rho(x,z_x)+r_K
\\
    &\le
    f(x)+3r_K .
\end{aligned}
\]
Together with~\eqref{eq:McShane_lower_approx_entropy}, this yields
$
    \|f-E_q\|_\infty
\le
    3r_K
\le
    \varepsilon
$,
where the last inequality follows from~\eqref{eq:choice_of_K_entropy_proof}.
Hence
$
    \mathcal N
    \big(
        \mathcal{F}_1,
        \varepsilon,
        \|\cdot\|_\infty
    \big)
\le
    |\mathcal{Q}_K|
$.
Using~\eqref{eq:terminal_trace_cardinality_bound} and
\eqref{eq:choice_of_K_entropy_proof}, we conclude that
\begin{equation}
\label{eq:UB_ghahahh}
    \log_2
    \mathcal N
    \bigl(
        \mathcal{F}_1,
        \|\cdot\|_{\infty},
        \varepsilon
    \bigr)
    \le
    C\,2^{Kd}
    \le
    C\,\varepsilon^{-d}.
\end{equation}
Now, we obtain a matching lower-bound to~\eqref{eq:UB_ghahahh}.
Let $0<\varepsilon<\varepsilon_0$, where $\varepsilon_0>0$ is chosen small enough
so that
$
    8\varepsilon\le D_X
$.
Set
$
    r\eqdef 2\varepsilon
$. 
Choose a maximal $4r$-separated set
$
    \{x_1,\dots,x_M\}\subseteq X
$.  
By maximality, the balls
$
    \{B(x_j,4r)\}_{j=1}^M
$
cover $X$.  Hence, using the upper Ahlfors bound,
$
    \mathfrak{m}(X)
\le
    \sum_{j=1}^M
        \mathfrak{m}(B(x_j,4r))
\le
    M\,C_{\rm A}(4r)^d
$.
Since $\mathfrak{m}(X)>0$, it follows that
\begin{equation}
\label{eq:packing_lower_bound_Ahlfors}
    M
    \ge
    c\,r^{-d}
    \ge
    c\,\varepsilon^{-d}.
\end{equation}
For every $j\in[M]_+$, define the metric tent function
$
    \psi_j(x)
\eqdef
    \bigl(r-\rho(x,x_j)\bigr)_+
$ for each $x\in X$.
Each $\psi_j$ is $1$-Lipschitz, satisfies $0\le\psi_j\le r$, and has support
contained in $B(x_j,r)$.  Since the points $x_1,\dots,x_M$ are $4r$-separated,
these supports are pairwise disjoint.
\hfill\\
For each sign vector
$
    \sigma=(\sigma_1,\dots,\sigma_M)\in\{-1,1\}^M
$,
define
$
    f_\sigma(x)
    \eqdef
    \sum_{j=1}^M
        \sigma_j\psi_j(x)
$ and $x\in X$.
Since the supports of the $\psi_j$'s are pairwise disjoint and
$
    0\le\psi_j\le r<1
$,
we have
$
    f_\sigma(X)\subseteq[-1,1]
$.
\hfill\\
We verify that $f_\sigma$ is $1$-Lipschitz.  Let $x,y\in X$.  If $x$ and $y$
belong to the support of the same tent $\psi_j$, then
$
    |f_\sigma(x)-f_\sigma(y)|
=
    |\psi_j(x)-\psi_j(y)|
\le
    \rho(x,y)
$.
If
$
    x\in\operatorname{supp}(\psi_i)
$
and
$
    y\in\operatorname{supp}(\psi_j)
$
with $i\neq j$, then
\[
\begin{aligned}
    \rho(x,y)
    &\ge
    \rho(x_i,x_j)
    -
    \rho(x,x_i)
    -
    \rho(y,x_j)
\\
    &\ge
    4r
    -
    (r-\psi_i(x))
    -
    (r-\psi_j(y))
\\
    &=
    2r+\psi_i(x)+\psi_j(y)
    \ge
    \psi_i(x)+\psi_j(y)
\\
    &\ge
    |f_\sigma(x)-f_\sigma(y)|.
\end{aligned}
\]
Finally, if
$
    x\in\operatorname{supp}(\psi_i)
$
and $y$ does not belong to any support, then
$
    f_\sigma(y)=0
$
and
$
    \rho(y,x_i)\ge r
$.
Thus
\[
    \rho(x,y)
    \ge
    \rho(y,x_i)-\rho(x,x_i)
    \ge
    r-\rho(x,x_i)
    =
    \psi_i(x)
    =
    |f_\sigma(x)-f_\sigma(y)|.
\]
The remaining case, where both $x$ and $y$ lie outside all supports, is trivial.
Therefore $f_\sigma\in\mathcal{F}_1$.
\hfill\\
If $\sigma,\tau\in\{-1,1\}^M$ and $\sigma\neq\tau$, then there exists
$j\in[M]_+$ such that $\sigma_j\neq\tau_j$.  Evaluating at the centre $x_j$ gives
$
    |f_\sigma(x_j)-f_\tau(x_j)|
=
    |\sigma_j-\tau_j|\,\psi_j(x_j)
=
    2r
$.
Hence, $\|f_\sigma-f_\tau\|_\infty
\ge
    2r
=
    4\varepsilon
$.  
Thus, the set
$
    \{f_\sigma:\sigma\in\{-1,1\}^M\}
$
is $4\varepsilon$-separated in $\|\cdot\|_\infty$.  Therefore every
$\varepsilon$-cover of $\mathcal{F}_1$ contains at least $2^M$ elements.  By
\eqref{eq:packing_lower_bound_Ahlfors},
\begin{equation}
\label{eq:LB_ghahahh}
    \log_2
    \mathcal N
    \bigl(
        \mathcal{F}_1,
        \|\cdot\|_{\infty},
        \varepsilon
    \bigr)
    \ge
    M
    \ge
    c\,\varepsilon^{-d}.
\end{equation}
Combining~\eqref{eq:UB_ghahahh} and~\eqref{eq:LB_ghahahh} yields~\eqref{eq:Lip_class_entropy_asymptotics}, and we are happy.
\end{proof}
\begin{lem}[Near-optimality of the McShane-Whitney parametrization]
\label{lem:MW_near_optimality_Lipschitz_width}
Let $(X,\rho,\mathfrak{m})$ be a compact $d$-Ahlfors regular metric measure space,
for some $d>0$.  
Then, there exist constants
$
    0<c_d'\le C_d<\infty
$\footnote{Depending only on $(X,\rho,\mathfrak{m})$} such
that: for every $N\ge 2$
\begin{equation}
\label{eq:MW_near_optimality_Ahlfors_metric}
    c_d'\,
    \frac{1}{P_N^{1/d}(\log_2 P_N)^{1/d}}
    \le
    d^{\gamma_d}_{P_N}(\mathcal F_{1:d})_{C(X)}
    \le
    C_d\,P_N^{-1/d}
\end{equation}
where 
$
    P_N\eqdef N
$ and $\gamma_d
    \eqdef
    1+\sqrt d
%    \max\{2,1+\sqrt d\}
$.
\end{lem}
\begin{proof}[{Proof of Lemma~\ref{lem:MW_near_optimality_Lipschitz_width}}]
We prove the upper and lower bounds separately.

\noindent
\emph{Step 1: The McShane-Whitney upper bound.}
Since $(X,\rho,\mathfrak{m})$ is compact and $d$-Ahlfors regular, there exists a
constant $C_d<\infty$ such that, for every $N\in\mathbb N_+$, one may choose
points
$
    b_1,\dots,b_N\in X
$
satisfying
\begin{equation}
\label{eq:Ahlfors_net_radius_X}
    h_N
    \eqdef
    \sup_{x\in X}
        \min_{m\in[N]_+}
            \rho(x,b_m)
    \le
    C_d\,N^{-1/d}.
\end{equation}
Indeed, the upper covering estimate follows from Ahlfors regularity by taking a
maximal separated set at scale comparable to $N^{-1/d}$.
For $a=(a_1,\dots,a_N)\in[-1,1]^N$, define for each $x\in X$
\[
    \Phi_N(a)(x)
\eqdef
    \tfrac{1}{2}
    \left(
        \min_{m\in[N]_+}
            \bigl(a_m+\rho(x,b_m)\bigr)
        +
        \max_{m\in[N]_+}
            \bigl(a_m-\rho(x,b_m)\bigr)
    \right)
.
\]
Let $f\in\mathcal F_{1:d}$ and choose
$
    a_m\eqdef f(b_m)
$ for each $m\in[N]_+$.  
Since $f(X)\subseteq[-1,1]$, this vector $a$ belongs to $[-1,1]^N$.
Recall the definitions of the upper and lower McShane-Whitney envelopes
$
    U_f(x)
    \eqdef
    \min_{m\in[N]_+}
        \bigl(f(b_m)+\rho(x,b_m)\bigr)
$ and $
    L_f(x)
\eqdef
\linebreak
    \max_{m\in[N]_+}
        \bigl(f(b_m)-\rho(x,b_m)\bigr)
$.  
Then, $
    \Phi_N(a)(x)
=
    \frac{1}{2}\bigl(U_f(x)+L_f(x)\bigr)
$.
Since $f$ is $1$-Lipschitz, for every $x\in X$ and every $m\in[N]_+$,
$
    f(b_m)-\rho(x,b_m)
\le
    f(x)
\le
    f(b_m)+\rho(x,b_m)
$.
Taking the maximum on the left and the minimum on the right gives
\begin{equation}
\label{eq:MW_squeeze_Ahlfors_metric}
    L_f(x)
\le
    f(x)
\le
    U_f(x)
\end{equation}
for each $x\in X$.
Now fix $x\in X$ and choose $m_x\in[N]_+$ such that
$
    \rho(x,b_{m_x})
\le
    h_N
$.
By the definitions of $U_f$ and $L_f$,
$
    U_f(x)
    \le
    f(b_{m_x})+\rho(x,b_{m_x})
$ and $
    L_f(x)
\ge
    f(b_{m_x})-\rho(x,b_{m_x})
$.
Consequently,
$
    0
\le
    U_f(x)-L_f(x)
\le
    2\rho(x,b_{m_x})
\le
    2h_N
$.
Since $f(x)$ lies in the interval $[L_f(x),U_f(x)]$ by
\eqref{eq:MW_squeeze_Ahlfors_metric}, and $\Phi_N(a)(x)$ is the midpoint of this
interval, we obtain
$
    |f(x)-\Phi_N(a)(x)|
\le
    \frac{1}{2}\bigl(U_f(x)-L_f(x)\bigr)
\le
    h_N
$.
Taking the supremum over $x\in X$ and using~\eqref{eq:Ahlfors_net_radius_X}
yields
\begin{equation}
\label{eq:MW_error_bound_Ahlfors_metric}
    \|f-\Phi_N(a)\|_{C(X)}
\le
    h_N
\le
    C_d\,N^{-1/d}
.
\end{equation}
We next verify that this is an admissible Lipschitz-width parametrization.  Equip
$\mathbb R^N$ with the coordinate-wise $\ell^\infty$ norm.  If
$a,a'\in[-1,1]^N$, then, for every $x\in X$,
\[
\begin{aligned}
    |\Phi_N(a)(x)-\Phi_N(a')(x)|
    &\le
    \frac{1}{2}
    \left|
        \min_{m\in[N]_+}
            \bigl(a_m+\rho(x,b_m)\bigr)
        -
        \min_{m\in[N]_+}
            \bigl(a_m'+\rho(x,b_m)\bigr)
    \right|
\\
    &\quad
    +
    \frac{1}{2}
    \left|
        \max_{m\in[N]_+}
            \bigl(a_m-\rho(x,b_m)\bigr)
        -
        \max_{m\in[N]_+}
            \bigl(a_m'-\rho(x,b_m)\bigr)
    \right|
    \\
    &\le
    \|a-a'\|_{\ell^\infty}.
\end{aligned}
\]
Hence
$
    \|\Phi_N(a)-\Phi_N(a')\|_{C(X)}
\le
    \|a-a'\|_{\ell^\infty}
$.
Thus $\Phi_N:B_{\ell^\infty_N}\to C(X)$ is $1$-Lipschitz, and therefore also
$\gamma_d$-Lipschitz.  Combining this with
\eqref{eq:MW_error_bound_Ahlfors_metric} gives
\begin{equation}
\label{eq:MW_upper_rate_parameter_count_Ahlfors_metric}
    d^{\gamma_d}_{P_N}(\mathcal F_{1:d})_{C(X)}
    \le
    C_d\,N^{-1/d}
    =
    C_d\,P_N^{-1/d}.
\end{equation}

\noindent
\emph{Step 2: The Lipschitz-width lower bound.}
By Lemma~\ref{lem:Nondegeneracy_metric_entropy}, there exist constants
$c,C>0$ and $\varepsilon_0\in(0,1)$ such that, for every
$0<\varepsilon<\varepsilon_0$, we have that
$
    c\,\varepsilon^{-d}
\le
    \log_2
    \mathcal N
    \bigl(
        \mathcal F_{1:d},
        \varepsilon,
        \|\cdot\|_\infty
    \bigr)    \le
    C\,\varepsilon^{-d}.
$
In particular, the entropy numbers of $\mathcal F_{1:d}$ satisfy the lower-growth condition: for each $2\le n \in \mathbb{N}_+$ we have
\begin{equation}
\label{eq:entropy_number_lower_Ahlfors_metric}
    \varepsilon_n(\mathcal F_{1:d})_{C(X)}
\ge
    c_d\,n^{-1/d}
,
\end{equation}
after decreasing $c_d>0$ if necessary.  Indeed, if
$
    \varepsilon<c_d n^{-1/d}
$,
then
$
    \log_2\mathcal N(\mathcal F_{1:d},\varepsilon,\|\cdot\|_\infty)>n
$,
so $\mathcal F_{1:d}$ cannot be covered by $2^n$ balls of radius $\varepsilon$.
We may now apply the beautiful~\cite[Theorem 3.1 (i)]{petrova2023lipschitz}.  The lower entropy
estimate~\eqref{eq:entropy_number_lower_Ahlfors_metric} is the hypothesis of that
result with
$
    \alpha=\frac{1}{d}
$ and $
    \beta=0
$ in their notation.  
Noting that,
$
    \operatorname{rad}(\mathcal F_{1:d})_{C(X)}
\le
    1,
$ 
because the constant function $0$ approximates every
$
    f\in\mathcal F_{1:d}
$
with error at most $1$, then we deduce that 
\[
    2\operatorname{rad}(\mathcal F_{1:d})_{C(X)}
\le
    2
\le
    \gamma_d.
\]
Thus~\cite[Theorem 3.1 (i)]{petrova2023lipschitz} yields a constant
$c_d'>0$ such that, for every $P\ge2$,
\begin{equation}
\label{eq:Lipschitz_width_lower_rate_Lip_class_Ahlfors_metric}
    d^{\gamma_d}_{P}(\mathcal F_{1:d})_{C(X)}
    \ge
    c_d'\,
    \frac{1}{P^{1/d}(\log_2 P)^{1/d}}.
\end{equation}
Applying~\eqref{eq:Lipschitz_width_lower_rate_Lip_class_Ahlfors_metric} with
$P=P_N=N$ and combining it with
\eqref{eq:MW_upper_rate_parameter_count_Ahlfors_metric}, we obtain~\eqref{eq:MW_near_optimality_Ahlfors_metric}.
\end{proof}
We note to the reader that, in the Euclidean cube case, if the centres $b_m\in Q_d\subseteq\mathbb R^d$ are
treated as free parameters, then the parameter count is $N(d+1)$, corresponding
to $N$ labels and $Nd$ centre coordinates.  In the intrinsic metric-space
formulation above, the centres $b_1,\dots,b_N\in X$ are fixed design points,
chosen once and for all as a near-optimal net.  The only free real parameters are
the labels $a_1,\dots,a_N$.  Hence the Lipschitz-width parameter count is
$P_N=N$.

\subsubsection{Step 4: Matching the Lower Bounds}
\label{s:Optimality__ss:ProofStep4Matching}
We are now ready to complete our proof and deduce minimax-optimality, in the stable non-linear approximation theoretic sense of~\cite{petrova2023lipschitz}; namely that we our formula matches the Lipschitz widths up to polylogarithmic factors.

\begin{proof}[{Proof of Theorem~\ref{thm:stable_near_minimax_optimality_oracle_regime}}]
Fix $N\ge2$.  Since $(X,\rho,\mathfrak{m})$ is compact and $d$-Ahlfors regular,
there exists a constant $C_d<\infty$ and a choice of design points
$
    x_1,\dots,x_N\in X
$
such that
\[
    \sup_{x\in X}
        \min_{n\in[N]_+}
            \rho(x,x_n)
    \le
    C_d\,N^{-1/d}.
\]
Equivalently, after changing $C_d$ if necessary, the chosen design points form a
$C_dN^{-1/d}$-net of $X$.
Let $f\in\mathcal F_{1:d}$ and set
$
    \theta_f
\eqdef
    \big((f(x_n),x_n,1)\big)_{n=1}^N
$.
Since $f(X)\subseteq[-1,1]$, one has
$
    \theta_f
    \in
    \big([-1,1]\times X\times\{1\}\big)^N 
$.
Applying Lemma~\ref{lem:oracle_access__approximation}, with the net radius
$
    C_dN^{-1/d}
$
in place of $N^{-1/d}$, gives
\[
    \|f-\hat{f}_{\theta_f}\|_{C(X)}
    \le
    C_d\,N^{-1/d}
    =
    C_d\,P_N^{-1/d}.
\]
Taking the supremum over $f\in\mathcal F_{1:d}$ yields the asserted upper
reconstruction error.
Next, Lemma~\ref{lem:stability}, applied with $R=1$ and
$
    \operatorname{diam}(X)\le1
$,
shows that the parameter-to-function map
$
    \mathfrak{E}:
    \big([-1,1]\times X\times\{1\}\big)^N
    \to
    C(X)
$ for each $\theta\mapsto \hat{f}_\theta$, 
is $1$-Lipschitz.  Since
$
    1\le\gamma_d
$,
the same map is also $\gamma_d$-Lipschitz.  Applying Lemma~\ref{lem:MW_near_optimality_Lipschitz_width} gives
\[
    c_d'\,
    \frac{1}{P_N^{1/d}(\log_2 P_N)^{1/d}}
\le
    d^{\gamma_d}_{P_N}(\mathcal F_{1:d})_{C(X)}
\le
    C_d\,P_N^{-1/d}
.
\]
The left-hand inequality says that no $\gamma_d$-Lipschitz reconstruction scheme
using $P_N$ real parameters can have worst-case error smaller than
$
    c_d'\,
    P_N^{-1/d}(\log_2 P_N)^{-1/d}
$.
The construction above achieves error at most
$
    C_dP_N^{-1/d}
$.
Therefore the reconstruction formula~\eqref{eq:formula} is optimal, in the sense
of Lipschitz widths, up to the logarithmic factor
$
    (\log_2 P_N)^{1/d}
$.
\end{proof}

\begin{proof}[{Proof of Proposition~\ref{prop:approximation}}]
Directly follows from the proof of Theorem~\ref{thm:stable_near_minimax_optimality_oracle_regime} without assuming Ahlfors regularity, and only the weaker $0<d$-doubling condition.
\end{proof}

\subsubsection{Proof of Optimal Regularity}
\label{s:Optimal_Regularity}
\begin{proof}
Set $q_m\eqdef(aG^{(\beta)})_m$ for $m\in[M]_+$. By~\cite[Proposition~1.3.7]{cobzas2019lipschitz}, the map $x\mapsto\rho(x,b_m)$ is $1$-Lipschitz. Hence, by~\cite[Proposition~2.3.3]{cobzas2019lipschitz}, each of the maps
$
    x\mapsto q_m+L\rho(x,b_m)
$
and
$
    x\mapsto q_m-L\rho(x,b_m)
$
is $L$-Lipschitz. Therefore,~\cite[Proposition~2.3.9]{cobzas2019lipschitz} implies that
$
    x\mapsto\min_{m\in[M]_+}\{q_m+L\rho(x,b_m)\}
$
and
$
    x\mapsto\max_{m\in[M]_+}\{q_m-L\rho(x,b_m)\}
$
are both $L$-Lipschitz. Applying~\cite[Proposition~2.3.3]{cobzas2019lipschitz} once more to their half-sum yields
$
    \operatorname{Lip}\bigl(\hat{f}_\theta(\cdot|G^{(\beta)})\bigr)
    \le \tfrac12(L+L)=L
$,
as required.
\end{proof}

\subsubsection{Proofs of Parametric Optimality}
\label{s:Optimality}
\begin{proof}[{Proof of Proposition~\ref{prop:representation}}]
%%%
Fix $\theta=((a_m,b_m,c_m))_{m=1}^N$ and write
$
    \overline a_m\eqdef(aG^{(\beta)})_m
$.
Since
$
    |t|=\operatorname{ReLU}(t)+\operatorname{ReLU}(-t),
$
each map
$
    x\mapsto c_m\|x-b_m\|_1
$
admits an exact constant-depth $\operatorname{ReLU}$ realization with $\mathcal{O}(d)$ non-zero parameters. Hence, by parallelization~\cite[Lemma~5.3]{petersen2024mathematical}, the map $f_\theta:\mathbb{R}^d
\mapsto
    \mathbb{R}^{2N}$ given for each $x\in \mathbb{R}^d$ by
\[
    x
    \mapsto
    \left(
        \big(\overline a_m+c_m\|x-b_m\|_1\big)_{m=1}^N,
        \big(\overline a_m-c_m\|x-b_m\|_1\big)_{m=1}^N
    \right)
\]
admits an exact constant-depth $\operatorname{ReLU}$ realization of width and size $\mathcal \mathcal{O}(N)$.
By~\cite[Lemma~5.11]{petersen2024mathematical}, the maps
$
    (z_m)_{m=1}^N\mapsto\min_{m\in[N]_+}z_m
$
and
$
    (z_m)_{m=1}^N\mapsto\max_{m\in[N]_+}z_m
$
admit exact $\operatorname{ReLU}$ realizations of depth at most $\lceil\log_2N\rceil$, width at most $3N$, and size at most $16N$. Parallelizing these two networks and composing them with the preceding realization and the affine map $(u,v)\mapsto\tfrac12(u+v)$, as permitted by~\cite[Lemmata~5.2-5.3]{petersen2024mathematical}, yields a \texttt{ReLU}-MLP $\Phi:\mathbb R^d\to\mathbb R$ satisfying
\[
    \Phi(x)
    =
    \frac12
    \left(
        \min_{m\in[N]_+}
        \{\overline a_m+c_m\|x-b_m\|_1\}
        +
        \max_{m\in[N]_+}
        \{\overline a_m-c_m\|x-b_m\|_1\}
    \right)
    =
    \hat{f}_\theta(x|G^{(\beta)})
\]
for every $x\in\mathbb R^d$, with depth $\mathcal{O}(\log(N))$ and width and size $\mathcal \mathcal{O}(N)$.
\end{proof}

We directly obtain the transformer version.
\begin{proof}[{Proof of Corollary~\ref{cor:transformer_realization}}]
 Directly follows upon applying Proposition~\ref{prop:transformerification__SparseVersion} to $\Phi_{\theta,G^{(\beta)}}$ obtained from Proposition~\ref{prop:representation}.   
\end{proof}

\begin{proof}[{Proof of Corollary~\ref{cor:lp_relu_approximation}}]
Suppose first that $1\le p\le2$. Since $\|x-y\|_p\le\|x-y\|_1$, the map $f$ is $L$-Lipschitz with respect to $\|\cdot\|_1$. A standard tensor grid with at most $N$ points has $\ell^1$-covering radius at most $dN^{-1/d}$; hence Proposition~\ref{prop:approximation} and Proposition~\ref{prop:representation} yield a \texttt{ReLU}-MLP $\Phi$ with the asserted parametric complexity and $\|f-\Phi\|_{C(X)}\le LdN^{-1/d}$. Since $\Phi$ is $L$-Lipschitz with respect to $\|\cdot\|_1$ and $\|x-y\|_1\le d^{1-1/p}\|x-y\|_p$, we obtain $\operatorname{Lip}(\Phi)\le Ld^{1-1/p}$ with respect to $\|\cdot\|_p$.
\hfill\\
\noindent
If instead $2<p\le\infty$, then $\|x-y\|_p\le d^{1/p}\|x-y\|_\infty$, so $f$ is $Ld^{1/p}$-Lipschitz with respect to $\|\cdot\|_\infty$. Since the same grid has $\ell^\infty$-covering radius at most $N^{-1/d}$, the corresponding $\ell^\infty$ versions of Propositions~\ref{prop:approximation} and~\ref{prop:representation} give $\|f-\Phi\|_{C(X)}\le Ld^{1/p}N^{-1/d}$ and $\operatorname{Lip}(\Phi)\le Ld^{1/p}$ with respect to $\|\cdot\|_p$. Therefore the conclusion follows.
\end{proof}

\begin{proof}[{Proof of Corollary~\ref{cor:realizability}}]
In the setting of
Theorem~\ref{thrm:nonparametric_noisy_case}, define the effective parameter
$
    \theta^{(\beta,q_\varepsilon,1)}
\eqdef
    \big(
        (    V_k^{(\beta,q_\varepsilon)},X_k,1
        )
    \big)_{k=1}^{N_0}$.  
By Lemma~\ref{lem:graph_averaged_labels_equal_virtual_McShane_Whitney},
\[
    \hat{f}_{\mathbb{D}_N}^{(\beta,q_\varepsilon,1)}
    =
    \hat{f}_{\theta^{(\beta,q_\varepsilon,1)}}
\]
pointwise on $X$.  Proposition~\ref{prop:representation}, applied with
$M=N_0$, therefore gives an exact \texttt{ReLU}-MLP realization
$
    \Phi:\mathbb{R}^d\to\mathbb{R}
$
of depth $\mathcal O(\log N_0)$ and with
$\mathcal O_d(N_0)$ nonzero parameters.  Since $N_0\le N$, these bounds
are also $\mathcal O(\log N)$ and $\mathcal O_d(N)$, respectively.
Finally,
$
    \Phi
    =
    \hat{f}_{\mathbb{D}_N}^{(\beta,q_\varepsilon,1)}
$
on $X$, so the probabilistic uniform-recovery estimate follows directly
from~\eqref{eq:main_robustified_uniform_bound}.
\end{proof}

\subsubsection{Proofs of Parametric Optimality}
\label{s:Proofs__ss:ParametricOptimality}
\begin{proof}[{Proof of Corollary~\ref{cor:realizability}}]
Fix $\theta=((a_m,b_m,c_m))_{m=1}^N$ and write
$
    \overline a_m\eqdef(aG^{(\beta)})_m
$.
Since
$
    |t|=\operatorname{ReLU}(t)+\operatorname{ReLU}(-t),
$
each map
$
    x\mapsto c_m\|x-b_m\|_1
$
admits an exact constant-depth $\operatorname{ReLU}$ realization with $\mathcal{O}(d)$ non-zero parameters. Hence, by parallelization~\cite[Lemma~5.3]{petersen2024mathematical}, the map $f_\theta:\mathbb{R}^d
\mapsto
    \mathbb{R}^{2N}$ given for each $x\in \mathbb{R}^d$ by
\[
    x
    \mapsto
    \left(
        \big(\overline a_m+c_m\|x-b_m\|_1\big)_{m=1}^N,
        \big(\overline a_m-c_m\|x-b_m\|_1\big)_{m=1}^N
    \right)
\]
admits an exact constant-depth $\operatorname{ReLU}$ realization of width and size $\mathcal \mathcal{O}(N)$.
By~\cite[Lemma~5.11]{petersen2024mathematical}, the maps
$
    (z_m)_{m=1}^N\mapsto\min_{m\in[N]_+}z_m
$
and
$
    (z_m)_{m=1}^N\mapsto\max_{m\in[N]_+}z_m
$
admit exact $\operatorname{ReLU}$ realizations of depth at most $\lceil\log_2N\rceil$, width at most $3N$, and size at most $16N$. Parallelizing these two networks and composing them with the preceding realization and the affine map $(u,v)\mapsto\tfrac12(u+v)$, as permitted by~\cite[Lemmata~5.2-5.3]{petersen2024mathematical}, yields a \texttt{ReLU}-MLP $\Phi:\mathbb R^d\to\mathbb R$ satisfying
\[
    \Phi(x)
    =
    \frac12
    \left(
        \min_{m\in[N]_+}
        \{\overline a_m+c_m\|x-b_m\|_1\}
        +
        \max_{m\in[N]_+}
        \{\overline a_m-c_m\|x-b_m\|_1\}
    \right)
    =
    \hat{f}_\theta(x|G^{(\beta)})
\]
for every $x\in\mathbb R^d$, with depth $\mathcal{O}(\log(N))$ and width and size $\mathcal \mathcal{O}(N)$.
Finally, in the setting of Theorem~\ref{thrm:nonparametric_noisy_case}, applying this exact realization to $\theta^\star$ gives
$
    \Phi=\hat{f}_{\theta^\star}(\cdot|G^{(\beta)})
$
on $X$, so the claimed probabilistic estimate follows directly from that theorem.
\end{proof}

\subsection{Proofs of Examples}
\label{s:Proofs_Examples}

\begin{proof}[{Proof of Corollary~\ref{cor:uniformlycontinuous}}]
Set $\rho^\omega\eqdef\omega\circ\rho$. By the preceding discussion, $(X,\rho^\omega)$ is a compact metric space and $f$ is $1$-Lipschitz with respect to $\rho^\omega$. Fix $\varepsilon>0$ and choose a finite $\varepsilon/2$-net $\{b_m\}_{m=1}^M$ of $(X,\rho^\omega)$. Set
$
    \theta\eqdef((f(b_m),b_m,1))_{m=1}^M.
$
Since $G^0=I_M$, the formula in~\eqref{eq:formula__realizable___snowflaked} is precisely the Whitney-McShane midpoint construction applied on $(X,\rho^\omega)$. The argument of Lemma~\ref{lem:oracle_access__approximation} therefore gives
$
    \|f-\hat{f}_{\theta:\omega}(\cdot|G^0)\|_\infty<\varepsilon.
$
Finally, each map $x\mapsto f(b_m)\pm\rho^\omega(x,b_m)$ is $1$-Lipschitz with respect to $\rho^\omega$, and finite minima, maxima, and their average preserve this bound. Hence
$
    |\hat{f}_{\theta:\omega}(x|G^0)-\hat{f}_{\theta:\omega}(y|G^0)|
    \le \rho^\omega(x,y)
    =\omega(\rho(x,y)),
$
which proves the claimed $\omega$-uniform continuity.
\end{proof}

\begin{proof}[{Proof of Corollary~\ref{cor:DeepSets}}]
Since $T_\phi$ is injective on $\mathcal{K}$, $d_\phi$ is a metric, and $(\mathcal{K},d_\phi)$ is compact by assumption. Apply Corollary~\ref{cor:uniformlycontinuous} with the modulus $\omega(t)=Lt$. This yields, for every $\varepsilon>0$, parameters $\theta$ such that
$
    \|f-\hat{f}_\theta(\cdot|G^0)\|_{C(\mathcal{K})}<\varepsilon
$
and $\hat{f}_\theta(\cdot|G^0)$ is $L$-Lipschitz with respect to $d_\phi$. By~\eqref{eq:DeepSets_factorization}, $\hat{f}_\theta(\nu|G^0)=\Psi_\theta(T_\phi(\nu))$, while the $\ell^\infty$-version of the realization argument in Proposition~\ref{prop:representation} gives an exact \texttt{ReLU}-MLP realization of $\Psi_\theta$. The conclusion follows.
\end{proof}

\appendix
\section{Examples for Numerics}

Fix $d,P\in\mathbb N_+$, and let
$
L^\infty([0,1];(\mathbb R^d,\|\cdot\|_2))
$
denote the space of essentially bounded Lebesgue-a.e.\ equivalence classes of
functions from $[0,1]$ into $\mathbb R^d$.  Let $\Pi_P$ denote the subset of
functions $f$ of the form
\[
    f=\sum_{i=1}^m a_i\mathbf 1_{I_i},
    \qquad m\le P,
\]
where $a_i\in\mathbb R^d$, and where $\{I_i\}_{i=1}^m$ is a partition of
$[0,1]$, up to Lebesgue-null sets, by non-empty intervals of the form
$[\alpha,\beta)$, $(\alpha,\beta)$, or $(\alpha,\beta]$, with
$0\le \alpha<\beta\le 1$.
Then, for any two
$
    f=\sum_{i=1}^m a_i\mathbf 1_{I_i}
$ and $g=\sum_{j=1}^n b_j\mathbf 1_{J_j}$ in $\Pi_P$ 
their distance is given exactly by
\begin{equation}
\label{eq:oracle_distance}
    \|f-g\|_{L^\infty([0,1];\mathbb R^d)}
    =
    \max_{\substack{1\le i\le m,\;1\le j\le n\\
                    \lambda(I_i\cap J_j)>0}}
    \|a_i-b_j\|_2 ,
\end{equation}
where $\lambda$ denotes Lebesgue measure.

\section{Additional Numerics}\label{sec:numerics_appendix}

\subsection{Closed-Form Reconstruction on $\mathbb{R}$: Different $N_{\mathrm{train}}$}

\begin{figure}[ht]
    \centering

    \begin{subfigure}{0.48\textwidth}
        \centering
        \includegraphics[width=\textwidth]{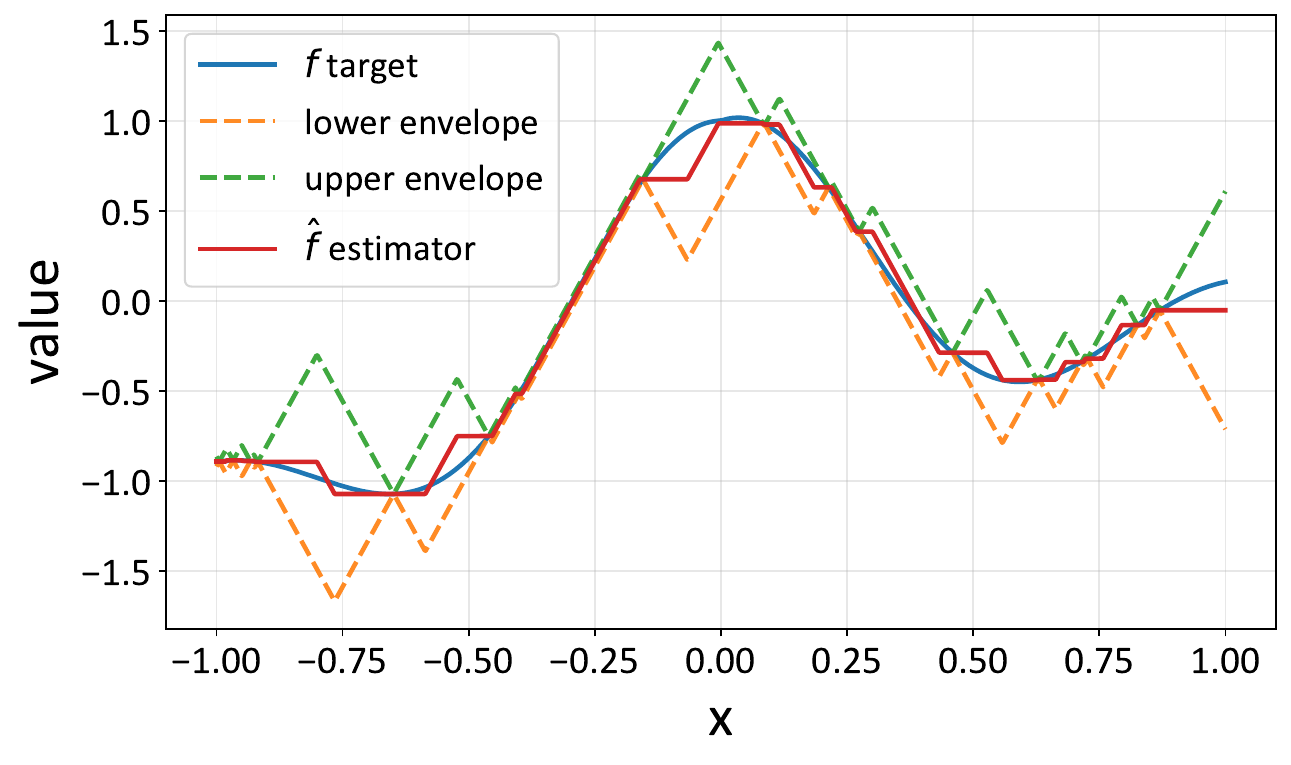}
        \caption{$N_{\mathrm{train}}=20$}
        \label{fig:closed-form-r1-n20}
    \end{subfigure}
    \hfill
    \begin{subfigure}{0.48\textwidth}
        \centering
        \includegraphics[width=\textwidth]{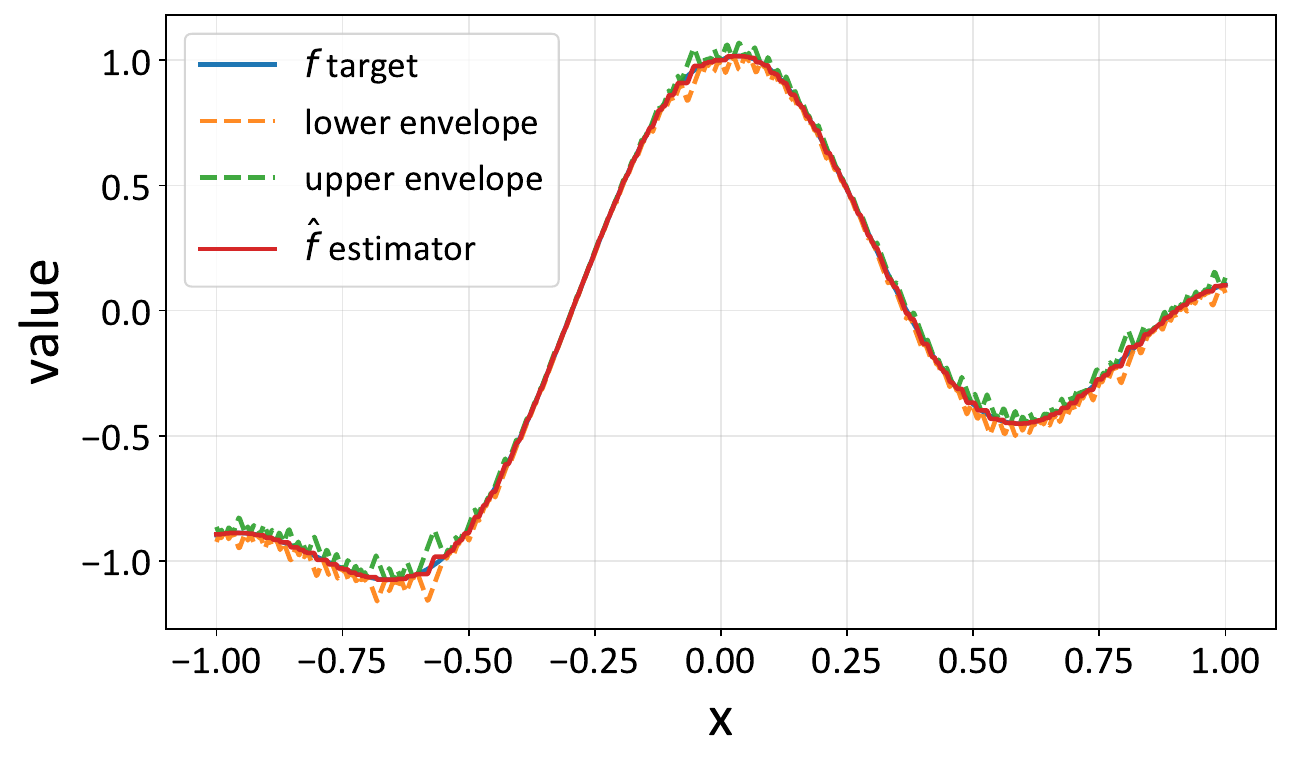}
        \caption{$N_{\mathrm{train}}=200$}
        \label{fig:closed-form-r1-n200}
    \end{subfigure}

    \caption{Closed-form Whitney-McShane midpoint reconstruction on $[-1,1]$ with noiseless samples. The lower and upper envelopes correspond to the Whitney and McShane bounds, and the estimator is their midpoint.}
    \label{fig:closed-form-r1}
\end{figure}

\subsection{Closed-Form Reconstruction on $\mathbb{R}^2$: $\ell_2$-distance}

\begin{figure}[ht]
    \centering
    \includegraphics[width=0.95\textwidth]{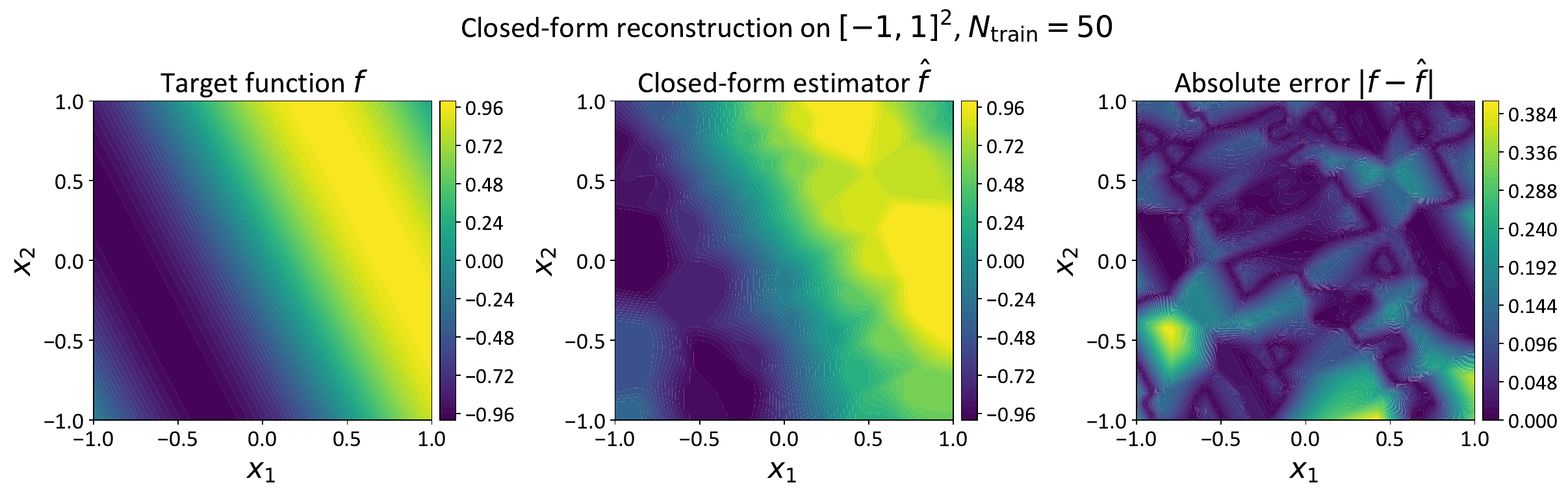}
    \caption{Closed-form reconstruction on $[-1,1]^2$ with $N_{\mathrm{train}}=50$ noiseless samples and $\ell_2$-distance as a metric.}
    \label{fig:closed-form-r2-n50-ell2}
\end{figure}

\begin{figure}[ht]
    \centering
    \includegraphics[width=0.95\textwidth]{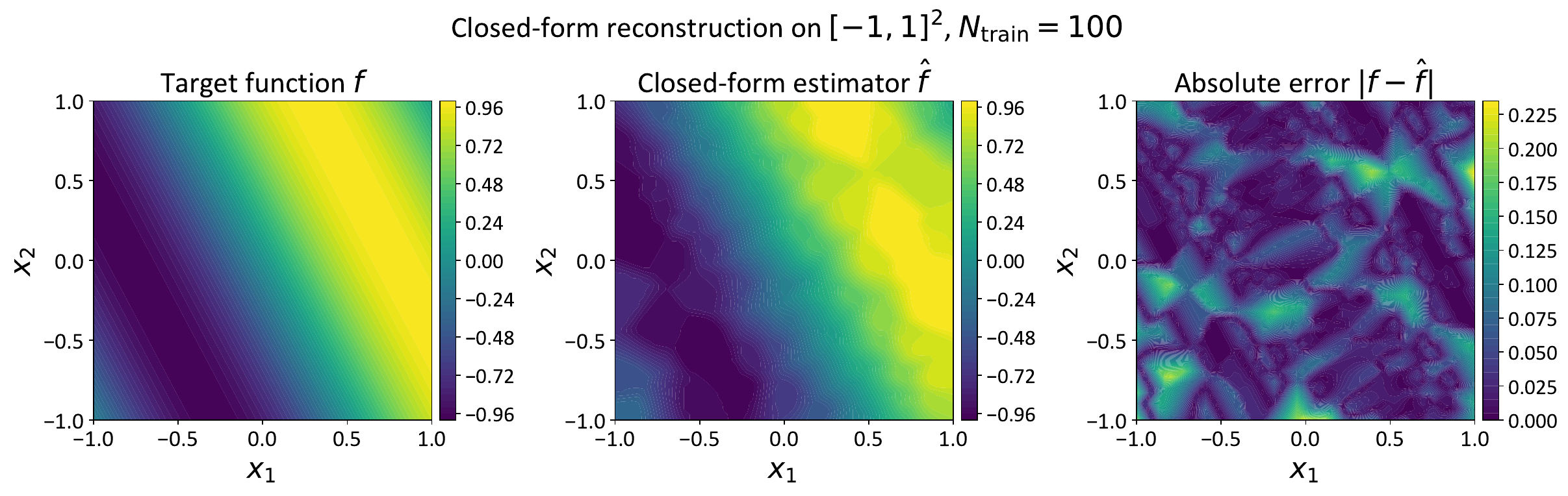}
    \caption{Closed-form reconstruction on $[-1,1]^2$ with $N_{\mathrm{train}}=100$ noiseless samples and $\ell_2$-distance as a metric.}
    \label{fig:closed-form-r2-n100-ell2}
\end{figure}

\begin{figure}[ht]
    \centering
    \includegraphics[width=0.95\textwidth]{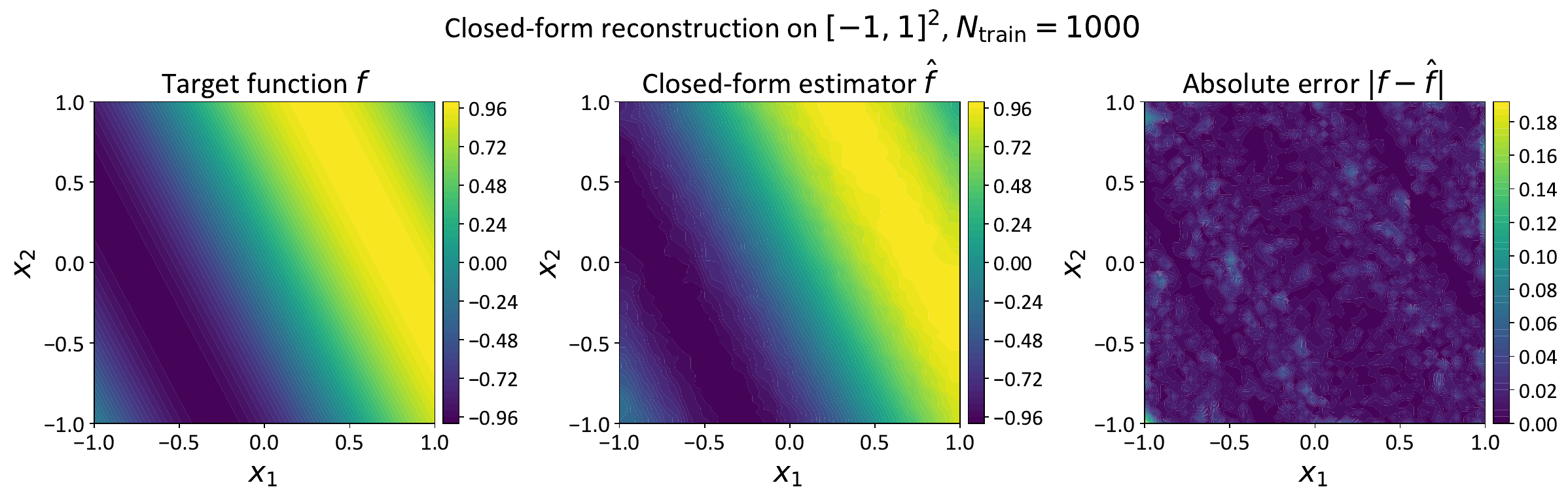}
    \caption{Closed-form reconstruction on $[-1,1]^2$ with $N_{\mathrm{train}}=1000$ noiseless samples and $\ell_2$-distance as a metric.}
    \label{fig:closed-form-r2-n1000-ell2}
\end{figure}

\subsection{Gradient Descent Recovery: 1-Dimensional Function}

%We numerically illustrate the local recovery statement of Section~\ref{s:Intro__ss:Contribution___sss:GD} in the one-dimensional noiseless setting. 
We numerically investigate whether gradient descent initialized near the closed-form parameter configuration recovers it in the one-dimensional noiseless setting.
We sample training points $X_n \in [-1,1]$ uniformly and set $Y_n=f(X_n)$, where
\[
f(x)=e^{-x^2}\cos(5x)+\max\{0,x\}+\min\{0,-x^2\}.
\]
We then form the closed-form parameter configuration
\[
\theta^\star=(Y_m,X_m,L)_{m=1}^N,
\]
where $N=1000$ and $L$ is chosen as a strict empirical Lipschitz upper bound. In particular, we estimate the empirical Lipschitz constant from the training data and use a slightly inflated value of $L$ to ensure a positive empirical Lipschitz margin.

Gradient descent is initialized near $\theta^\star$ by perturbing each component as
\[
a_m^{(0)}=Y_m+\sigma \xi_m,\qquad
b_m^{(0)}=X_m+\sigma \zeta_m,\qquad
c_m^{(0)}=L+\sigma \upsilon_m,
\]
where $\xi_m,\zeta_m,\upsilon_m$ are independent standard normal random variables and $\sigma>0$ controls the initialization gap.
We run the experiment for $\lambda =1$, $\eta = 0.001$ and three initialization gaps,
\[
\sigma \in \{10^{-4}, 5\cdot 10^{-4}, 10^{-3}\}.
\]
For each value of $\sigma$, we report the evolution of the regularized loss, the parameter error $\|\theta_t-\theta^\star\|_2$, the maximum training residual, and the function error to the closed-form estimator in iterations.

\begin{figure}[ht]
    \centering
    \includegraphics[width=0.75\textwidth]{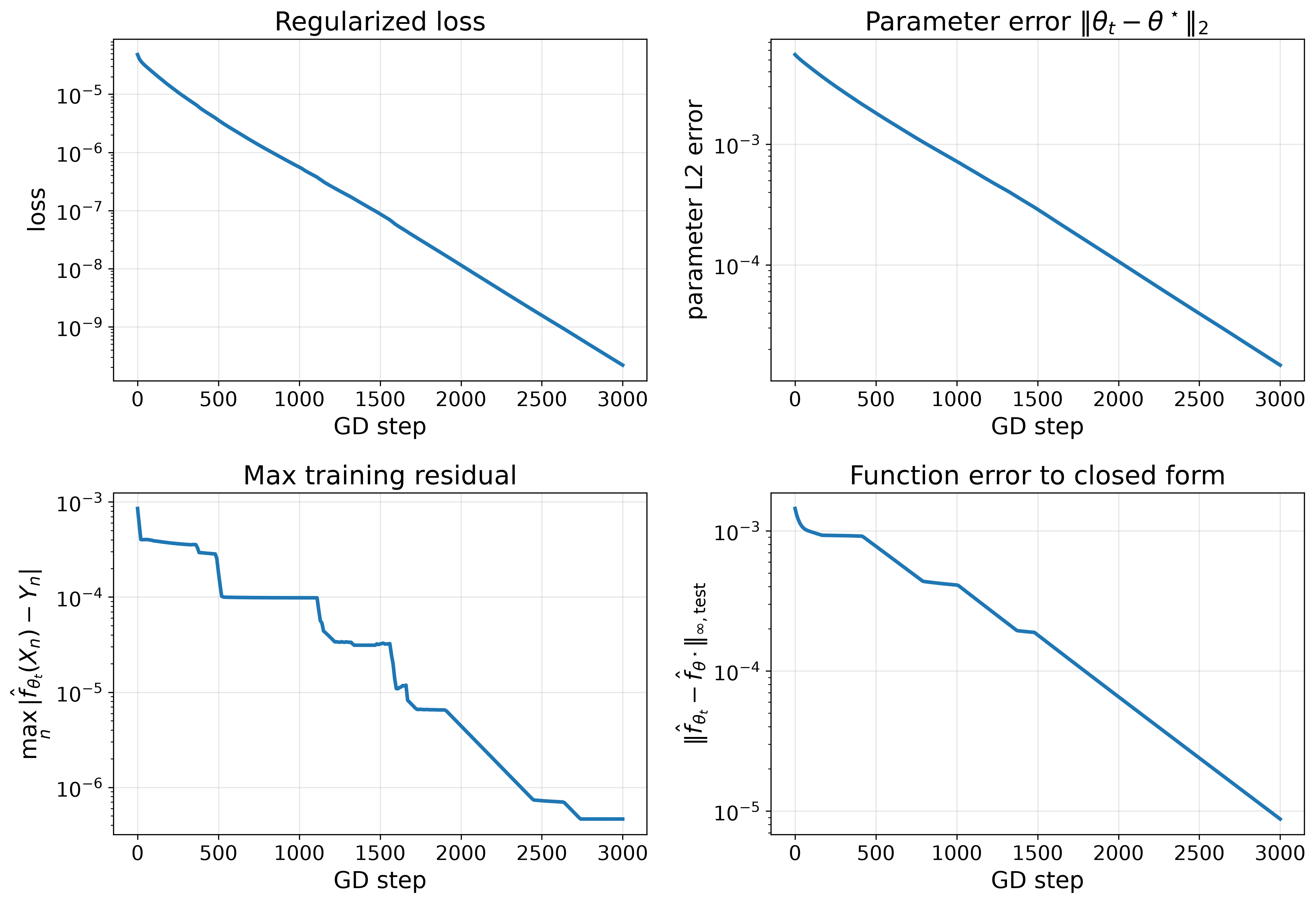}
    \caption{Gradient descent recovery in the one-dimensional noiseless setting with initialization gap $\sigma=10^{-4}$.}
    \label{fig:gd-recovery-r1-1e-4}
\end{figure}

\begin{figure}[ht]
    \centering
    \includegraphics[width=0.75\textwidth]{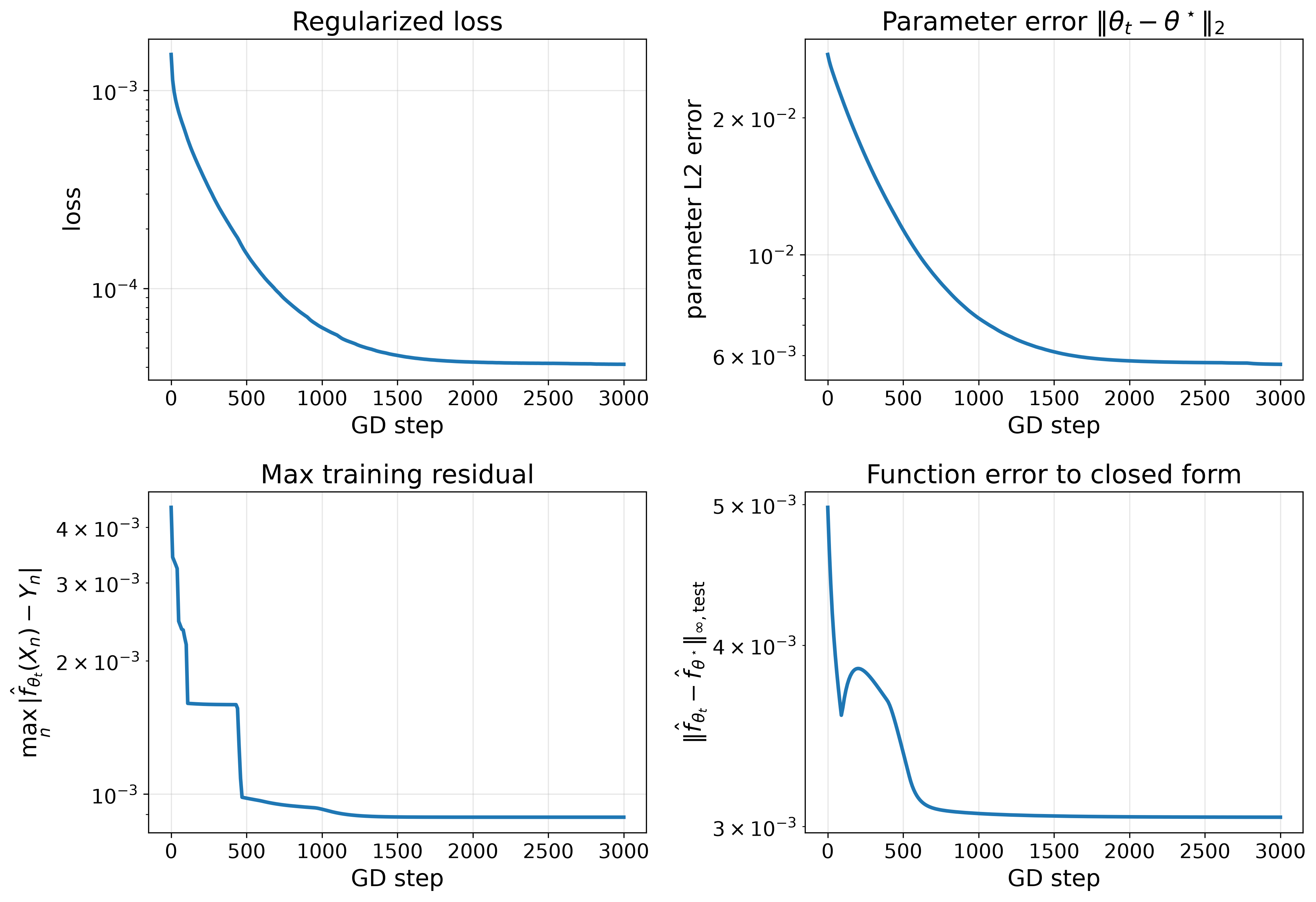}
    \caption{Gradient descent recovery in the one-dimensional noiseless setting with initialization gap $\sigma=5\cdot 10^{-4}$.}
    \label{fig:gd-recovery-r1-5e-4}
\end{figure}

\begin{figure}[ht]
    \centering
    \includegraphics[width=0.75\textwidth]{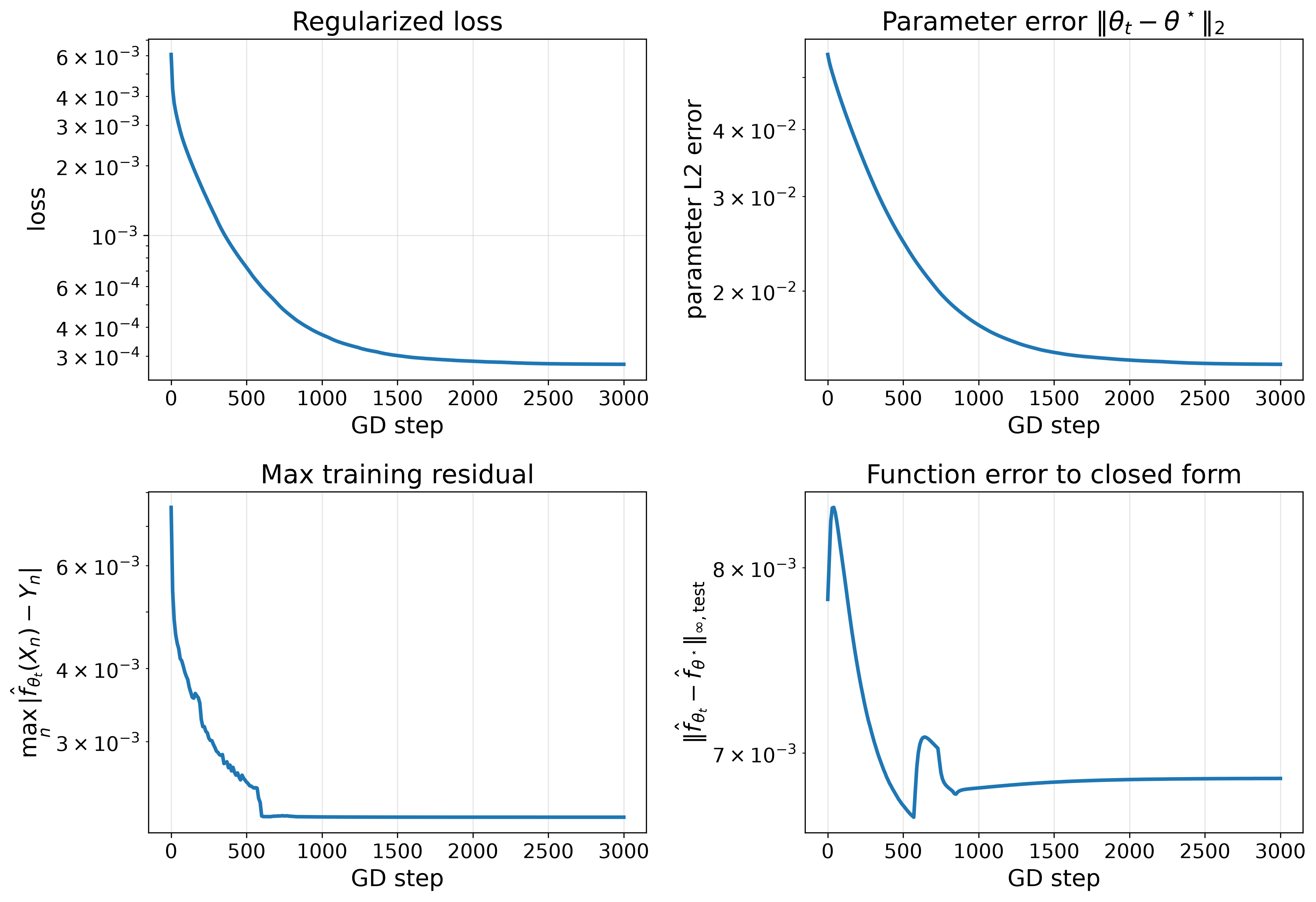}
    \caption{Gradient descent recovery in the one-dimensional noiseless setting with initialization gap $\sigma=10^{-3}$.}
    \label{fig:gd-recovery-r1-1e-3}
\end{figure}

% \newpage
\section{Deep learning architectures}
\label{a:deeplearningarchs}
For completeness, we briefly recall the relevant deep learning models.

\begin{definition}[Fully connected MLPs]
Let $d,D,L\in\mathbb{N}$, and let $d_0,d_1,\dots,d_L\in\mathbb{N}$ such that $d_0=d$ and $d_L=D$.
Let $\sigma\in \mathcal{C}(\mathbb{R})$ be an activation function.
A fully-connected MLP with depth $L$ and activation function $\sigma$ is a map $f:\mathbb{R}^d\to\mathbb{R}^D$ with iterative representation:
\begin{align*} %\label{eq:representation_MLP}
    &X^{(0)} 
    \eqdef  X, \\
    &X^{(j+1)} 
    \eqdef  
    \sigma\bullet({\bf A}^{(j)} X^{(j)} + b^{(j)}) \qquad j=0,1,\dots,L-2,
    \qquad \\
    &f(X)
    \eqdef  
    {\bf A}^{(L-1)} X^{(L-1)} + b^{(L-1)},
\end{align*}
where $\bullet$ denotes component-wise application, ${\bf A}^{(j)}\in\mathbb{R}^{d_{j+1}\times d_j}$, and $b^{(j)}\in\mathbb{R}^{d_{j+1}}$.
%Here, $J$ is the depth of the MLP and $\Upsilon\eqdef  \max_{j=1,\dots,J-1}\,d_j$ is its width.
\end{definition}

\begin{definition}[Transformers with Multi-head attention]
%Fix a temperature parameter $\lambda>0$ and dimensions $d_{\rm in},d_{{\rm key}},d_{\rm out},N\in \mathbb N$.  
Let $\lambda>0$, and let $d_{\rm in},d_{{\rm key}},d_{\rm out},N\in \mathbb N$. 
Let
${\bf Q},{\bf K}\in \mathbb{R}^{d_{{\rm key}}\times d_{\rm in}}$
and
${\bf V}\in \mathbb{R}^{d_{\rm out}\times d_{\rm in}}$. 
We define the associated attention mechanism
$\operatorname{Attn}(\cdot|{\bf Q},{\bf K},{\bf V},\lambda):\mathbb{R}^{N\times d_{\rm in}}\to \mathbb{R}^{N\times d_{\rm out}}
$
by
\begin{equation} \label{eq:def_attention}
    [\operatorname{Attn}(\mathbf{X}|{\bf Q},{\bf K},{\bf V},\lambda)]_n
    \eqdef 
    \sum_{m=1}^N
    \frac{e^{\lambda \langle {\bf Q}[\mathbf{X}]_n, {\bf K}[\mathbf{X}]_m\rangle/\sqrt{d_{{\rm key}}}}}{
    \sum_{l=1}^N e^{\lambda \langle {\bf Q}[\mathbf{X}]_n, {\bf K}[\mathbf{X}]_l \rangle/\sqrt{d_{{\rm key}}}}}
    {\bf V}[\mathbf{X}]_m, \qquad n=1,\dots,N.
\end{equation}
Here $[\mathbf{X}]_j \in \mathbb{R}^{d_{\rm in}}$ denotes the $j$th row of $\mathbf{X}$, interpreted as a column vector when multiplied by ${\bf Q},{\bf K},{\bf V}$.
%A transformer network operates by iteratively applying row-wise activation-bias maps and then applying attention mechanisms matrix-wise.  
A transformer network operates by iterating layers, each consisting of a row-wise activation–bias map followed by an attention mechanism.
Concretely, let $d,D,L\in\mathbb{N}$, and let $d_0,d_1,\dots,d_L, H_0,H_1,\dots,H_L\in\mathbb{N}$ such that $d_0=d$ and $d_L=D$.
For $j=0,1,\dots,L-2$ and $\mathtt{h}=1,\dots,H_j$, let $d^{(j)}_{{\rm key}, \mathtt{h}}, d^{(j)}_{{\rm value}, \mathtt{h}}\in\mathbb{N}$ such that $d_{j+1} = \sum_{{\mathtt{h}}=1}^{H_j} d^{(j)}_{{\rm value}, \mathtt{h}}$.
A $j$th \emph{transformer layer with Multi-head attention}, for $j=0,1,\dots,L-2$, is a map
$
\mathcal{T}_j:\mathbb{R}^{N\times d_j}\to \mathbb{R}^{N\times d_{j+1}}
$
that sends an input matrix $\mathbf{X}^{(j)}\in \mathbb{R}^{N\times d_j}$ to
$
\mathcal{T}_j(\mathbf{X}^{(j)})\eqdef  \mathbf{X}^{(j+1)}\in \mathbb{R}^{N\times d_{j+1}},
$
where 
\begin{alignat}{2} %\label{eq:representation_transformer}
    \nonumber \mathbf{Z}^{(j)}
    &\eqdef 
    \underbrace{
        \bigoplus_{\mathtt{h}=1}^{H_j}
            \operatorname{Attn}\big(
                \mathbf{X}^{(j)}
                |{\bf Q}_{\mathtt{h}}^{(j)},{\bf K}_{\mathtt{h}}^{(j)},{\bf V}_{\mathtt{h}}^{(j)},\lambda
            \big)
    }_{\text{Multi-head attention}} && \\
    \label{eq:representation_transformer} [\mathbf{X}^{(j+1)}]_n
    &\eqdef  
    %\underbrace{
        \sigma\bullet
        \big(
            [\mathbf{Z}^{(j)}]_n + b^{(j)}
        \big),
    %}_{\text{Activation and Bias}},
    &&\qquad n=1,\dots,N.
\end{alignat}
Here $\bigoplus$ denotes concatenation along the second (column) dimension, $
{\bf Q}_{\mathtt{h}}^{(j)},{\bf K}_{\mathtt{h}}^{(j)}\in \mathbb{R}^{d_{{\rm key}, \mathtt{h}}^{(j)}\times d_j}$ and 
$
{\bf V}_{\mathtt{h}}^{(j)}\in \mathbb{R}^{d_{{\rm value}, \mathtt{h}}^{(j)}\times d_j}$, and $b^{(j)}\in\mathbb{R}^{d_{j+1}}$.
Further, the addition in \eqref{eq:representation_transformer} is applied row-wise.
A \emph{transformer with Multi-head attention} with depth $L$, obtained by composing the layers $\mathcal{T}_0,\mathcal{T}_1,\dots,\mathcal{T}_{L-2}$, is a map $\mathcal{T}:\mathbb{R}^{N\times d}\to\mathbb{R}^{N\times D}$ sending an input matrix ${\bf X}^{(0)}\eqdef {\bf X}\in\mathbb{R}^{N\times d}$ to $\mathcal{T}({\bf X}) \eqdef {\bf A}^{(L-1)}{\bf X}^{(L-1)} + b^{(L-1)}\in\mathbb{R}^{N\times D}$, for some ${\bf A}^{(L-1)}\in\mathbb{R}^{D\times d_{L-1}}$ and $b^{(L-1)}\in\mathbb{R}^D$.
\end{definition}

Similar to the strategy of~\cite{petersen2020equivalence}, we 
show that every ReLU MLP can be converted into a transformer in a canonical fashion. 
This allows us to deduce a quantitative version of the in-context universality results of~\cite{furuya2026transformers} for the transformer model. 
The next result is the direct Multi-head generalization of~\cite[Proposition A.1]{kratsios2026adaptivity} and an analogue of~\cite[Proposition 11]{kratsios2025context} for our formulation of the transformer used here. 

\begin{proposition}[Prescribed-head Transformerification of MLPs]
\label{prop:transformerification__SparseVersion}
Let $L\ge2$, let
$
    f:\mathbb{R}^d\to\mathbb{R}^D
$
be a \texttt{ReLU}-MLP of depth $L$, hidden width $W$, and with
$K$ nonzero affine parameters.
Then, for every $\lambda>0$ and every $H\in\mathbb{N}_+$,
$f$ can be implemented exactly by a \texttt{ReLU} transformer
of depth $L$, with exactly $H$ attention heads at each attention
layer, hidden width at most
$
    W+H-1,
$
and at most $K$ nonzero affine parameters.
\end{proposition}

\begin{proof}
Let $d_0,\ldots,d_L\in\mathbb{N}_+$ satisfy
$
    d_0=d
$
and
$
    d_L=D,
$
and write
\begin{equation}
\label{eq:MLP_prescribed_heads}
\begin{aligned}
    X^{(0)}
    &\eqdef X,
    \\
    X^{(j+1)}
    &\eqdef
    \operatorname{ReLU}\bullet
    \big(
        {\bf A}^{(j)}X^{(j)}+b^{(j)}
    \big),
    \qquad
    j=0,\ldots,L-2,
    \\
    f(X)
    &\eqdef
    {\bf A}^{(L-1)}X^{(L-1)}+b^{(L-1)},
\end{aligned}
\end{equation}
where
$
    {\bf A}^{(j)}
    \in
    \mathbb{R}^{d_{j+1}\times d_j}
$
and
$
    b^{(j)}\in\mathbb{R}^{d_{j+1}}.
$
Set
\[
    \widetilde d_0\eqdef d_0,
    \qquad
    \widetilde d_j\eqdef d_j+H-1,
    \quad
    j=1,\ldots,L-1,
\]
and define the canonical embedding and projection
\[
    \iota_j(z)
    \eqdef
    (z,0_{H-1})
    \in\mathbb{R}^{\widetilde d_j},
    \qquad
    \pi_j(z,u)
    \eqdef z
    \in\mathbb{R}^{d_j},
\]
for $j\ge1$, with
$
    \iota_0=\pi_0=\operatorname{id}_{\mathbb{R}^{d}}.
$

Fix an attention layer
$
    j\in\{0,\ldots,L-2\}.
$
We use $H$ heads.  Give the first head value dimension $d_{j+1}$
and each remaining head value dimension $1$; thus
\[
    d_{\mathrm{value},1}^{(j)}
    \eqdef d_{j+1},
    \qquad
    d_{\mathrm{value},h}^{(j)}
    \eqdef1,
    \quad h=2,\ldots,H,
\]
and therefore
\[
    \sum_{h=1}^{H}
        d_{\mathrm{value},h}^{(j)}
    =
    d_{j+1}+H-1
    =
    \widetilde d_{j+1}.
\]
Take $d_{\mathrm{key},h}^{(j)}\eqdef1$ and set
$
    {\bf Q}^{(j)}_h
    \eqdef0
$, ${\bf K}^{(j)}_h
    \eqdef 0$ for each $h\in[H]_+$.  
For the value maps, set
$
    {\bf V}^{(j)}_1
    \eqdef
    {\bf A}^{(j)}\pi_j
    \in
    \mathbb{R}^{d_{j+1}\times\widetilde d_j}
$ 
and
$
    {\bf V}^{(j)}_h
    \eqdef
    0
    \in
    \mathbb{R}^{1\times\widetilde d_j}
$ for each $h=2,\ldots,H$.  
Finally, take
$
    \widetilde b^{(j)}
    \eqdef
    \iota_{j+1}\big(b^{(j)}\big)
    =
    \big(b^{(j)},0_{H-1}\big)
$ and identify a vector with the one-token matrix having that vector as its
unique row.  Since there is exactly one token, the unique softmax
coefficient in every attention head equals
\[
    \frac{
        \exp\!\big(
            \lambda
            \langle
                {\bf Q}^{(j)}_h x,
                {\bf K}^{(j)}_h x
            \rangle
        \big)
    }{
        \exp\!\big(
            \lambda
            \langle
                {\bf Q}^{(j)}_h x,
                {\bf K}^{(j)}_h x
            \rangle
        \big)
    }
    =1.
\]
Hence, for every $\lambda>0$,
$
    \operatorname{Attn}
    \big(
        x
        \mid
        {\bf Q}^{(j)}_h,
        {\bf K}^{(j)}_h,
        {\bf V}^{(j)}_h,
        \lambda
    \big)
    =
    {\bf V}^{(j)}_h x
$; consequently, whenever
$
    \widetilde X^{(j)}
    =
    \iota_j(X^{(j)}),
$
the concatenated Multi-head output is
\begin{align*}
    \bigoplus_{h=1}^{H}
    \operatorname{Attn}
    \big(
        \widetilde X^{(j)}
        \mid
        {\bf Q}^{(j)}_h,
        {\bf K}^{(j)}_h,
        {\bf V}^{(j)}_h,
        \lambda
    \big)
    &=
    \big(
        {\bf A}^{(j)}X^{(j)},
        0_{H-1}
    \big)
    \\
    &=
    \iota_{j+1}
    \big(
        {\bf A}^{(j)}X^{(j)}
    \big).
\end{align*}
Since $\operatorname{ReLU}(0)=0$, it follows that
\begin{align*}
    \widetilde X^{(j+1)}
    &=
    \operatorname{ReLU}\bullet
    \left[
        \iota_{j+1}
        \big(
            {\bf A}^{(j)}X^{(j)}
        \big)
        +
        \iota_{j+1}
        \big(
            b^{(j)}
        \big)
    \right]
    \\
    &=
    \iota_{j+1}
    \left(
        \operatorname{ReLU}\bullet
        \big(
            {\bf A}^{(j)}X^{(j)}
            +
            b^{(j)}
        \big)
    \right)
    \\
    &=
    \iota_{j+1}
    \big(
        X^{(j+1)}
    \big).
\end{align*}
Induction over $j=0,\ldots,L-2$ therefore gives
$
    \widetilde X^{(j)}
    =
    \iota_j\big(X^{(j)}\big)
$ for each $j=0,\ldots,L-1$.
For the final affine map, define
\[
    \widetilde{\bf A}^{(L-1)}
    \eqdef
    {\bf A}^{(L-1)}\pi_{L-1}
    \in
    \mathbb{R}^{D\times\widetilde d_{L-1}},
    \qquad
    \widetilde b^{(L-1)}
    \eqdef
    b^{(L-1)}.
\]
Then
$
    \mathcal{T}(X)
=
    \widetilde{\bf A}^{(L-1)}
    \widetilde X^{(L-1)}
    +
    \widetilde b^{(L-1)}
=
    {\bf A}^{(L-1)}
    X^{(L-1)}
    +
    b^{(L-1)}
    =
    f(X)$.
Thus, $\mathcal{T}$ realizes $f$ exactly and has the same depth $L$.

It remains to verify the complexity claims.  At every hidden layer,
the construction enlarges $d_j$ by exactly $H-1$ coordinates, and hence
the hidden width is at most $W+H-1$.  Moreover,
${\bf V}^{(j)}_1$ is obtained from ${\bf A}^{(j)}$ only by adjoining
zero columns, while every
${\bf Q}^{(j)}_h$, ${\bf K}^{(j)}_h$, and
${\bf V}^{(j)}_h$ for $h\ge2$ is identically zero.  Likewise,
$\widetilde b^{(j)}$ is obtained from $b^{(j)}$ only by adjoining zeros,
and $\widetilde{\bf A}^{(L-1)}$ is obtained from
${\bf A}^{(L-1)}$ only by adjoining zero columns.  Therefore,
writing $\|\cdot\|_0$ for the number of nonzero entries are
$
    \sum_{j=0}^{L-1}
    \,
    \big(
        \|{\bf A}^{(j)}\|_0
        +
        \|b^{(j)}\|_0
    \big)
=
    K
$.  
\end{proof}

%---
% Insert acknowledgments and information
% regarding funding at the end of the last
% section, i.e., right before the bibliography.
%---

\begin{ack}
The authors acknowledge that resources used in preparing this research were provided, in part, by the Province of Ontario, the Government of Canada through CIFAR, and companies sponsoring the Vector Institute\footnote{\href{https://vectorinstitute.ai/partnerships/current-partners/}{https://vectorinstitute.ai/partnerships/current-partners/}}.    
\end{ack}

\begin{funding}
A.\ Kratsios, H.\ Ghoukasian, and R.\ Hong acknowledge financial support from an NSERC Discovery Grant No.\ RGPIN-2023-04482 and No.\ DGECR-2023-00230.
\end{funding}

\bibliographystyle{plain}
\bibliography{Refs}

@article{barzilai2026beyond,
  title={Beyond benign overfitting in Nadaraya-Watson interpolators},
  author={Barzilai, Daniel and Kornowski, Guy and Shamir, Ohad},
  journal={Advances in Neural Information Processing Systems},
  volume={38},
  pages={15320--15348},
  year={2026}
}

@inproceedings{goel2024can,
  title={Can a transformer represent a Kalman filter?},
  author={Goel, Gautam and Bartlett, Peter},
  booktitle={6th Annual Learning for Dynamics \& Control Conference},
  pages={1502--1512},
  year={2024},
  organization={PMLR}
}

@article{kratsios2025beyond,
  title={An Algorithm with Certified Recovery from Noisy Data Allowing Scalable Fine-Tuning for Structured Task Shifts},
  author={Kratsios, Anastasis and Cheng, Tin Sum and Roy, Daniel},
  journal={arXiv preprint arXiv:2509.00924},
  year={2025}
}

@article{petrova2023limitations,
	author = {Petrova, Guergana and Wojtaszczyk, Przemyslaw},
	journal = {Journal of Machine Learning Research},
	number = {353},
	pages = {1--38},
	title = {Limitations on approximation by deep and shallow neural networks},
	volume = {24},
	year = {2023}}

@article{ShenYangZhang_JMPA_OptApprx_ReLU,
	author = {Shen, Zuowei and Yang, Haizhao and Zhang, Shijun},
	doi = {10.1016/j.matpur.2021.07.009},
	fjournal = {Journal de Math\'{e}matiques Pures et Appliqu\'{e}es. Neuvi\`eme S\'{e}rie},
	issn = {0021-7824,1776-3371},
	journal = {J. Math. Pures Appl. (9)},
	mrclass = {41A63 (41A25 41A46 68V15)},
	mrnumber = {4351074},
	mrreviewer = {Bao\ Huai\ Sheng},
	pages = {101--135},
	title = {Optimal approximation rate of {R}e{LU} networks in terms of width and depth},
	url = {https://doi.org/10.1016/j.matpur.2021.07.009},
	volume = {157},
	year = {2022}}

@article{petrova2023lipschitz,
	author = {Petrova, Guergana and Wojtaszczyk, Przemys{\l}aw},
	journal = {Constructive Approximation},
	number = {2},
	pages = {759--805},
	publisher = {Springer},
	title = {Lipschitz widths},
	volume = {57},
	year = {2023}}

@inproceedings{vardi2021optimal,
	author = {Vardi, Gal and Yehudai, Gilad and Shamir, Ohad},
	booktitle = {International Conference on Learning Representations},
	title = {On the Optimal Memorization Power of ReLU Neural Networks},
	year = {2021}}

@article{siegel2023optimal,
	author = {Siegel, Jonathan W},
	journal = {Journal of Machine Learning Research},
	number = {357},
	pages = {1--52},
	title = {Optimal approximation rates for deep ReLU neural networks on Sobolev and Besov spaces},
	volume = {24},
	year = {2023}}

@article{cuchiero2026global,
	author = {Cuchiero, Christa and Schmocker, Philipp and Teichmann, Josef},
	journal = {Constructive Approximation},
	pages = {1--76},
	publisher = {Springer},
	title = {Global universal approximation of functional input maps on weighted spaces},
	year = {2026}}

@article{schneider2026nonlocal,
	author = {Schneider, Cornelia and Ullrich, Mario and Vybiral, Jan},
	journal = {Journal of Machine Learning Research},
	number = {11},
	pages = {1--41},
	title = {Nonlocal techniques for the analysis of deep relu neural network approximations},
	volume = {27},
	year = {2026}}

@article{hornik1991approximation,
	author = {Hornik, Kurt},
	journal = {Neural networks},
	number = {2},
	pages = {251--257},
	publisher = {Elsevier},
	title = {Approximation capabilities of multilayer feedforward networks},
	volume = {4},
	year = {1991}}

@article{papon2022universal,
	author = {Kratsios, Anastasis and Papon, L{\'e}onie},
	journal = {Journal of Machine Learning Research},
	number = {196},
	pages = {1--73},
	title = {Universal approximation theorems for differentiable geometric deep learning},
	volume = {23},
	year = {2022}}

@article{LuydmillaJP_UAT_2018,
	author = {Grigoryeva, Lyudmila and Ortega, Juan-Pablo},
	fjournal = {Journal of Machine Learning Research (JMLR)},
	issn = {1532-4435,1533-7928},
	journal = {J. Mach. Learn. Res.},
	mrclass = {68Q05 (62M45)},
	mrnumber = {3862431},
	pages = {Paper No. 24, 40},
	title = {Universal discrete-time reservoir computers with stochastic inputs and linear readouts using non-homogeneous state-affine systems},
	volume = {19},
	year = {2018}}

@article{petersen2018optimal,
	author = {Petersen, Philipp and Voigtlaender, Felix},
	journal = {Neural Networks},
	pages = {296--330},
	publisher = {Elsevier},
	title = {Optimal approximation of piecewise smooth functions using deep ReLU neural networks},
	volume = {108},
	year = {2018}}

@article{schmidthieber2020nonparametric,
	author = {Schmidt-Hieber, Johannes},
	doi = {10.1214/19-AOS1875},
	journal = {The Annals of Statistics},
	number = {4},
	pages = {1875--1897},
	title = {Nonparametric Regression Using Deep Neural Networks with {ReLU} Activation Function},
	volume = {48},
	year = {2020}}

@article{bolcskei2019optimal,
	author = {Bolcskei, Helmut and Grohs, Philipp and Kutyniok, Gitta and Petersen, Philipp},
	journal = {SIAM Journal on Mathematics of Data Science},
	number = {1},
	pages = {8--45},
	publisher = {SIAM},
	title = {Optimal approximation with sparsely connected deep neural networks},
	volume = {1},
	year = {2019}}

@article{hong2024bridging,
	author = {Hong, Ruiyang and Kratsios, Anastasis},
	journal = {arXiv preprint arXiv:2409.12335},
	title = {Bridging the gap between approximation and learning via optimal approximation by relu mlps of maximal regularity},
	year = {2024}}

@article{Whitney1934AnalyticExtensions,
	author = {Whitney, Hassler},
	doi = {10.1090/S0002-9947-1934-1501735-3},
	journal = {Transactions of the American Mathematical Society},
	number = {1},
	pages = {63--89},
	title = {Analytic extensions of differentiable functions defined in closed sets},
	volume = {36},
	year = {1934}}

@article{McShane1934ExtensionRangeFunctions,
	author = {McShane, E. J.},
	doi = {10.1090/S0002-9904-1934-05978-0},
	journal = {Bulletin of the American Mathematical Society},
	number = {12},
	pages = {837--842},
	title = {Extension of range of functions},
	volume = {40},
	year = {1934}}

@article{kolmogorov1959varepsilon,
	author = {Kolmogorov, Andrei Nikolaevich and Tikhomirov, Vladimir Mikhailovich},
	journal = {Uspekhi Matematicheskikh Nauk},
	number = {2},
	pages = {3--86},
	publisher = {Russian Academy of Sciences, Steklov Mathematical Institute of Russian~{\ldots}},
	title = {$\varepsilon$-entropy and $\varepsilon$-capacity of sets in function spaces},
	volume = {14},
	year = {1959}}

@book{pinkus1985nwidths,
	address = {Berlin, Heidelberg},
	author = {Pinkus, Allan},
	doi = {10.1007/978-3-642-69894-1},
	publisher = {Springer},
	series = {Ergebnisse der Mathematik und ihrer Grenzgebiete},
	title = {n-Widths in Approximation Theory},
	volume = {7},
	year = {1985}}

@book{novak2008tractability,
	address = {Z{\"u}rich},
	author = {Novak, Erich and Wo{\'z}niakowski, Henryk},
	doi = {10.4171/026},
	publisher = {European Mathematical Society},
	series = {EMS Tracts in Mathematics},
	title = {Tractability of Multivariate Problems. Volume I: Linear Information},
	volume = {6},
	year = {2008}}

@article{devore1998nonlinear,
	author = {DeVore, Ronald A.},
	doi = {10.1017/S0962492900002816},
	journal = {Acta Numerica},
	pages = {51--150},
	publisher = {Cambridge University Press},
	title = {Nonlinear approximation},
	volume = {7},
	year = {1998}}

@article{kolmogorov1959entropy,
	author = {Kolmogorov, A. N. and Tikhomirov, V. M.},
	journal = {Uspekhi Matematicheskikh Nauk},
	note = {English translation: Amer. Math. Soc. Transl. Ser. 2, 17, 277--364, 1961},
	number = {2},
	pages = {3--86},
	title = {$\varepsilon$-entropy and $\varepsilon$-capacity of sets in function spaces},
	volume = {14},
	year = {1959}}

@article{devore1989optimal,
	author = {DeVore, Ronald A. and Howard, Ralph and Micchelli, Charles A.},
	journal = {Manuscripta Mathematica},
	number = {4},
	pages = {469--478},
	title = {Optimal Nonlinear Approximation},
	volume = {63},
	year = {1989}}

@article{marcati2023exponential,
	author = {Marcati, Carlo and Opschoor, Joost AA and Petersen, Philipp C and Schwab, Christoph},
	journal = {Foundations of Computational Mathematics},
	number = {3},
	pages = {1043--1127},
	publisher = {Springer},
	title = {Exponential ReLU neural network approximation rates for point and edge singularities},
	volume = {23},
	year = {2023}}

@article{gonon2025universal,
	author = {Gonon, Lukas and Jacquier, Antoine},
	journal = {IEEE Transactions on Neural Networks and Learning Systems},
	publisher = {IEEE},
	title = {Universal approximation theorem and error bounds for quantum neural networks and quantum reservoirs},
	year = {2025}}

@article{gonon2023approximation,
	author = {Gonon, Lukas and Grigoryeva, Lyudmila and Ortega, Juan-Pablo},
	journal = {The Annals of Applied Probability},
	number = {1},
	pages = {28--69},
	publisher = {Institute of Mathematical Statistics},
	title = {Approximation bounds for random neural networks and reservoir systems},
	volume = {33},
	year = {2023}}

@article{enflo1973counterexample,
	author = {Enflo, Per},
	doi = {10.1007/BF02392270},
	journal = {Acta Mathematica},
	number = {1},
	pages = {309--317},
	title = {A Counterexample to the Approximation Problem in Banach Spaces},
	volume = {130},
	year = {1973}}

@book{grothendieck1955produits,
	address = {Providence, RI},
	author = {Grothendieck, Alexander},
	number = {16},
	pages = {140},
	publisher = {American Mathematical Society},
	series = {Memoirs of the American Mathematical Society},
	title = {Produits tensoriels topologiques et espaces nucl{\'e}aires},
	year = {1955}}

@article{figiel1973approximation,
	author = {Figiel, Tadeusz and Johnson, William B.},
	journal = {Proceedings of the American Mathematical Society},
	pages = {197--200},
	title = {The approximation property does not imply the bounded approximation property},
	volume = {41},
	year = {1973}}

@article{gottlieb2016adaptive,
	author = {Gottlieb, Lee-Ad and Kontorovich, Aryeh and Krauthgamer, Robert},
	journal = {Theoretical Computer Science},
	pages = {105--118},
	publisher = {Elsevier},
	title = {Adaptive metric dimensionality reduction},
	volume = {620},
	year = {2016}}

@article{bartlett2017spectrally,
	author = {Bartlett, Peter L and Foster, Dylan J and Telgarsky, Matus J},
	journal = {Advances in neural information processing systems},
	title = {Spectrally-normalized margin bounds for neural networks},
	volume = {30},
	year = {2017}}

@inproceedings{neyshabur2018a,
	author = {Behnam Neyshabur and Srinadh Bhojanapalli and Nathan Srebro},
	booktitle = {International Conference on Learning Representations},
	title = {A {PAC}-Bayesian Approach to Spectrally-Normalized Margin Bounds for Neural Networks},
	url = {https://openreview.net/forum?id=Skz_WfbCZ},
	year = {2018}}

@article{kratsios2025context,
	author = {Kratsios, Anastasis and Furuya, Takashi},
	journal = {arXiv preprint arXiv:2502.03327},
	title = {Is in-context universality enough? mlps are also universal in-context},
	year = {2025}}

@article{neufeld2023universal,
	author = {Neufeld, Ariel and Schmocker, Philipp},
	journal = {arXiv preprint arXiv:2312.08410},
	title = {Universal approximation property of random neural networks},
	year = {2023}}

@article{devore1993wavelet,
	author = {DeVore, Ronald A and Kyriazis, George and Leviatan, Dany and Tikhomirov, Vladimir M},
	journal = {Adv. Comput. Math.},
	number = {2},
	pages = {197--214},
	title = {Wavelet compression and nonlinear n-widths.},
	volume = {1},
	year = {1993}}

@article{zhang2024deep,
	author = {Zhang, Shijun and Lu, Jianfeng and Zhao, Hongkai},
	journal = {Journal of Machine Learning Research},
	number = {35},
	pages = {1--39},
	title = {Deep network approximation: Beyond {ReLU} to diverse activation functions},
	volume = {25},
	year = {2024}}

@inproceedings{pmlr-v235-cheng24g,
	author = {Cheng, Tin Sum and Lucchi, Aurelien and Kratsios, Anastasis and Belius, David},
	booktitle = {Proceedings of the 41st International Conference on Machine Learning},
	editor = {Salakhutdinov, Ruslan and Kolter, Zico and Heller, Katherine and Weller, Adrian and Oliver, Nuria and Scarlett, Jonathan and Berkenkamp, Felix},
	month = {21--27 Jul},
	pages = {8141--8162},
	publisher = {PMLR},
	series = {Proceedings of Machine Learning Research},
	title = {Characterizing Overfitting in Kernel Ridgeless Regression Through the Eigenspectrum},
	url = {https://proceedings.mlr.press/v235/cheng24g.html},
	volume = {235},
	year = {2024}}

@article{kratsios2025kolmogorov,
	author = {Kratsios, Anastasis and Furuya, Takashi},
	journal = {arXiv preprint arXiv:2504.15110},
	title = {Kolmogorov-Arnold Networks: Approximation and Learning Guarantees for Functions and their Derivatives},
	year = {2025}}

@article{bartlett2019nearly,
	author = {Bartlett, Peter L and Harvey, Nick and Liaw, Christopher and Mehrabian, Abbas},
	journal = {Journal of Machine Learning Research},
	number = {63},
	pages = {1--17},
	title = {Nearly-tight VC-dimension and pseudodimension bounds for piecewise linear neural networks},
	volume = {20},
	year = {2019}}

@inproceedings{rahimi2007random,
	author = {Rahimi, Ali and Recht, Benjamin},
	booktitle = {Advances in Neural Information Processing Systems},
	title = {Random Features for Large-Scale Kernel Machines},
	volume = {20},
	year = {2007}}

@inproceedings{jacot2018ntk,
	author = {Jacot, Arthur and Gabriel, Franck and Hongler, Cl{\'e}ment},
	booktitle = {Advances in Neural Information Processing Systems (NeurIPS)},
	title = {Neural tangent kernel: Convergence and generalization in neural networks},
	volume = {31},
	year = {2018}}

@inproceedings{kidger2020universal,
	author = {Kidger, Patrick and Lyons, Terry},
	booktitle = {Conference on learning theory},
	organization = {PMLR},
	pages = {2306--2327},
	title = {Universal approximation with deep narrow networks},
	year = {2020}}

@article{acciaio2024designing,
	author = {Acciaio, Beatrice and Kratsios, Anastasis and Pammer, Gudmund},
	journal = {Mathematical Finance},
	number = {2},
	pages = {671--735},
	publisher = {Wiley Online Library},
	title = {Designing universal causal deep learning models: The geometric (hyper) transformer},
	volume = {34},
	year = {2024}}

@article{hou2023instance,
	author = {Hou, Songyan and Kassraie, Parnian and Kratsios, Anastasis and Krause, Andreas and Rothfuss, Jonas},
	journal = {Journal of Machine Learning Research},
	number = {349},
	pages = {1--51},
	title = {Instance-dependent generalization bounds via optimal transport},
	volume = {24},
	year = {2023}}

@article{Stone1982OptimalGlobalRates,
	author = {Stone, Charles J.},
	doi = {10.1214/aos/1176345969},
	fjournal = {The Annals of Statistics},
	journal = {Ann. Statist.},
	number = {4},
	pages = {1040--1053},
	title = {Optimal Global Rates of Convergence for Nonparametric Regression},
	volume = {10},
	year = {1982}}

@article{ismailov2014approximation,
	author = {Ismailov, Vugar E},
	journal = {Journal of Mathematical Analysis and Applications},
	number = {2},
	pages = {963--969},
	publisher = {Elsevier},
	title = {On the approximation by neural networks with bounded number of neurons in hidden layers},
	volume = {417},
	year = {2014}}

@article{cybenko1989approximation,
	author = {Cybenko, George},
	doi = {10.1007/BF02551274},
	journal = {Mathematics of Control, Signals and Systems},
	number = {4},
	pages = {303--314},
	title = {Approximation by Superpositions of a Sigmoidal Function},
	volume = {2},
	year = {1989}}

@article{yarotsky2017error,
	author = {Yarotsky, Dmitry},
	journal = {Neural Networks},
	pages = {103--114},
	publisher = {Elsevier},
	title = {Error bounds for approximations with deep ReLU networks},
	volume = {94},
	year = {2017}}

@article{shen2022optimal,
	author = {Shen, Zuowei and Yang, Haizhao and Zhang, Shijun},
	journal = {Journal de Math{\'e}matiques Pures et Appliqu{\'e}es},
	pages = {101--135},
	publisher = {Elsevier},
	title = {Optimal approximation rate of ReLU networks in terms of width and depth},
	volume = {157},
	year = {2022}}

@inproceedings{yarotsky2018optimal,
	author = {Yarotsky, Dmitry},
	booktitle = {Conference on learning theory},
	organization = {PMLR},
	pages = {639--649},
	title = {Optimal approximation of continuous functions by very deep ReLU networks},
	year = {2018}}

@article{petersen2024mathematical,
	archiveprefix = {arXiv},
	author = {Petersen, Philipp and Zech, Jakob},
	eprint = {2407.18384},
	journal = {arXiv preprint},
	title = {Mathematical theory of deep learning},
	year = {2024}}

@article{pinkus1999approximation,
	author = {Pinkus, Allan},
	doi = {10.1017/S0962492900002919},
	isbn = {0-521-77088-2},
	journal = {Acta numerica, 1999},
	mrclass = {41A30 (41A63 65D15 92B20)},
	mrnumber = {1819645},
	mrreviewer = {Andrei\ Mart\'inez Finkelshtein},
	pages = {143--195},
	publisher = {Cambridge Univ. Press, Cambridge},
	series = {Acta Numer.},
	title = {Approximation theory of the {MLP} model in neural networks},
	url = {https://doi.org/10.1017/S0962492900002919},
	volume = {8},
	year = {1999}}

@article{cohen2022optimal,
	author = {Cohen, Albert and DeVore, Ronald and Petrova, Guergana and Wojtaszczyk, Przemyslaw},
	journal = {Foundations of Computational Mathematics},
	number = {3},
	pages = {607--648},
	publisher = {Springer},
	title = {Optimal stable nonlinear approximation},
	volume = {22},
	year = {2022}}

@article{cheng2024comprehensive,
	author = {Cheng, Tin Sum and Lucchi, Aurelien and Kratsios, Anastasis and Belius, David},
	journal = {Advances in Neural Information Processing Systems},
	pages = {24659--24723},
	title = {A comprehensive analysis on the learning curve in kernel ridge regression},
	volume = {37},
	year = {2024}}

@article{kearns1994efficient,
	author = {Michael J. Kearns and Robert E. Schapire},
	doi = {10.1016/S0022-0000(05)80062-5},
	journal = {Journal of Computer and System Sciences},
	number = {3},
	pages = {464--497},
	title = {Efficient Distribution-Free Learning of Probabilistic Concepts},
	volume = {48},
	year = {1994}}

@article{bartlett1996fat,
	author = {Peter L. Bartlett and Philip M. Long and Robert C. Williamson},
	doi = {10.1006/jcss.1996.0033},
	journal = {Journal of Computer and System Sciences},
	number = {3},
	pages = {434--452},
	title = {Fat-Shattering and the Learnability of Real-Valued Functions},
	volume = {52},
	year = {1996}}

@article{gottlieb2014efficient,
	author = {Gottlieb, Lee-Ad and Kontorovich, Aryeh and Krauthgamer, Robert},
	journal = {IEEE Transactions on Information Theory},
	number = {9},
	pages = {5750--5759},
	publisher = {IEEE},
	title = {Efficient classification for metric data},
	volume = {60},
	year = {2014}}

@article{Yudovich1963Nonstationary,
	author = {Yudovich, V. I.},
	doi = {10.1016/0041-5553(63)90247-7},
	journal = {USSR Computational Mathematics and Mathematical Physics},
	number = {6},
	pages = {1407--1456},
	title = {Non-stationary flow of an ideal incompressible liquid},
	volume = {3},
	year = {1963}}

@book{RevuzYor1999Continuous,
	author = {Revuz, Daniel and Yor, Marc},
	doi = {10.1007/978-3-662-06400-9},
	edition = {3},
	publisher = {Springer},
	series = {Grundlehren der mathematischen Wissenschaften},
	title = {Continuous Martingales and Brownian Motion},
	volume = {293},
	year = {1999}}

@inproceedings{saez2024neural,
	author = {S{\'a}ez de Oc{\'a}riz Borde, Haitz and Kratsios, Anastasis},
	booktitle = {International Conference on Learning Representations},
	pages = {54236--54250},
	title = {Neural snowflakes: Universal latent graph inference via trainable latent geometries},
	volume = {2024},
	year = {2024}}

@article{Schoenberg1937MetricSpaces,
	author = {Schoenberg, I. J.},
	doi = {10.2307/1968835},
	journal = {Annals of Mathematics},
	number = {4},
	pages = {787--793},
	title = {On Certain Metric Spaces Arising From Euclidean Spaces by a Change of Metric and Their Imbedding in Hilbert Space},
	volume = {38},
	year = {1937}}

@article{Wilson1935MetricTransformations,
	author = {Wilson, W. A.},
	doi = {10.2307/2372019},
	journal = {American Journal of Mathematics},
	number = {1},
	pages = {62--68},
	title = {On Certain Types of Continuous Transformations of Metric Spaces},
	volume = {57},
	year = {1935}}

@article{naor2012assouad,
	author = {Naor, Assaf and Neiman, Ofer},
	journal = {Revista Matematica Iberoamericana},
	number = {4},
	pages = {1123--1142},
	title = {Assouad's theorem with dimension independent of the snowflaking},
	volume = {28},
	year = {2012}}

@article{assouad111plongements,
	author = {Assouad, P},
	journal = {Bull. Soc. Math. France},
	number = {4},
	title = {Plongements lipschitziens dans R\^{} n Rn},
	volume = {111}}

@inproceedings{saez2025neural,
	author = {S{\'a}ez de Oc{\'a}riz Borde, Haitz and Kratsios, Anastasis and Law, Marc T and Dong, Xiaowen and Bronstein, Michael},
	booktitle = {International Conference on Learning Representations},
	pages = {56302--56341},
	title = {Neural spacetimes for DAG representation learning},
	volume = {2025},
	year = {2025}}

@article{shikhkhalil2025loss,
	author = {Shikh Khalil, Karim R.},
	doi = {10.1016/j.jfa.2025.111066},
	journal = {Journal of Functional Analysis},
	number = {8},
	pages = {111066},
	title = {On the Loss and Propagation of Modulus of Continuity for the Two-Dimensional Incompressible Euler Equations},
	volume = {289},
	year = {2025}}

@article{mendelnaor2011dichotomies,
	author = {Mendel, Manor and Naor, Assaf},
	journal = {arXiv preprint arXiv:1102.1800},
	title = {A Note on Dichotomies for Metric Transforms},
	year = {2011}}

@article{zaheer2017deep,
	author = {Zaheer, Manzil and Kottur, Satwik and Ravanbakhsh, Siamak and Poczos, Barnabas and Salakhutdinov, Russ R and Smola, Alexander},
	journal = {Advances in neural information processing systems},
	title = {Deep sets},
	volume = {30},
	year = {2017}}

@article{wagstaff2022universal,
	author = {Wagstaff, Edward and Fuchs, Fabian B and Engelcke, Martin and Osborne, Michael A and Posner, Ingmar},
	journal = {Journal of Machine Learning Research},
	number = {151},
	pages = {1--56},
	title = {Universal approximation of functions on sets},
	volume = {23},
	year = {2022}}

@article{syed2026embedding,
	author = {Syed, Ali and Nambiar, Aditya and Siegel, Jonathan W},
	journal = {arXiv preprint arXiv:2605.08377},
	title = {Embedding Dimension Lower Bounds for Universality of Deep Sets and Janossy Pooling},
	year = {2026}}

@article{murari2026approximation,
	author = {Murari, Davide and Furuya, Takashi and Sch{\"o}nlieb, Carola-Bibiane},
	journal = {Advances in Neural Information Processing Systems},
	pages = {142888--142919},
	title = {Approximation theory for 1-lipschitz resnets},
	volume = {38},
	year = {2026}}

@article{riegler2024generating,
	author = {Riegler, Erwin and B{\"u}hler, Alex and Pan, Yang and B{\"o}lcskei, Helmut},
	journal = {arXiv preprint arXiv:2412.05109},
	title = {Generating rectifiable measures through neural networks},
	year = {2024}}

@inproceedings{meunier2022dynamical,
	author = {Meunier, Laurent and Delattre, Blaise J and Araujo, Alexandre and Allauzen, Alexandre},
	booktitle = {International Conference on Machine Learning},
	organization = {PMLR},
	pages = {15484--15500},
	title = {A dynamical system perspective for Lipschitz neural networks},
	year = {2022}}

@inproceedings{araujo2023a,
	author = {Alexandre Araujo and Aaron J Havens and Blaise Delattre and Alexandre Allauzen and Bin Hu},
	booktitle = {The Eleventh International Conference on Learning Representations},
	title = {A Unified Algebraic Perspective on Lipschitz Neural Networks},
	url = {https://openreview.net/forum?id=k71IGLC8cfc},
	year = {2023}}

@book{cobzas2019lipschitz,
	address = {Cham},
	author = {{\c S}tefan Cobza{\c s} and Radu Miculescu and Adriana Nicolae},
	doi = {10.1007/978-3-030-16489-8},
	publisher = {Springer},
	series = {Lecture Notes in Mathematics},
	title = {Lipschitz Functions},
	volume = {2241},
	year = {2019}}

@book{deberg2008computational,
	address = {Berlin, Heidelberg},
	author = {Mark de Berg and Otfried Cheong and Marc van Kreveld and Mark Overmars},
	doi = {10.1007/978-3-540-77974-2},
	edition = {3},
	isbn = {978-3-540-77973-5},
	publisher = {Springer},
	title = {Computational Geometry: Algorithms and Applications},
	year = {2008}}

@article{dupont2026improving,
	author = {Dupont, Emilien and Eisenberger, Marvin and Kozlovskii, Borislav and Mehrabian, Abbas and Ruiz, Francisco J. R. and See, Abigail and Zhou, Renfei and Alman, Josh and Vassilevska Williams, Virginia and Balog, Matej},
	journal = {arXiv preprint arXiv:2608.16884},
	title = {Improving the Matrix Multiplication Exponent with Modern Optimization and {AlphaEvolve}},
	year = {2026}}

@book{nesterov2018lectures,
	author = {Nesterov, Yurii},
	doi = {10.1007/978-3-319-91578-4},
	publisher = {Springer},
	series = {Springer Optimization and Its Applications},
	title = {Lectures on Convex Optimization},
	volume = {137},
	year = {2018}}

@misc{lu2024numerical,
	archiveprefix = {arXiv},
	author = {Lu, Jun},
	eprint = {2107.02579},
	primaryclass = {math.HO},
	title = {Numerical Matrix Decomposition},
	year = {2024}}

@misc{furuya2026transformers,
	author = {Takashi Furuya and Maarten V. de Hoop and Matti Lassas},
	title = {Transformers through the lens of support-preserving maps between measures},
	url = {https://openreview.net/forum?id=xm5MELPxTv},
	year = {2026}}

@article{bartl2025we,
	author = {Bartl, Daniel and Mendelson, Shahar},
	journal = {arXiv preprint arXiv:2502.15118},
	title = {Do we really need the Rademacher complexities?},
	year = {2025}}

@article{bartlett2005local,
	author = {Bartlett, Peter L. and Bousquet, Olivier and Mendelson, Shahar},
	doi = {10.1214/009053605000000282},
	journal = {The Annals of Statistics},
	number = {4},
	pages = {1497--1537},
	publisher = {Institute of Mathematical Statistics},
	title = {Local Rademacher complexities},
	volume = {33},
	year = {2005}}

@article{alon1997scale,
	author = {Alon, Noga and Ben-David, Shai and Cesa-Bianchi, Nicol{\`o} and Haussler, David},
	doi = {10.1145/263867.263927},
	journal = {Journal of the ACM},
	number = {4},
	pages = {615--631},
	publisher = {Association for Computing Machinery},
	title = {Scale-sensitive dimensions, uniform convergence, and learnability},
	volume = {44},
	year = {1997}}

@article{ba2022high,
	author = {Ba, Jimmy and Erdogdu, Murat A and Suzuki, Taiji and Wang, Zhichao and Wu, Denny and Yang, Greg},
	journal = {Advances in Neural Information Processing Systems},
	pages = {37932--37946},
	title = {High-dimensional asymptotics of feature learning: How one gradient step improves the representation},
	volume = {35},
	year = {2022}}

@article{holzmuller2022training,
	author = {Holzm{\"u}ller, David and Steinwart, Ingo},
	journal = {Journal of Machine Learning Research},
	number = {181},
	pages = {1--82},
	title = {Training Two-Layer {ReLU} Networks with Gradient Descent is Inconsistent},
	volume = {23},
	year = {2022}}

@article{jentzen2025nonconvergence,
	author = {Jentzen, Arnulf and Riekert, Adrian},
	doi = {10.1137/24M1639464},
	journal = {SIAM/ASA Journal on Uncertainty Quantification},
	number = {3},
	pages = {1294--1333},
	title = {Non-convergence to Global Minimizers for {Adam} and Stochastic Gradient Descent Optimization and Constructions of Local Minimizers in the Training of Artificial Neural Networks},
	volume = {13},
	year = {2025}}

@article{do2026nonconvergence,
	author = {Do, Thang and Hannibal, Sonja and Jentzen, Arnulf},
	doi = {10.1016/j.jmaa.2026.130724},
	journal = {Journal of Mathematical Analysis and Applications},
	number = {2},
	pages = {130724},
	title = {Non-convergence to Global Minimizers in Data Driven Supervised Deep Learning: {Adam} and Stochastic Gradient Descent Optimization Provably Fail to Converge to Global Minimizers in the Training of Deep Neural Networks with {ReLU} Activation},
	volume = {564},
	year = {2026}}

@misc{do2025optimalrisk,
	archiveprefix = {arXiv},
	author = {Do, Thang and Jentzen, Arnulf and Riekert, Adrian},
	eprint = {2503.01660},
	primaryclass = {cs.LG},
	title = {Non-convergence to the Optimal Risk for {Adam} and Stochastic Gradient Descent Optimization in the Training of Deep Neural Networks},
	year = {2025}}

@inproceedings{safran2018spurious,
	author = {Safran, Itay and Shamir, Ohad},
	booktitle = {Proceedings of the 35th International Conference on Machine Learning},
	pages = {4433--4441},
	publisher = {PMLR},
	series = {Proceedings of Machine Learning Research},
	title = {Spurious Local Minima are Common in Two-Layer {ReLU} Neural Networks},
	volume = {80},
	year = {2018}}

@article{cheridito2021nonconvergence,
	author = {Cheridito, Patrick and Jentzen, Arnulf and Rossmannek, Florian},
	doi = {10.1016/j.jco.2020.101540},
	journal = {Journal of Complexity},
	pages = {101540},
	title = {Non-convergence of Stochastic Gradient Descent in the Training of Deep Neural Networks},
	volume = {64},
	year = {2021}}

@misc{naor2008concentration,
	author = {Naor, Assaf},
	title = {Concentration of measure},
	year = {2008}}

@article{gottlieb2017efficientregression,
	author = {Gottlieb, Lee-Ad and Kontorovich, Aryeh and Krauthgamer, Robert},
	doi = {10.1109/TIT.2017.2713820},
	journal = {IEEE Transactions on Information Theory},
	number = {8},
	pages = {4838--4849},
	title = {Efficient Regression in Metric Spaces via Approximate Lipschitz Extension},
	volume = {63},
	year = {2017}}

@article{balazs2026nearoptimal,
	author = {Bal{\'a}zs, G{\'a}bor},
	journal = {Journal of Machine Learning Research},
	number = {134},
	pages = {1--41},
	title = {Near-optimal {Delta}-convex Estimation of Lipschitz Functions},
	volume = {27},
	year = {2026}}

@article{luxburg2004distance,
	author = {Luxburg, Ulrike von and Bousquet, Olivier},
	journal = {Journal of Machine Learning Research},
	number = {Jun},
	pages = {669--695},
	title = {Distance-based classification with Lipschitz functions},
	volume = {5},
	year = {2004}}

@article{kratsios2026adaptivity,
	author = {Kratsios, Anastasis and Neuman, A Martina and Petersen, Philipp},
	journal = {arXiv preprint arXiv:2605.04995},
	title = {Adaptivity Under Realizability Constraints: Comparing In-Context and Agentic Learning},
	year = {2026}}

@inproceedings{kratsios2020noneuclidean,
  title     = {Non-Euclidean Universal Approximation},
  author    = {Kratsios, Anastasis and Bilokopytov, Ievgen},
  booktitle = {Advances in Neural Information Processing Systems},
  volume    = {33},
  pages     = {10635--10646},
  year      = {2020}
}

@article{kratsios2022universalgeo,
  title   = {Universal Approximation Theorems for Differentiable Geometric Deep Learning},
  author  = {Kratsios, Anastasis and Papon, L{\'e}onie},
  journal = {Journal of Machine Learning Research},
  volume  = {23},
  number  = {196},
  pages   = {1--73},
  year    = {2022}
}

@article{kratsios2023metric,
  title   = {An Approximation Theory for Metric Space-Valued Functions With A View Towards Deep Learning},
  author  = {Kratsios, Anastasis and Liu, Chong and Lassas, Matti and de Hoop, Maarten V. and Dokmani{\'c}, Ivan},
  journal = {arXiv preprint arXiv:2304.12231},
  year    = {2023}
}

@article{galimberti2025neural,
  title   = {Neural Networks in Non-Metric Spaces},
  author  = {Galimberti, Luca},
  journal = {Analysis and Applications},
  year    = {2025},
  note    = {Accepted/in press}
}

@article{ismailov2026topological,
  title   = {On Shallow Feedforward Neural Networks with Inputs from a Topological Space},
  author  = {Ismailov, Vugar E.},
  journal = {Annals of Mathematics and Artificial Intelligence},
  year    = {2026},
  doi     = {10.1007/s10472-026-10003-7}
}

@article{godeke2025encoderdecoder,
  title   = {New Universal Operator Approximation Theorem for Encoder--Decoder Architectures},
  author  = {G{\"o}deke, Janek and Fernsel, Pascal},
  journal = {arXiv preprint arXiv:2503.24092},
  year    = {2025}
}

@article{nadaraya1964estimating,
  title   = {On Estimating Regression},
  author  = {Nadaraya, E. A.},
  journal = {Theory of Probability and Its Applications},
  volume  = {9},
  number  = {1},
  pages   = {141--142},
  year    = {1964},
  doi     = {10.1137/1109020}
}

@article{watson1964smooth,
  title   = {Smooth Regression Analysis},
  author  = {Watson, G. S.},
  journal = {Sankhy{\=a}: The Indian Journal of Statistics, Series A},
  volume  = {26},
  number  = {4},
  pages   = {359--372},
  year    = {1964}
}

@article{yuan2006model,
  title={Model selection and estimation in regression with grouped variables},
  author={Yuan, Ming and Lin, Yi},
  journal={Journal of the Royal Statistical Society: Series B (Statistical Methodology)},
  volume={68},
  number={1},
  pages={49--67},
  year={2006}
}

@article{zou2005regularization,
  title={Regularization and variable selection via the elastic net},
  author={Zou, Hui and Hastie, Trevor},
  journal={Journal of the Royal Statistical Society: Series B (Statistical Methodology)},
  volume={67},
  number={2},
  pages={301--320},
  year={2005}
}

@article{kimeldorf1971some,
  title={Some Results on Tchebycheffian Spline Functions},
  author={Kimeldorf, George and Wahba, Grace},
  journal={Journal of Mathematical Analysis and Applications},
  volume={33},
  number={1},
  pages={82--95},
  year={1971}
}

@article{smola2004tutorial,
  title={A Tutorial on Support Vector Regression},
  author={Smola, Alex J. and Sch{\"o}lkopf, Bernhard},
  journal={Statistics and Computing},
  volume={14},
  pages={199--222},
  year={2004}
}

@inproceedings{chen2016xgboost,
  title     = {XGBoost: A Scalable Tree Boosting System},
  author    = {Chen, Tianqi and Guestrin, Carlos},
  booktitle = {Proceedings of the 22nd ACM SIGKDD International Conference on Knowledge Discovery and Data Mining},
  pages     = {785--794},
  year      = {2016},
  publisher = {ACM},
  doi       = {10.1145/2939672.2939785}
}

@inproceedings{ke2017lightgbm,
  title     = {LightGBM: A Highly Efficient Gradient Boosting Decision Tree},
  author    = {Ke, Guolin and Meng, Qi and Finley, Thomas and Wang, Taifeng
               and Chen, Wei and Ma, Weidong and Ye, Qiwei and Liu, Tie-Yan},
  booktitle = {Advances in Neural Information Processing Systems},
  volume    = {30},
  pages     = {3146--3154},
  year      = {2017}
}

@article{lukovsevivcius2009reservoir,
  title={Reservoir computing approaches to recurrent neural network training},
  author={Luko{\v{s}}evi{\v{c}}ius, Mantas and Jaeger, Herbert},
  journal={Computer science review},
  volume={3},
  number={3},
  pages={127--149},
  year={2009},
  publisher={Elsevier}
}

@article{allen2019can,
  title={What can resnet learn efficiently, going beyond kernels?},
  author={Allen-Zhu, Zeyuan and Li, Yuanzhi},
  journal={Advances in neural information processing systems},
  volume={32},
  year={2019}
}

@article{mhaskar2016deep,
  title={Deep vs. shallow networks: An approximation theory perspective},
  author={Mhaskar, Hrushikesh N and Poggio, Tomaso},
  journal={Analysis and Applications},
  volume={14},
  number={06},
  pages={829--848},
  year={2016},
  publisher={World Scientific}
}

@article{petersen2020equivalence,
  title={Equivalence of approximation by convolutional neural networks and fully-connected networks},
  author={Petersen, Philipp and Voigtlaender, Felix},
  journal={Proceedings of the American Mathematical Society},
  volume={148},
  number={4},
  pages={1567--1581},
  year={2020}
}

\end{document}